\documentclass[conference]{IEEEtran}

\usepackage{graphicx}
\usepackage{caption}
\graphicspath{
    {figures/}
}

\usepackage{amsmath,amssymb,enumerate}

\usepackage{amsthm}
\usepackage{setspace}
\usepackage{booktabs}
\usepackage[usenames,dvipsnames,svgnames,table]{xcolor}
\usepackage{mathtools}
\usepackage{blindtext}
\usepackage{gensymb}
\usepackage{xparse}
\usepackage{lipsum}
\usepackage{mathrsfs}
\usepackage[mathscr]{euscript}
\usepackage{times}
\usepackage[numbers,sort&compress]{natbib}
\usepackage{multicol}
\definecolor{darkgreen}{rgb}{0,0.6,0}
\usepackage[bookmarks=true,colorlinks=true,citecolor=black,pdfpagemode=UseNone,linkcolor=black,urlcolor=BrickRed,pagebackref]{hyperref}
\usepackage[caption=false,font=footnotesize]{subfig}
\usepackage{amsfonts}
\usepackage{cleveref}
\usepackage[utf8]{inputenc}
\usepackage[T1]{fontenc}
\usepackage{textcomp}
\usepackage{arydshln}
\usepackage{tabu}
\usepackage{cuted}
\usepackage{flushend}
\usepackage{balance}

\usepackage{thmtools,thm-restate}

\newtheorem{problem}{Problem}
\newtheorem{theorem}{Theorem}

\newtheorem{proposition}[theorem]{Proposition}

\newtheorem{definition}{Definition}

\newtheorem{assumption}{Assumption}

\DeclareMathOperator*{\argmin}{arg\,min}

\renewcommand{\thefigure}{\arabic{figure}}

\definecolor{note}{rgb}{0.1,0.1,1}
\definecolor{rephase}{rgb}{0.15,0.7,0.15}
\definecolor{bag}{rgb}{0.6,0.6,0.2}

\makeatletter
\renewcommand*\env@matrix[1][c]{\hskip -\arraycolsep
  \let\@ifnextchar\new@ifnextchar
  \array{*\c@MaxMatrixCols #1}}
\makeatother

\newcommand{\transpose}{\mathsf{T}}

\makeatletter
\newcommand{\mathleft}{\@fleqntrue\@mathmargin0pt}
\newcommand{\mathcenter}{\@fleqnfalse}
\makeatother

\usepackage{bm}
\usepackage{cleveref}

\usepackage{amsfonts}
\usepackage{amssymb}
\usepackage{amsbsy}
\usepackage{amsmath}
\usepackage{graphicx}
\usepackage{subcaption}
\usepackage{algorithm}
\usepackage{tablefootnote}
\usepackage{algorithm}
\usepackage{multirow}
\usepackage{algpseudocode}

\newcommand{\M}[1]{{\bm #1}} 

\newcommand{\vb}{\boldsymbol{b}}

\newcommand{\ve}{\boldsymbol{e}}

\newcommand{\vv}{\boldsymbol{v}}

\newcommand{\vw}{\boldsymbol{w}}

\IEEEoverridecommandlockouts

\newcommand{\GMKF}{GMKF\xspace}
\newcommand{\GMKFlong}{Generalized Moment Kalman Filter\xspace}
\newcommand{\BPUE}{{BPUE}\xspace} 
\newcommand{\BLUE}{{\color{blue}BLUE}\xspace}

\title{GMKF: \GMKFlong for Polynomial Systems with Arbitrary Noise}

\author{Author Names Omitted for Anonymous Review. Submission Number: 236}
\author{Sangli Teng,
Harry Zhang\textsuperscript{*}, 
David Jin\textsuperscript{*}, 
Ashkan Jasour, 
Maani Ghaffari,
Luca Carlone\\

\thanks{S.~Teng and M.~Ghaffari are with the University of Michigan, Ann Arbor, MI 48109, USA. {\tt\small\{sanglit,maanigj\} @umich.edu}}
\thanks{H.~Zhang, D.~Jin and L.~Carlone are with the Massachusetts Institute of Technology, Cambridge, MA, 02139, USA. {\tt\small\{harryz,jindavid,lcarlone\} @mit.edu}}
\thanks{A.~Jasour is with Team 347T-Robotic Aerial Mobility, Jet Propulsion Lab, Pasadena, CA, 91109, USA. {\tt\small jasour@jpl.caltech.edu}}%
\thanks{* Equal contribution.}
}
\begin{document}

\maketitle
\thispagestyle{plain}
\pagestyle{plain}


\begin{abstract}

    This paper develops a new filtering approach for state estimation in polynomial systems corrupted by arbitrary noise, which commonly arise in robotics. We first consider a batch setup where we perform state estimation using all data collected from the initial to the current time. We formulate the batch state estimation problem as a Polynomial Optimization Problem (POP) and relax the assumption of Gaussian noise by specifying a finite number of moments of the noise. 
    We solve the resulting POP using a moment relaxation and prove that under suitable conditions on the rank of the relaxation, (i) we can extract a provably optimal estimate from the moment relaxation, and (ii) we can obtain a belief representation from the dual (sum-of-squares) relaxation.
    We then turn our attention to the filtering setup and apply similar insights to develop a \emph{\GMKFlong} (\GMKF) for recursive state estimation in polynomial systems with arbitrary noise. The GMKF formulates the prediction and update steps as POPs and solves them using moment relaxations, carrying over a possibly non-Gaussian belief. In the linear-Gaussian case, GMKF reduces to the standard Kalman Filter. We demonstrate that GMKF performs well under highly non-Gaussian noise and outperforms common alternatives, including the Extended and Unscented Kalman Filter, and their variants on matrix Lie group.

\end{abstract}
\section{Introduction}
The Kalman Filter (KF) is an optimal estimator for linear systems with Gaussian noise in the sense that it provably computes a minimum-variance estimate of the system's state. 
However, the optimality of the KF breaks when dealing with nonlinear systems or non-Gaussian noise. 
A plethora of extensions of the KF has been developed to deal with nonlinearity and non-Gaussianity.
For example, linearization-based methods, such as the Extended Kalman Filter (EKF)~\cite{song1992extended}, linearize the process and measurement models and propagate the covariance using a standard KF. Other techniques, such as the Unscented Kalman Filter (UKF)~\cite{wan2000unscented}, rely on sampling to better capture the nonlinearities of the system but then still fit a Gaussian distribution to the belief.
Unfortunately, these extensions do not enjoy the desirable theoretical properties of the KF: 
they are not guaranteed to produce an optimal estimate and
---in the nonlinear or non-Gaussian case--- their mean and covariance might be far from describing the actual belief distribution. 
Consider for example a system with non-Gaussian noise sampled from a discrete distribution over $\{-1, 1\}$; clearly, approximating this noise as a zero-mean Gaussian fails to capture the bimodal nature of the distribution. For this reason, existing nonlinear filters, such as the EKF and the UKF, can hardly capture the noise distribution precisely using only the mean and the covariance of the noise. 

\begin{figure}
    \centering
    \includegraphics[width=0.95\columnwidth]{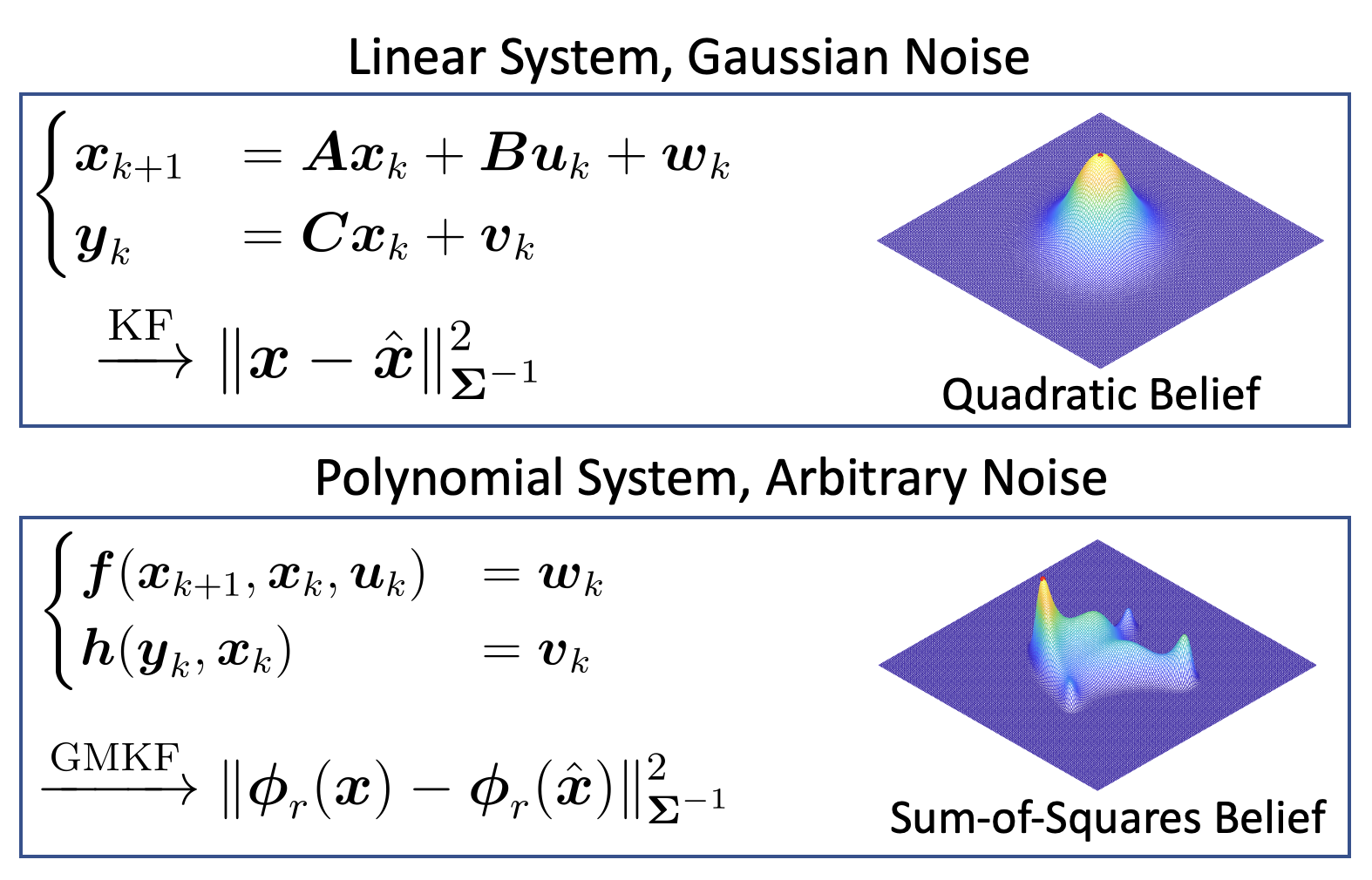}
    \caption{
    We generalize the classical Kalman Filter (KF) to operate on polynomial systems with arbitrary noise. 
    In linear Gaussian systems, the noise is described by its mean and covariance and the KF ``summarizes'' past measurements
    into a Gaussian belief, whose negative log-likelihood takes the form $\|\M x - \hat{\M  x}\|^2_{\M \Sigma^{-1}}$ (up to constants). 
    The proposed \GMKF generalizes the KF by considering an arbitrary number of moments of the noise, and summarizes the measurements into a \emph{Sum-of-Squares} belief, which generalizes the quadratic form of the KF to include higher-order moments (the vector-valued function $\M\phi_r(\M{x})$ contains monomials of degree up to $r$).     
    \label{fig:mkf-image3}
    }
\end{figure}

Thus, a natural question is, \emph{how can we provide reliable state estimation for nonlinear systems corrupted by arbitrary noise?} When the noise is non-Gaussian, and the system is non-linear, the estimation problem becomes less structured for analysis, and traditional methods only focus on the first two moments (i.e., mean and covariance) to describe the noise and the estimator belief, which are too crude to capture the true underlying distributions. We argue that by accounting for higher-order moments, we can recover the state distribution more precisely. In this work, we focus on \emph{polynomial systems}, a broad class of potentially nonlinear systems described by polynomial equations. 
These systems can model dynamics over Lie groups~\cite{Teng-RSS-23, ghaffari2022progress, teng2022error, teng2022lie, 10301632} and a broad set of measurement models~\cite{yang2022certifiably}. To handle arbitrary noise, we let the user specify an arbitrary number
of moments to better describe the distribution of the process and measurement noise.

{\bf Contribution.} 
To address the above question, we provide two key contributions. 
Our first contribution (\Cref{sec:BPUE}) is to consider a batch setup, where one has to estimate the state of the polynomial system from the initial to the current time given all the measurements available until that time. 
In the batch case, we show that one can formulate optimal estimation as a Polynomial Optimization Problem (POP) and relax the assumption of Gaussian noise by instead requiring the user to specify a finite set of moments of the noise. 
    We solve the resulting POP using a semidefinite (SDP) moment relaxation and prove that under suitable conditions on the rank of the relaxation  
    (i) we can extract a provably optimal estimate from the moment relaxation, and (ii) we can obtain a belief representation from the dual (sum-of-squares) relaxation. We call the resulting estimator \emph{\BPUE} (\emph{Best Polynomial Unbiased Estimator}) and prove its theoretical properties in \Cref{sec:theory}.
    While in the batch setup we can compute a provably optimal estimate (whenever the SDP produced a rank-1 solution), such setup is mostly of theoretical interest, since the sizes of the to-be-estimated state and the resulting SDP grow over time and quickly become unmanageable.

Therefore, our second contribution is to apply similar insights to develop a \emph{\GMKFlong} (\GMKF) for recursive state estimation in polynomial systems with arbitrary noise (\Cref{sec:GMKF}).
The GMKF formulates the prediction and update steps as POPs and solves them using moment relaxations, 
carrying over a possibly non-Gaussian belief representation.
In the Linear Gaussian case, GMKF reduces to the standard Kalman Filter.
 While currently we cannot prove that the GMKF is an optimal estimator, we empirically show it
performs well under highly non-Gaussian noise and largely
outperforms common alternatives including the Extended, Unscented, and Invariant Kalman Filter.

Before delving into our contributions, we briefly review preliminaries on polynomial optimization, moment relaxations, and sum-of-squares programming in \Cref{sec:preliminaries}, and provide a succinct problem statement in \Cref{sec:formulation}, while we postpone 
discussion and a broader literature review to Sections~\ref{sec:discussion}-\ref{sec:relatedWork}. 
\Cref{sec:conclusion} concludes the paper.

\section{Preliminaries}
\label{sec:preliminaries}

We review key facts about polynomials, moment relaxations, and sum-of-squares programming. We also provide introductory examples for the non-expert reader in Appendix~\ref{apx:B_mat_example}. 
\subsection{Polynomials and Moment Matrix}
Let $\mathbb{R}[\M x]$ be the ring of polynomials with real coefficients where $\M x:= \begin{bmatrix}x_1, x_2, \dots, x_n  \end{bmatrix}^\transpose$. 
Given an integer $r $, we define the set $ \mathbb{N}_r^n := \{ \alpha \in \mathbb{N}^n | \sum_i \alpha_i \le r \}$ (i.e., $\alpha$ is a vector of integers that sums up to $r$). 
A monomial of degree up to $r$ is denoted as $\M x^{\alpha}:=x_1^{\alpha_1}x_2^{\alpha_2}\dots x_n^{\alpha_n}$, $ \alpha \in \mathbb{N}_r^n$.  
For a polynomial $p(\M x):=\sum_\gamma c_{\gamma}\M x^{\gamma} $, its degree $\deg p$ is defined as the largest $\|\gamma_i\|_1$ with $c_{\gamma} \neq 0$. 

Let $s(r, n):= |\mathbb{N}_r^n| = \begin{pmatrix}
    n+r \\ 
    n
\end{pmatrix}$. We then define the homogeneous basis function $\boldsymbol{\phi}_r(\M x): \mathbb{R}^{n} \rightarrow \mathbb{R}^{ s(r, n)- 1}$ using all the entries of the canonical basis except for $1$:
\begin{equation}
        {\M \phi}_r(\M x)= \begin{bmatrix}
        x_1, x_2,\dots,x_n, x_1^2,x_1x_2,\dots,x_n^2,x_1^r,x_2^r,\dots,x_n^r
        \end{bmatrix}^{\transpose}. 
\end{equation}
The canonical basis for polynomials  of degree up to $r$ is:
\begin{equation}
\label{eq:basis}
    \vv_r(\M x):= \begin{bmatrix}1 \\ {\M \phi}_r(\M x)\end{bmatrix}. 
\end{equation}
Using $\vv_r(\M x)$, we define the $r$-th order moment matrix:
\begin{equation}
    \M{M}_r(\M x) = \vv_r(\M x)\vv^{\transpose}_r(\M x).
\end{equation}
Note that the moment matrix $\M{M}_r(\M x)$ is positive semidefinite and has rank one for any value of $\M x$. 
%

We say that a polynomial is a sum-of-squares (SOS) if it can be factored as: $$\sigma(\M x) = \vv^{\transpose}_r(\M x)\M{Y}\vv_r(\M x) := \|\vv_r(\M x)\|^2_{\M{Y}} $$ 
for some positive semidefinite matrix $ \M{Y} \succeq 0$. Note that for a given SOS polynomial $\sigma(\M{x})$, the associated matrix $\M{Y}$ may not be unique, as discussed in \Cref{lemma:exist_sos_xi}, Appendix~\ref{apx:proof:exist_sos_xi}.

\subsection{Polynomial Optimization and Semidefinite Relaxations}
We introduce the following equality-constrained Polynomial Optimization Problem \eqref{prob:pop} and its moment and sum-of-squares (SOS) semidefinite relaxations.
\begin{definition}[Equality-constrained POP]
An equality-constrained polynomial optimization problem (POP) is an optimization problem in the form
\begin{equation}
\label{prob:pop}
\begin{aligned}
    p^*:= &\inf & p(\M x) & & \\
          &\mathrm{s.t.} & g_j(\M x) = 0,& \quad \forall j \in \{1,\ldots,m\}, \\
\end{aligned}\tag{POP}
\end{equation}
where $p(\M x), g_j(\M x) \in \mathbb{R}[\M x]$. 
\end{definition}
In general, POPs are hard non-convex problems. In order to obtain a convex relaxation, 
we first rewrite the POP as a function of the moment matrix. 
The key insight here is that we can write a polynomial $p(\M x)$ of degree up to $2r$ as a linear function of the moment matrix $\M{M}_r(\M x)$, which contains all the monomials of degree up to $2r$:
\begin{equation}
    \label{eq:basis}
    p(\M x) = \operatorname{tr}(\M{C}\M{M}_r(\M x)) := \langle \M{C}, \M{M}_r(\M x) \rangle,
\end{equation}
where we observe that $\M{M}_r(\M x)$, by construction, is a positive semidefinite rank-1 matrix.
We then relax the problem by dropping the rank-one constraint on the moment matrix $\M{M}_r(\M x)$ while enforcing:
\begin{equation}
\begin{aligned}
    \M{X} \succeq 0, \M X_{1, 1} = 1.
\end{aligned}\tag{Semidefinite relaxation}
\end{equation}
Since the moment matrix $\M M_r(x)$ contains multiple repeated entries, we further add \emph{moment constraints} that enforce these repeated entries to be identical:
\begin{equation}
\begin{aligned}
    \langle \M B^{\perp}_i, \M X\rangle = 0. \\
\end{aligned}\tag{Moment constraints}
\end{equation}
Finally, we enforce the equality constraints, using \emph{localizing constraints} in the form: 
\begin{equation}
    \langle \M G_j, \M X \rangle = 0.
    \tag{Localizing constraints}
\end{equation}
These steps lead to the standard semidefinite moment relaxation of \eqref{prob:pop}, see~\cite{lasserre2001global, lasserre2015introduction}.
\begin{definition}[Moment Relaxation of~\eqref{prob:pop}~\cite{lasserre2001global, lasserre2015introduction}]
\begin{equation}
\label{prob:pop_sdp}
\begin{aligned}
    \hat{\rho}&:= \inf_{\M{X}} \langle \M{C}, \M{X} \rangle \\
    \mathrm{s.t.} \quad & \M{X} \succeq 0, \M{X}_{1,1} = 1\\
                  \quad & \langle \M{B}^{\perp}_{i}, \M{X}\rangle = 0, \langle \M{G}_{j}, \M{X}\rangle = 0 \quad \forall i, j.
\end{aligned}\tag{MOM}
\end{equation}
\end{definition}
If a rank-1 solution is obtained from \eqref{prob:pop_sdp}, its solution $\hat{\M X}$ is a valid moment matrix. 
In such a case, through SVD of $\hat{\M X}$ , we can extract $\hat{\M X} = \vv_r(\hat{\M x})\vv^{\transpose}_r( \hat{\M x})$ where $\M x$ is a provably optimal solution to~\eqref{prob:pop}.

An equivalent relaxation can be obtained by taking the dual of the moment relaxation, which leads to the following sum-of-squares (SOS)  relaxation:
\begin{definition}[SOS Relaxation of~\eqref{prob:pop}~\cite{parrilo2003semidefinite}]
\begin{equation}
\label{prob:pop_sdp*}
\begin{aligned}
    \hat{\rho}&:=\sup_{\rho, \sigma, \{h_j \}} \rho \\
    \mathrm{s.t.} \quad &p(\M x) - \rho = \sigma(\M x) + \sum_{j=1}^{m} h_j(\M x)g_j(\M x)\\
\end{aligned}
\tag{SOS}
\end{equation}
for some polynomial $\{h_j\}_{j=1}^{m}$ and SOS polynomial $\sigma$ satisfying: $ \operatorname{deg} \sigma(\M x) \le 2r, \operatorname{deg} h_j(\M x) g_j(\M x) \le 2r.$
\end{definition}

\section{Problem Formulation}
\label{sec:formulation}


We consider a polynomial dynamical system with additive noise written in implicit form as follows: 
    \begin{equation}
    \label{eq:polyDynSys}
\left\{
	\begin{array}{l}
	\M f( \M x_{k+1}, \M x_{k}, \M u_k ) = \M w_k \\
	\ \ \ \M h( \M y_k, \M x_{k} ) = \M v_k
	\end{array},  \qquad \M x_k \in \mathcal{K},\;\;\forall \; k
\right.
\end{equation}
where $\M x_k$ is the state, $\M y_k$ are the measurements, and $\M u_k$ is the control input, all at discrete time $k$; the state is restricted to the domain $\mathcal{K}$, e.g., the set of 2D poses.
Both the process model $\M f$ and the observation model $\M h$ are vector-valued real polynomials. 
We take the standard assumption that the process noise $\M w_k$ and measurement noise $\M v_k$ are zero-mean and identically and independently distributed across time steps.  

\begin{assumption}[$\mathcal{K}$] \label{ass:domain}
The domain $\mathcal{K}$ is either the Euclidean space $\mathcal{K}=\mathbb{R}^{n}$ (i.e., the state is unconstrained), or
can be described by polynomial equality constraints.
\end{assumption}

Assumption~\ref{ass:domain} is relatively mild and captures a broad set of robotics problems where the variables belong to semi-algebraic sets (e.g., rotations, poses); see for instance~\cite{carlone2022estimation}. 


Assuming the measurements $\M y_k$ and the controls $\M u_k$ to be known for all $k$ up to the current time,  and 
assuming the moments of the noises $\M v_k, \M w_k$ are also given, our goal is to estimate the state ${\M x}_k$.
Compared to the classical Kalman filter (and its variants) that only uses the mean (1st-order moment) and covariance (2nd-order moment) of the noise, 
our goal is to improve the state estimation by incorporating higher-order moments when the noise is non-Gaussian. 

\section{Batch Estimation in Polynomial Systems with~Arbitrary~Noise}
\label{sec:BPUE}

In this section, we consider a batch setup where we are given measurements from time $0$ to the current time $N$ and our goal is to estimate the state ${\M x}_0, \ldots, {\M x}_{N}$.
We first show how to account for the higher-order moments of the noise (Section~\ref{sec:extendedPolySys}) and then introduce an optimal estimator  (Section~\ref{sec:BPUEdetails}). 
For the remainder of the paper, { we use $\bar{(\cdot)}$ to denote the true state} and $\hat{(\cdot)}$ for the estimated state. 

\subsection{Extended Polynomial Dynamical Systems}
\label{sec:extendedPolySys}
To utilize the higher-order moments of the noise, we need an extended version of system \eqref{eq:polyDynSys} that exposes the moments of the noise in the state-space model.
To achieve this goal, we apply $\M \phi_r(\cdot)$ to both sides of eq.~\eqref{eq:polyDynSys} to generate monomials of degree up to $r$ from the entries of $\M w_k$ and $\M v_k$: 
\begin{equation}
\begin{aligned}
    \label{eq:ext-polyDynSys}
    &\left\{
	\begin{array}{l}
	\M \phi_r\left(\M f( \M x_{k+1}, \M x_{k}, \M u_k )\right) = \M \phi_r(\M w_k) \\
	\M \phi_r\left(\M h( \M y_k, \M x_{k} )\right) = \M \phi_r(\M v_k) \\
	\end{array}.
\right. \\
\end{aligned}
\end{equation}
Intuitively, eq.~\eqref{eq:ext-polyDynSys} combines the entries of the original process and measurement models into higher-order polynomials by taking powers and products of the entries of $\M f$ and $\M h$. Below, we refer to $\M \phi_r(\M w_k)$ and $\M \phi_r(\M v_k)$ as the \emph{extended noise}.

To design an unbiased estimator, we subtract the mean of the extended noises on both sides of \eqref{eq:ext-polyDynSys}, resulting in the following ``moment conditions'' of the extended system \eqref{eq:ext-polyDynSys}:\footnote{The astute reader may recognize that this is equivalent to the moment conditions imposed in the Generalized Method of Moments \cite{hansen1982large, hayashi2011econometrics}.} 
%
\begin{equation}
\label{gmm_condition}
\begin{aligned}
    \M m_{dyn}(\M x_{k+1}, \M x_{k}, \M u_k) &:= \M \phi_r(\M f(\M x_{k+1}, \M x_{k}, \M u_k)) - \mathbb{E}[\M \phi_r(\M w_k)]  
    \\
    & {\color{gray} \left( = \M \phi_r(\M w_k) - \mathbb{E}[\M \phi_r(\M w_k)] \right) }
    \\
    \M m_{obs}(\M y_k, \M x_k) &:= \M \phi_r(\M h(\M y_k, \M x_k)) -  \mathbb{E}[\M \phi_r(\M v_k)]
    \\
    & {\color{gray} \left(= \M \phi_r(\M v_k) - \mathbb{E}[\M \phi_r(\M v_k)] \right) }
\end{aligned}.
\end{equation}
By construction, $\M m_{dyn}$ and $\M m_{obs}$ are zero-mean vectors if $\M x_k$ and $\M x_{k+1}$ are taken at the true values $\bar{\M x}_k$ and $\bar{\M x}_{k+1}$. Then, we define the covariance of the extended noise on the RHS of \eqref{gmm_condition}:
\begin{equation}
\begin{aligned}
    \M Q_{r,k} &:= \operatorname{Cov}\left[\left(\M \phi_r(\vw_k) - \mathbb{E}\left[\M \phi_r(\vw_k)\right]\right)\right]  \\
    \M R_{r,k} &:= \operatorname{Cov}\left[\left(\M \phi_r(\vv_k) - \mathbb{E}\left[\M \phi_r(\vv_k)\right]\right)\right]
\end{aligned}.
\label{RQdef}
\end{equation}


Given the moment conditions~\eqref{gmm_condition} and the covariance of the extended noise~\eqref{RQdef}, 
we are now ready to define the optimal estimator for the batch case.
\begin{problem}[Optimal State Estimator in the Batch Setup] 
Given $N$ control inputs $\{\M u_0, \M u_1, \cdots, \M u_{N-1}\}$, measurements $\{\M y_0, \M y_1, \cdots, \M y_{N-1}\}$,
and covariance matrices of the extended noise 
    ${ \{ \M Q^{-1}_{r, 0}, \M R^{-1}_{r, 0}, \cdots, \M Q^{-1}_{r, N-1}, \M R^{-1}_{r, N-1}\} }$, 
the optimal estimator for the state $\M x = \{\M x_0, \M x_1, \cdots, \M x_N\}$ is the solution of the following optimization problem.
\label{prob:filter_batch}
    \begin{equation}
    \begin{aligned}
        \min_{ \{\M x_{0}, \dots \M x_N \} }  \sum_{k=0}^{N-1} & \| \M m_{dyn}(\M x_{k+1}, \M x_{k}, \M u_{k})\|^2_{\M Q^{-1}_{r, k}}
        + \| \M m_{obs}(\M y_k, \M x_{k})\|^2_{\M R^{-1}_{r, k}} \\
        + & \| \M \phi_r(\M x_0) - \M \phi_r(\hat{\M  x}_0)\|^2_{\M \Sigma^{-1}_{r, 0}} \\
        \mathrm{s.t.} & \quad \M x_{k} \in \mathcal{K}, ~\forall k \in \{0, 1, \dots, N\},
    \end{aligned}
\end{equation}
where $\| \M \phi_r(\M x_0) - \M \phi_r(\hat{\M  x}_0)\|^2_{\M \Sigma^{-1}_{r, 0}}$ is a given prior on the initial state at time $0$ such that $\M \phi_r(\M x_0) - \M \phi_r(\hat{\M  x}_0)$ is a zero-mean random variable with covariance $\M \Sigma_{r, 0}$. 
\end{problem}


\subsection{Best Polynomial Unbiased Estimator}
\label{sec:BPUEdetails}



Since the moment conditions $\M{m}_{dyn}$, $\M{m}_{obs}$, and the term $\M \phi_r(\M x_0) - \M \phi_r(\hat{\M  x}_0)$ in the prior in \Cref{prob:filter_batch} are vectors of polynomials, and under~\Cref{ass:domain}, \Cref{prob:filter_batch} is a polynomial optimization problem~\eqref{prob:pop}. 
Since each polynomial of degree up to $r$ can be written as a linear combination of the monomial basis $\M{\phi}_r(\M{x}), $\footnote{The non-expert reader can find examples in Appendix~\ref{appx:momcond_linear}.}
we can rewrite both the moment conditions and the prior in \Cref{prob:filter_batch} as linear expressions of the monomials in $\M{\phi}_r(\M{x})$:
\begin{equation}
\label{eq:gmm_cond}
    \M{b}(\M{y}) - \M{A}(\M{y})\M{\phi}_r(\M{x}) = \M{v}.
\end{equation}
The matrix $\M{A}(\M{y})$ and the vector $\M{b}(\M{y})$ only depend on the measurements ($\M{y}$ denotes the set of all measurements from time $0$ to the current time) and $\M{v}$ is zero-mean noise with covariance $\M{V}$, where $\M{V}$ is a block diagonal matrix with blocks corresponding to the noise covariances $\M Q^{-1}_{r, 0}, \M R^{-1}_{r, 0}, \cdots, \M Q^{-1}_{r, N-1}, \M R^{-1}_{r, N-1}$, and the prior covariance $\M \Sigma^{-1}_{r, 0}$. 
This leads to the definition of the
Best Polynomial Unbiased Estimator (\BPUE), which simply rewrites  \Cref{prob:filter_batch} in a more compact form.
\begin{definition}[Best Polynomial Unbiased Estimator] 
The {Best Polynomial Unbiased Estimator (\BPUE)} is the solution of the following optimization problem: 
\label{def:BPUE}
    \begin{equation}
    \label{eq:BPUE}
        \hat{\M x} = \argmin_{{\M x}\in \mathcal{K}} \|\M b(\M y) - \M A(\M y)\M \phi_r(\M x)\|^2_{\M{V}^{-1}}.
        \tag{\BPUE}
    \end{equation}
\end{definition}
For simplicity, we omit the dependency of $\M A(\M y)$ and $\M b(\M y)$ on $\M y$ in the objective of \eqref{eq:BPUE} in the rest of this paper.
\begin{proposition}
    \eqref{eq:BPUE} is an unbiased estimator of $\bar{\M x}$.
\end{proposition}


\subsection{From Quadratic to SOS Beliefs}

Problem~\eqref{eq:BPUE} is a \eqref{prob:pop} so we can apply the standard moment relaxation \eqref{prob:pop_sdp} to search for a globally optimal solution. However, we typically want to obtain a belief representation for the estimate rather than a point estimate. Below we show that this can be done by interpreting the computation of the belief in terms of a factorization of the objective of~\eqref{eq:BPUE}.
We make this point by drawing a parallel with the \BLUE, the \emph{Best Linear Unbiased Estimator}, which is the solution to standard linear least squares.

{\bf \BLUE and Quadratic Belief.} 
Consider for an instant the standard linear Gaussian setup (without state constraints) used in classical Kalman Filtering. In such a case,~\eqref{eq:BPUE} with $r=1$ simplifies to standard linear least squares:
    \begin{equation}
    \label{eq:BLUE}
        \hat{\M x} = \argmin_{{\M x}}\quad \|\M b - \M A \M x\|^2_{{\M V}^{-1}}, \tag{\BLUE}
    \end{equation}
which admits as a solution the well-known Best Linear Unbiased Estimator $\hat{\M x} = (\M A^\transpose \M{V}^{-1}\M A)^{-1} \M A^\transpose \M{V}^{-1} \M b$ with covariance $ \M \Sigma = (\M A^\transpose \M{V}^{-1} \M A)^{-1}$.  
While the derivation of the covariance $\M \Sigma$ is typically obtained via probabilistic arguments (i.e., by linearly propagating the noise covariance $\M V$), here we provide an alternative optimization-theoretic interpretation of the 
computation of the estimator $\hat{\M x}$ and its covariance $\M \Sigma$.
Towards this goal, we first need to define the notion of quadratic belief.

\begin{definition}[Quadratic Belief]
\label{def:quadratic-belief}
Given the estimate $\hat{\M x}$ and the covariance ${\M{\Sigma}} \succ 0$, the quadratic belief is defined as:\footnote{The term ``quadratic'' belief is admittedly a misnomer: while in the linear Gaussian case, the quadratic form does define the Gaussian belief,  the quadratic form is actually the negative log-likelihood of the belief.} 
    \begin{equation*}
        \| \M x - \hat{\M x} \|^2_{{\M{\Sigma}}^{-1}}. 
    \end{equation*}
\end{definition}
\begin{restatable}[Quadratic Belief of \BLUE]{proposition}{propbluequad}
\label{prop:blue-quad}
    Assuming the matrix $\M A$ is full column rank, the cost function of~\eqref{eq:BLUE} can be factorized 
    into a quadratic belief:   
    \begin{equation}
        \| \M{A}\M{x} - \M{b} \|^2_{\M{V}^{-1}} = \| \M x - \hat{\M x} \|^2_{\M{\Sigma}^{-1}} + \hat{\rho},
    \end{equation}
   where $\hat{\M{x}} =  \M{\Sigma}\M{A}^{\transpose}\M{V}^{-1}\M{b}$ and $\M{\Sigma} = ( \M{A}^{\transpose}\M{V}^{-1}\M{A} )^{-1}$ are the \ estimate and its covariance respectively, and $\hat{\rho} \geq 0$ is the optimal objective of~\eqref{eq:BLUE}.
\end{restatable}

{\bf \BPUE and SOS Belief.} 
We are now ready to generalize~\Cref{def:quadratic-belief} and \Cref{prop:blue-quad} to the polynomial case with arbitrary noise. 
Towards this goal, we generalize~\Cref{def:quadratic-belief} to account for higher-degree polynomials.
\begin{definition}[SOS Belief]
\label{def:SOS-factor}
Given the estimate $\hat{\M x}$ and ${\M \Sigma} \succ 0$, the SOS belief $\sigma(\hat{\M x}, {\M \Sigma})$ is an SOS polynomial defined as:
    \begin{equation*}
        \sigma(\hat{\M x}, {\M \Sigma}) := \| \M \phi_r(\M x) - \M \phi_r(\hat{\M x}) \|^2_{{\M \Sigma}^{-1}}.
    \end{equation*}
\end{definition}
The SOS belief is a generalization of the quadratic belief (whenever $r>1$) and simplifies to a quadratic belief for $r=1$. 
We generalize \Cref{prop:blue-quad} as follows.
\begin{theorem}[SOS Belief of Unconstrained \BPUE]
\label{theorem:cov_mkf}
In the unconstrained case ($\mathcal{K}$ is the Euclidean space), 
under suitable rank conditions (formalized below in~\Cref{def:kkt_sdp}), 
the cost function of \eqref{eq:BPUE} can be factorized into an SOS belief: $$\|\M A \M \phi_r(\M x) - \M b\|^2_{\M{V}^{-1}} = \|\M \phi_r(\M x) - \M \phi_r(\hat{\M x})\|^2_{\M \Sigma^{-1}} + \hat{\rho}, $$
where $\hat{\M x}$ is the optimal solution of \eqref{eq:BPUE},  $\M \Sigma$ is a suitable positive definite matrix, and
$\hat{\rho} \geq 0$ is the optimal objective of~\eqref{eq:BPUE}.
\end{theorem}
\Cref{theorem:cov_mkf} extends \Cref{prop:blue-quad} to higher-order polynomials. A constrained version of \Cref{theorem:cov_mkf} (where $\mathcal{K}$ are polynomial equality constraints) is given in Appendix~\ref{appx:bpue-sos-cons}. 
We have already mentioned that we can compute $\M \phi_r(\hat{\M x})$ using a moment relaxation.
In the next section, we show that the matrix $\M \Sigma$ can be computed in closed-form from the \emph{dual} solution of problem~\eqref{eq:BPUE}, hence we have computational ways to obtain the exact SOS belief in the batch setup. 
\section{Theoretical Analysis}
\label{sec:theory}

In this section, we formalize that the SOS belief in~\Cref{theorem:cov_mkf} is such that $\M \phi_r(\hat{\M x})$ can be computed from a moment relaxation of~\eqref{eq:BPUE} while the matrix $\M \Sigma$ can be obtained from the dual (SOS) relaxation of~\eqref{eq:BPUE}. Moreover, we show that while the matrix $\M \Sigma$ does not admit the same interpretation as in the linear Gaussian case (where it is the covariance of the estimate $\hat{\M x}$), it can still be considered an approximate covariance.

\subsection{The SOS Belief Can Be Computed by Solving the Moment and SOS Relaxations}
The key result of this subsection, given in~\Cref{theorem:SDP=POP}, states that we can compute the SOS belief that factorizes the~\eqref{eq:BPUE} as in~\Cref{theorem:cov_mkf} by solving a moment relaxation of~\eqref{eq:BPUE} and its dual (SOS) relaxation.
To derive this result, we first need to state the following optimality conditions for the moment relaxation~\eqref{prob:pop_sdp}:
\begin{definition}[Optimality Condition of Rank-1 \eqref{prob:pop_sdp}]
Consider the Lagrangian of \eqref{prob:pop_sdp} as 
\begin{equation*}\label{def:kkt_sdp}
\begin{aligned}
 L(\M X, \M Y, \M \lambda, \M \mu, \rho) = &\langle \M C, \M X \rangle - \langle \M Y, \M X \rangle +  \sum_{j} \lambda_{j} \langle \M G_{j}, \M X \rangle  \\
+ \sum_{i} \mu_{i} &\langle \M B^{\perp}_{i},  \M X\rangle  + \rho(1 - \langle \M e_1\M e_1^{\transpose}, \M X \rangle),
\end{aligned}
\end{equation*}
where $\M Y$, $\M \lambda$, $\M \mu$ and $\rho$ are the multipliers.
By the stationary condition $\nabla_{\M X} L = 0$ and the feasibility of primal and dual problems, we have the following optimality conditions:
\begin{enumerate}
    \item $\M Y = \M C + \sum_{j} \lambda_{j} \M G_{j} + \left( \sum_{i}\mu_{i} \M B^{\perp}_{i} - \rho \M e_1\M e_1^{\transpose} \right)$, 
    \item $\langle \M G_{j}, \M X \rangle = 0, \langle \M B^{\perp}_{i}, \M X \rangle = 0$, $\langle \M e_1\M e_1^{\transpose}, \M X \rangle = 1$,
    \item $\M Y \succeq 0, \M X \succeq 0$, $\langle \M X, \M Y\rangle = 0$.
\end{enumerate}
We also require an additional rank condition \cite{cosse2021stable, cifuentes2020geometry}:
\begin{enumerate}[~~4)]
    \item $ Rank(\M X) = 1, Rank(\M Y) = s(r, n) - 1$.
\end{enumerate}
\end{definition}

In the remainder of this paper, we make the following assumption to focus on the case that the moment relaxation \eqref{prob:pop_sdp} returns rank-1 solutions.
\begin{assumption}[Rank-1 Solution of Moment Relaxation]
\label{assump:rank1sdp}
    The solutions to the moment relaxations solved in this work satisfy the strict optimality conditions in \Cref{def:kkt_sdp}. 
\end{assumption}

Inspired by \cite{cosse2021stable}, we derive the following theorem that connects primal-dual optimal solutions satisfying the optimality condition in \Cref{def:kkt_sdp}  to the SOS belief:

\begin{restatable}[Moment-SOS Relaxations and SOS Belief]{theorem}{theoremSDPPOP}
\label{theorem:SDP=POP}
Suppose the unique optimal solution of \eqref{prob:pop} is $\hat{\M x}$, and the \eqref{prob:pop_sdp} corresponding to \eqref{prob:pop} produces a rank-1 optimal solution $\hat{\M X}$. Then $\hat{\M X} = \M v_r(\hat{\M x})\M v_r(\hat{\M x})^{\transpose}$. Moreover, the dual solution $\hat{\M Y}$ takes the following form
\begin{equation*}
    \hat{\M Y}= \begin{bmatrix}
            \M \phi^{\transpose}_r(\hat{\M x})\M{\Sigma}^{-1} \M \phi_r(\hat{\M x}) & -\M \phi^{\transpose}_r(\hat{\M x})\M{\Sigma}^{-1}\\
            -\M{\Sigma}^{-1}\M \phi_r(\hat{\M x}) & \M{\Sigma}^{-1}
        \end{bmatrix}, \M \Sigma \succ 0,
\end{equation*}
if and only if the strict optimality conditions in \Cref{def:kkt_sdp} are satisfied. Furthermore, 
the resulting SOS belief $\sigma(\hat{\M x}, \M{\Sigma})$
factorizes the objective~\eqref{eq:BPUE} as described in~\Cref{theorem:cov_mkf}.
\end{restatable}


A detailed proof of \Cref{theorem:SDP=POP} is given in Appendix~\ref{apx:SDP=SOS}. The proof is a generalization of similar results proposed to tackle the matrix completion problem in \cite{cosse2021stable}. 

\subsection{Characterization of the Estimation Error in \BPUE}
\label{sec:BPUE_cov}
So far we have derived the SOS belief using purely optimization-theoretic arguments.
We now proceed to characterize the estimation error of \eqref{eq:BPUE}.
In particular, we show how to write the estimate $\M \phi_r(\hat{\M x})$ as a function of the ground truth $\M \phi_r(\M {x})$ plus an additional perturbation.
For simplicity, we analyze \eqref{eq:BPUE} without equality constraints, i.e. $\mathcal{K} = \mathbb{R}^n$ and defer the derivation for the constrained case to Appendix \ref{appx:bpue-sos-cons}.


Under \Cref{assump:rank1sdp}, we use condition 1) in \Cref{def:kkt_sdp} to construct the dual solution $\hat{\M Y}$. We first partition the $\M B^{\perp}$ and the cost matrix $\M{C}$ of \eqref{eq:BPUE} to separate the first row and column from the remaining principal submatrix: 
\begin{equation*}
    \M B^{\perp} = \begin{bmatrix}
        0 & \M B^{\perp \transpose}_{l} \\
        \M B^{\perp}_{l}& \M B^{\perp}_{Q} 
    \end{bmatrix}, \M C = \begin{bmatrix} \M b^{\transpose}\M V^{-1}\M b & (-{\M A}^{\transpose}\M V^{-1}{\M b})^{\transpose} \\
        -{\M A}^{\transpose}\M V^{-1}{\M b} & {\M A}^{\transpose}\M V^{-1}{\M A}
    \end{bmatrix}.
\end{equation*}
Then, the dual solution $\hat{\M Y}$ becomes:
\begin{equation}
\label{eq:BPUE_dual_sol}
\hat{\M Y} = \begin{bmatrix} \M b^{\transpose}\M V^{-1}\M b - {\rho} & (-{\M A}^{\transpose}\M V^{-1}{\M b} + \sum_i{\mu}_i \M B^{\perp}_{l, i})^{\transpose} \\
       * & {\M A}^{\transpose}\M V^{-1}{\M A} + \sum_i {\mu}_i\M B_{\M Q,i}^{\perp}
    \end{bmatrix}. 
\end{equation}
By the structure of $\hat{\M Y}$ shown in \Cref{theorem:SDP=POP}, we have:
\begin{equation}
    \M{\Sigma}^{-1} = {\M A}^{\transpose}\M V^{-1} {\M A} + \sum_i {\mu}_i\M B_{ Q,i}^{\perp},
\end{equation}
and we can directly extract the optimal estimate as:
\begin{equation}
\label{eq:BPUE_estimate}
\begin{aligned}
    \M \phi_r(\hat{\M x}) = {\M \Sigma} ({\M A}^{\transpose}\M V^{-1} {\M b} - \sum_i{\mu}_i \M B^{\perp}_{l, i}) \\
\end{aligned},
\end{equation}
which follows from the relation 
$- \M \Sigma \M \phi_r(\hat{\M x}) = -{\M A}^{\transpose}\M V^{-1} {\M b} + \sum_i{\mu}_i \M B^{\perp}_{l, i}$ in \Cref{theorem:SDP=POP}. 
We observe that the expression of the~\eqref{eq:BPUE} quite elegantly resembles the expression of the optimal estimate and covariance of~\eqref{eq:BLUE}, but it requires adding a ``correction'' controlled by the dual variables $\M \mu$.

Substituting $\M b = {\M A}\M \phi_r(\bar{\M x}) + \M v$ from \eqref{eq:gmm_cond} in~\eqref{eq:BPUE_estimate} we get:
\begin{equation}
\label{eq:BPUE_estimate2}
\begin{aligned}
   \M \phi_r(\hat{\M x}) &= {\M \Sigma}({\M A}^{\transpose}\M V^{-1}({\M A}\M \phi_r(\bar{\M x}) + \M v) - \sum_i{\mu}_i \M B^{\perp}_{l, i})
    \\
    & = {\M \Sigma}{\M A}^{\transpose}\M V^{-1} {\M A}\M \phi_r(\bar{\M x}) + 
    {\M \Sigma}{\M A}^{\transpose}\M V^{-1} \M v - \sum_i{\mu}_i {\M \Sigma}\M B^{\perp}_{l, i}.
\end{aligned}
\end{equation}
Let us call $\delta \M{\phi}$ the second term in~\eqref{eq:BPUE_estimate2}:
\begin{equation}
\label{eq:perturb_v}
\delta \M{\phi} := \M{\Sigma}{\M A}^{\transpose}\M V^{-1}\M v.
\end{equation}
In an ideal world, we would like the previous equation to read as $\M \phi_r(\hat{\M x}) = \M \phi_r({\bar{\M x}}) + \delta \M{\phi}$ and derive a covariance for $\delta \M{\phi}$ (this would be the covariance of the estimation error in the moment space). Unfortunately, there are two complications: (i) there are several terms multiplying $\M \phi_r({ \bar{\M x}})$ on the right-hand-side and (ii) the matrices $\M \Sigma$ and $\M A$ in~\eqref{eq:perturb_v} also depend on the noise realization ($\M A$ is a  function of the measurements, while $\M \Sigma$ depends on the multipliers $\M \mu$ which are also a function of the measurements) hence we cannot simply do a linear propagation of the noise from $\M v$ to $\delta \M{\phi}$ as we do in the linear Gaussian case. In the following section, we show that in the limiting case of noiseless measurements, we can indeed obtain an estimation error covariance.

\subsection{Approximate Covariance of \BPUE}
\label{sec:theory-covariance}

To show that in the limiting case of noiseless measurements, the matrix appearing in the SOS belief can be interpreted as a covariance of the estimation error.
Let us consider the noiseless case of \eqref{eq:gmm_cond}, i.e., 
$\bar{\M A} := \mathbb{E}[\M A(\M y)], \bar{\M b} := \mathbb{E}[\M b(\M y)]$. By $\mathbb{E}[\M{v}] = 0$,
the following linear equation holds for the true state $\bar{\M x}$:
\begin{equation}
\label{eq:Aphix=b}
   \bar{\M b} = \bar{\M A}\M \phi_r(\bar{\M x}) + \mathbb{E}[\M{v}] = \bar{\M A}\M \phi_r(\bar{\M x}).
\end{equation}

We also need the following asssumption for our analysis.
\begin{assumption}
\label{assump:observability}
    $\bar{\M A}$ has full column rank.
\end{assumption}
While this assumption can be interpreted as an observability condition, in this context we mostly use it to simplify the expression of the~\eqref{eq:BPUE}. In particular, we show that under this assumption choosing $\M \mu = \M 0$ produces a valid dual solution. 
In particular, 
if we substitute $\M A = \bar{\M A}$ and $\M b = \bar{\M b} = \bar{\M A}\M \phi_r(\bar{\M x})$ to \eqref{eq:BPUE_dual_sol} and set $\M \mu = \M 0$:
        \begin{equation}
            \hat{\M Y} = \begin{bmatrix} \bar{\M b}^{\transpose}\M{V}^{-1}\bar{\M b} - \rho & (-{\bar{\M A}}^{\transpose}\M{V}^{-1}{\bar{\M b}})^{\transpose} \\
        -{\bar{\M A}}^{\transpose}\M{V}^{-1}{\bar{\M b}} & {\bar{\M A}}^{\transpose}\M{V}^{-1}{\bar{\M A}}
        \end{bmatrix}, \rho = 0,
        \end{equation}
which can be easily seen to satisfy the optimality conditions in~\Cref{def:kkt_sdp}.\footnote{
To verify the strict complementary condition and the rank condition, we observe that, by \Cref{assump:observability}, $\M{\phi}_r(\bar{\M{x}})$ is the unique solution to \eqref{eq:Aphix=b}. It is then straightforward to see $\langle \M{v}_r(\bar{\M{x}})\M{v}_r(\bar{\M{x}})^{\transpose}, \hat{\M{Y}} \rangle = 0$ by substituting $\bar{\M{b}} = \bar{\M{A}}\M{\phi}_r(\bar{\M{x}})$ in $\hat{\M{Y}}$. $\operatorname{Rank}(\hat{\M{Y}}) = s(r, n) - 1$ can be verified using the additional observation that $\bar{\M{A}}^{\transpose}\M{V}^{-1}\bar{\M{A}}$ is invertible under \Cref{assump:observability}.

} 

    As $\bar{\M A}$ has full column rank, $\bar{\M \Sigma}^{-1} = \bar{\M A}^{\transpose}\M{V}^{-1} \bar{\M A}$ is invertible. 
    Substituting the expression of $\bar{\M \Sigma}^{-1}$ and  the fact that $\M \mu = \M 0$ in~\eqref{eq:BPUE_estimate2} we get:
    \begin{equation}
    \begin{aligned}
        \M \phi_r(\hat{\M x}) = \M \phi_r(\bar{\M x}) + \delta \M{\phi}, 
    \end{aligned}
    \end{equation}
where $\delta \M{\phi} = \bar{\M \Sigma} \bar{\M A}^{\transpose}\M V^{-1}\M v$. Now the 
coefficient matrix of the noise $\M v$ is a constant matrix, hence we can compute its covariance as:
\begin{equation}
\label{eq:op_eq_cov}
\begin{aligned}
       \operatorname{Cov}[\delta \M \phi] &= \operatorname{Cov}[\bar{\M{\Sigma}}{ \bar{\M A} }^{\transpose}\M{V}^{-1}\M v] \\
       &= \operatorname{Cov}[(\bar{\M A}^{\transpose}\M{V}^{-1} \bar{\M A})^{-1}{ \bar{\M A} }^{\transpose}\M{V}^{-1}\M v] \\
       &=(\bar{\M A}^{\transpose}\M V^{-1}\bar{\M A})^{-1} =\bar{\M{\Sigma}}
\end{aligned}
\end{equation}
proving that in the limiting case, the matrix $\bar{\M\Sigma}$ arising in the SOS belief can be interpreted as a covariance of the moments of the estimation error.

\section{Generalized Moment Kalman Filter}
\label{sec:GMKF}

The batch setup quickly becomes intractable over time, since the size of the state $\{\M x_0, \M x_1, \cdots, \M x_N\}$ and the resulting moment relaxation grows with the time $N$. We now leverage insights from the \eqref{eq:BPUE} to design a \emph{filter} that recursively estimates the current state $\M x_{k+1}$ given the belief at time $k$ and the most recent control and measurements. 
We call our filter the \emph{\GMKFlong} (\GMKF).

The working of the \GMKF follows the template of the classical Kalman Filter. 
In the classical KF, the prediction step uses the (quadratic) belief at time $k$ and the control action to compute a prior belief at time $k+1$. Then the update step uses the prior belief at time $k+1$ and incorporates the measurements to produce a new posterior belief.
Similarly, the \GMKF includes a prediction step that uses an SOS belief at time $k$ and the control action to compute a prior SOS belief at time $k+1$, and an update step that uses the prior SOS belief at time $k+1$ and incorporates the measurements to produce a new posterior SOS belief. However, while in the KF both prediction and update steps can be computed in closed form, 
the corresponding steps in the \GMKF require solving a moment and SOS semidefinite relaxations.
 In the following, we denote the state estimate after 
the prediction step with the superscript $(\cdot)^{-}$, and the state estimate after the update step with the
 superscript $(\cdot)^{+}$, following standard notation. 
 
We describe each step in the \GMKF, whose pseudocode is also given in~\Cref{alg:mkf}.


\subsubsection{Initialization}
At the initial time $0$ we are given a prior over the state of the system that we assume described by an SOS belief:
\begin{equation*}
        \sigma(\hat{\M x}_0, \hat{\M \Sigma}_0).
\end{equation*}
As we mentioned above, the SOS belief may include high-order moments, hence it is a strict generalization of the quadratic (Gaussian) belief, which is indeed recovered when $r=1$ in the SOS belief.

\subsubsection{Update}\!\!\footnote{Note that the order of prediction and update is irrelevant: we present the update first to be consistent with the notation in~\eqref{eq:polyDynSys}, where the measurements are collected at time $k$ and the prediction propagates the state to time $k+1$.}
The update step takes the SOS belief at time $k$ and incorporates the latest measurements $\M y_k$ in the SOS belief.
In particular, given the SOS belief at time $k$ and observation $\M m_{obs}(\M y_k, \M x_{k})$, 
the update step solves a moment relaxation and its (dual) SOS relaxation to get an SOS belief: 
\label{def:gmkf_update}
    \begin{equation*}
    \begin{aligned}
        &\sigma\left(\hat{\M  x}^{+}_k, \hat{\M  \Sigma}_k^{+}\right) \xleftarrow[]{\mathrm{BPUE}} \sigma\left(\hat{\M  x}^{-}_k, \hat{\M  \Sigma}_k^{-}\right) + \|\M m_{obs}(\M y_k, \M x_{k}) \|^2_{\M R_k^{-1}} 
    \end{aligned}.
    \label{update}
    \end{equation*}
Note that this is fully analogous to the KF, where the update phase can also be interpreted as an optimization problem over the belief and the measurements; however, in the KF both terms in the optimization are quadratic and the solution can be computed in closed form.

\subsubsection{Prediction}
The prediction step takes the SOS belief at time $k$ and ``propagates'' it to time $k+1$ using the system dynamics.
In particular, given the SOS belief at time $k$ and observation $\M m_{obs}(\M y_k, \M x_{k})$, 
the update step solves a moment relaxation and its (dual) SOS relaxation to get an SOS belief. 

In particular, given the SOS belief at time $k$ and an input $\M u_k$ in the dynamics $\M m_{dyn}(\M x_{k+1}, \M x_{k}, \M u_k)$, the prediction step solves a moment relaxation and its (dual) SOS relaxation to get an SOS belief over the two most recent states $\begin{bmatrix}\hat{\M x}^{-}_k \\ \hat{\M x}^{-}_{k+1}\end{bmatrix}$:
\label{def:gmkf_predict}
    \begin{equation*}
    \begin{aligned}
         & \sigma\left(\begin{bmatrix}\hat{\M x}^{-}_k  \\ \hat{\M x}^{-}_{k+1}\end{bmatrix}, \hat{\M \Sigma}_{k:k+1}^{-}\right) \\
         \xleftarrow[]{\mathrm{BPUE}}~ &\sigma\left(\hat{\M x}_k^{+}, \hat{\M \Sigma}_k^{+}\right) + \|\M m_{dyn}(\M x_{k+1}, \M x_{k}, \M u_k)\|^2_{\M Q_k^{-1}}
    \end{aligned}.
    \label{prediction}
    \end{equation*}
Again, this is in complete analogy with the classical KF where also the prediction step can be interpreted as (a linear least squares) optimization over the prior at time $k$ and the system dynamics governing the transition between time $k$ and $k+1$.

While the prediction has produced an SOS belief over both time $k$ and $k+1$,
to recursively compute the belief, we need to restrict the belief to the state at time $k+1$.
Interestingly, while this step is straightforward for the KF (marginalization of a Gaussian distribution only requires dropping certain entries from the mean and covariance of the joint distribution), it constitutes a roadblock in the nonlinear case, and it is the main reason why we cannot claim optimality of the \GMKF in general.
In the current algorithm, we leverage the interpretation in Section~\ref{sec:theory-covariance} and treat the SOS belief produced by the prediction step as a Gaussian (in the moment space) with mean $\begin{bmatrix}\hat{\M x}^{-}_k \\ \hat{\M x}^{-}_{k+1}\end{bmatrix}$ and covariance $\hat{\M \Sigma}_{k:k+1}^{-}$. 
Therefore, 
We partition $\M \phi_r \left(\begin{bmatrix}\M x_{k}\\ \M x_{k+1}\end{bmatrix} \right)$ (up to permutation) in three parts
    \begin{equation*}
        \M \phi_r\left(\begin{bmatrix}\M x_{k} \\ \M x_{k+1}\end{bmatrix}\right) =: \begin{bmatrix}
            \M \phi_r(\M x_{k+1}) \\
            \M \phi_r(\M x_{k}) \\
            c(\M x_{k}, \M x_{k+1})
        \end{bmatrix},
    \end{equation*}
   where $c(\cdot, \cdot)$ denotes the cross monomial terms involving entries of both $\M x_{k}$ and $\M x_{k+1}$. 
   We also partition $\hat{\M \Sigma}_{k:k+1}^{-}$ accordingly:
    \begin{equation*}
    \begin{aligned}
        \hat{\M \Sigma}_{k:k+1}^{-} &=: \begin{bmatrix}
            \M \Sigma_{k+1, k+1} & \M \Sigma_{k+1, k} & \M \Sigma_{k+1, c} \\
            * & \M \Sigma_{k, k} & \M \Sigma_{k, c} \\
            * & * & \M \Sigma_{c, c} \\
        \end{bmatrix} \\
    \end{aligned}.
    \end{equation*}
    Then, we only keep the principal submatrix corresponding to the vector $\M \phi_r( \M x_{k+1})$ and set:
    \begin{equation*}
        \hat{\M \Sigma}^{-}_{k+1}\leftarrow \M \Sigma_{k+1, k+1}.
    \end{equation*}
    To obtain the SOS belief for the next step, we discard the cross terms and only keep the terms at time $k+1$:
    \begin{equation*}
        \sigma\left(\hat{\M x}^{-}_{k+1}, \hat{\M \Sigma}^{-}_{k+1}\right).
    \end{equation*}






\begin{algorithm}[t]
\label{alg:general}

\caption{Generalized Moment Kalman Filter}\label{alg:mkf}

\begin{algorithmic}
\footnotesize
\Require Initial SOS belief: $\sigma(\hat{\M x}_0, \hat{\M \Sigma}_0)$, 
process model: $\M f(\M x_{k+1}, \M x_k, \M u_k) = \M w_k$,
measurement model: $\M h(\M y_{k}, \M x_k) = \M v_k$,
{degree of moments of the noise: $r$, noise covariance in ~\eqref{RQdef}: $\M Q_k, \M R_k$}.  

\For{time $k = 0, 1, \ldots$}
    \State {\texttt{// Update}}
    \State {Receive observation $\M y_k$}
    \begin{equation*}
    \begin{aligned}
        \sigma\left(\hat{\M  x}^{+}_k, \hat{\M  \Sigma}_k^{+}\right) \xleftarrow[]{\mathrm{BPUE}} \sigma\left(\hat{\M  x}^{-}_k, \hat{\M  \Sigma}_k^{-}\right) + \|\M m_{obs}(\M y_k, \M x_{k}) \|^2_{\M R_k^{-1}} 
    \end{aligned}
    \end{equation*}

    \State {\texttt{// Prediction}}
    \State {Receive input $\M u_k$}
    \begin{equation*}
    \begin{aligned}
         \sigma\left(\begin{bmatrix}\hat{\M x}^{-}_k  \\ \hat{\M x}^{-}_{k+1}\end{bmatrix}, \hat{\M \Sigma}_{k:k+1}^{-}\right) 
         &\xleftarrow[]{\mathrm{BPUE}}~ \sigma\left(\hat{\M x}_k^{+}, \hat{\M \Sigma}_k^{+}\right) \\
         & + \|\M m_{dyn}(\M x_{k}, \M x_{k+1}, \M u_k)\|^2_{\M Q_k^{-1}}
    \end{aligned}
    \end{equation*}
    
    \State {\texttt{// ``Marginalization''}}
    \State {$\sigma\left(\hat{\M x}^{-}_{k+1}, \hat{\M \Sigma}^{-}_{k+1}\right)\gets \sigma\left(\begin{bmatrix}\hat{\M x}^{-}_k  \\ \hat{\M x}^{-}_{k+1}\end{bmatrix}, \hat{\M \Sigma}_{k:k+1}^{-}\right) $ 

\EndFor

\noindent \Return $\{\hat{\M x}_k^+ \mid \forall k \}$
}
\end{algorithmic}
\end{algorithm}

\begin{restatable}[Generalized Kalman Filter]{theorem}{thmgenKF}
\label{thm:gen_KF}
In the unconstrained linear Gaussian case, Algorithm~\ref{alg:mkf} with $r=1$ produces the same solution as the standard Kalman Filter and hence it is an optimal estimator for linear Gaussian systems. 
\end{restatable}
A proof of the theorem and more comparisons with the classical KF are presented in Appendix \ref{apx:linearkf}. 

\section{Numerical Experiments}
To demonstrate the efficacy of \GMKF in robotics applications, we design numerical experiments to evaluate its performance and benchmark it against several KF variations as baselines. In all our experiments, 
the moment relaxation produced rank-1 optimal solutions in both the batch case and when solving the prediction and update step 
of the
\GMKF.
\subsection{Comparison of \BPUE and \BLUE}
We first compare \BPUE with \BLUE for linear systems contaminated by non-Gaussian noise. Though \BLUE is optimal when the system is linear and the noise is Gaussian, we show that \BPUE outperforms \BLUE when the noise is non-Gaussian since it considers higher-order moments of the noise. In particular, we consider the moments of the noise up to the fourth order, i.e., we choose $r = 2$ in \eqref{eq:BPUE}. 

We consider the linear measurement model:
\begin{equation}
\begin{aligned}
    y_1  &= x_1 +  v_1 \\
    y_2  &= x_2 +  v_2
\end{aligned}
\end{equation}
for $\M x, \M y \in \mathbb{R}^2$ and the noise $\M v \in \mathbb{R}^2$ with known moments. We choose the true value as $x_1 = x_2 = 0$. We design a binary noise distribution and a trigonometric noise distribution that are highly non-Gaussian. The binary noise 
is defined as:
\begin{equation}
    \vv_k = s\cdot\begin{bmatrix}
        q_1-0.5\\
        q_2-0.5
    \end{bmatrix} + \M \epsilon ,
\end{equation}
where $q_1, q_2\sim \text{Bernoulli}(0.5)$, $\M \epsilon\sim \text{Gaussian}(0, 0.1\M I)$, and $s$ is a scale factor controlling the magnitude of the noise.
In words, $q_1$ and $q_2$ are either 0 (with probability 0.5) or 1, while $\M \epsilon$ is a small zero-mean additive Gaussian noise with covariance $0.1 \M I$. This produces a multi-modal distribution with 4 modes, as visualized in \Cref{fig:binnoise} for $s=1$. 

The trigonometric noise is defined as:
\begin{equation}
    \vv_k = s\cdot\begin{bmatrix}
    \cos(q\pi)\\
    \sin(q)
    \end{bmatrix}+ \M \epsilon
\end{equation}
where $q\sim \text{Uniform}(-\pi, \pi)$, $\M \epsilon\sim \text{Gaussian}(0, 0.1\M I)$, and $s$ is a scale factor. 
The trigonometric noise is a highly non-Gaussian multi-modal distribution, as visualized in \Cref{fig:trignoise} for $s=1$. 

\begin{figure}[th]
    
\centering
\begin{minipage}{0.24\textwidth}
    \centering
    \includegraphics[width=\linewidth, trim={0.8cm 3cm 1cm 4cm},clip]{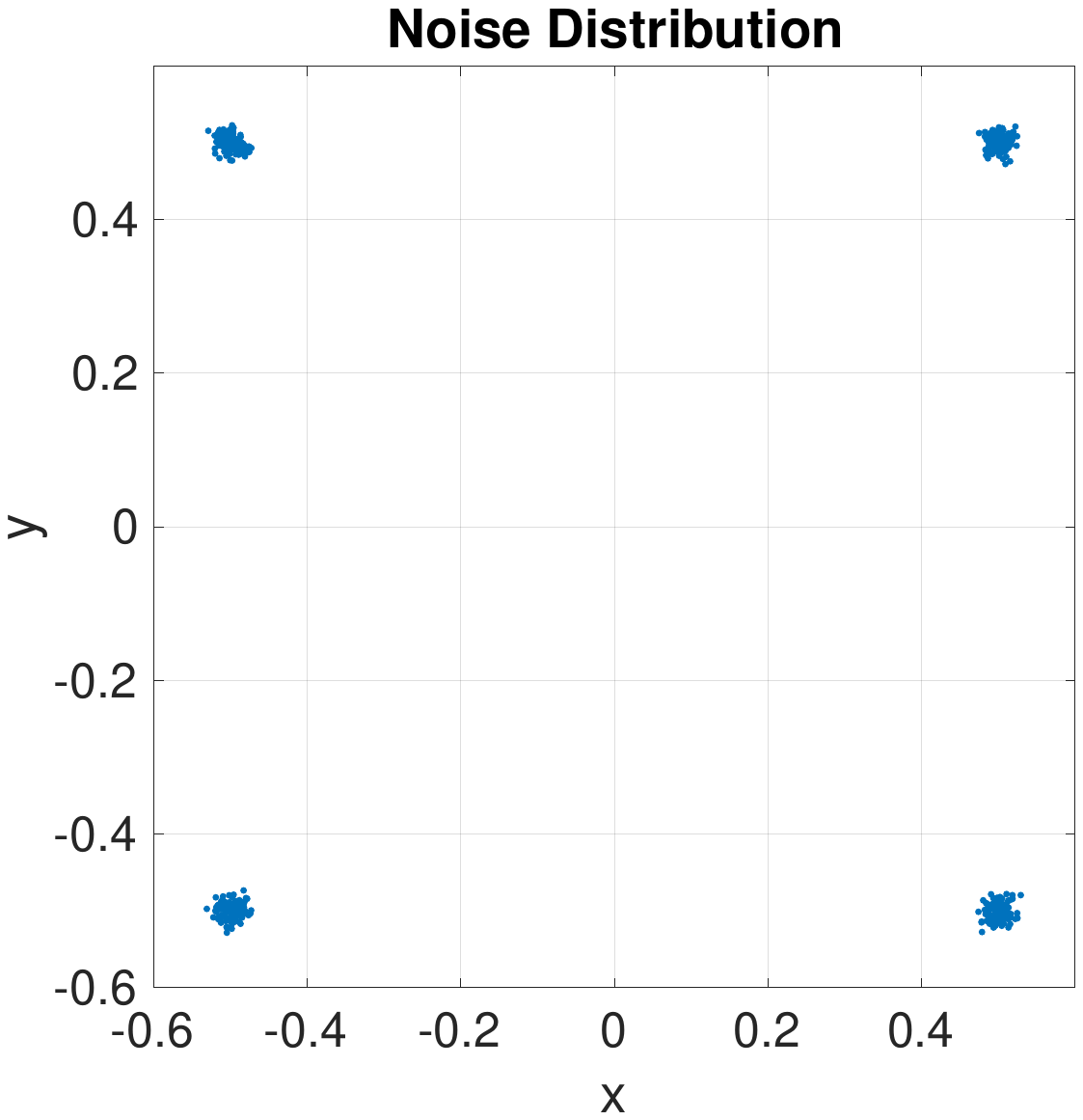} 
    \caption{Binary noise}
    \label{fig:binnoise}
\end{minipage}
\hfill
\begin{minipage}{0.24\textwidth}
    \centering
    \includegraphics[width=\linewidth, trim={0.8cm 3cm 1cm 4cm},clip]{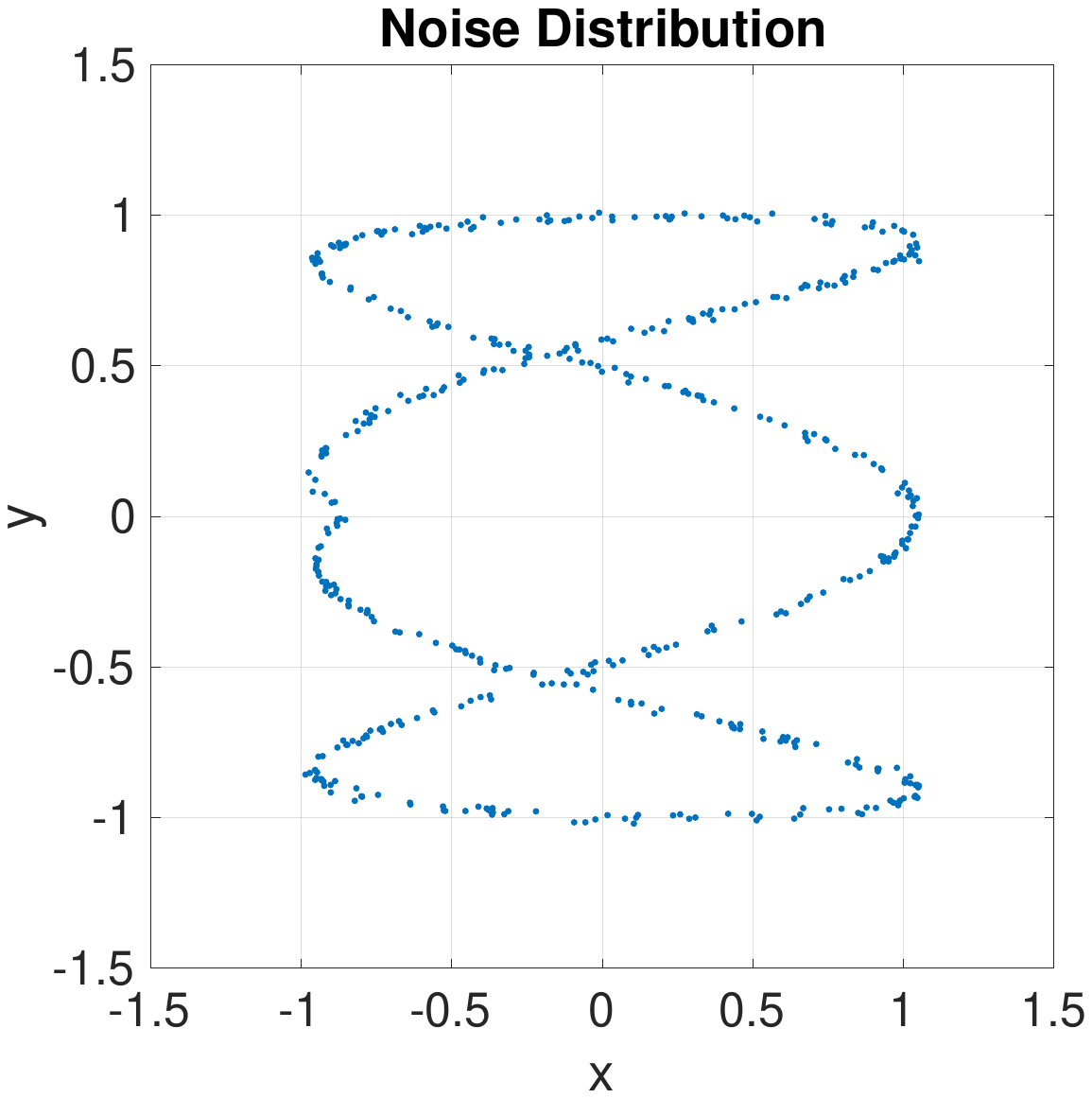} 
    \caption{Trigonometric noise}
    \label{fig:trignoise}
\end{minipage}
\end{figure}

Under this setup we evaluate the estimation error for increasing noise scales ranging from 0.1 and 10 and compare three approaches: (i) the batch \BPUE that takes 50 measurements and produces an estimate of the state $\M x$, (ii) a recursive \BPUE that adds one measurement at a time (this can be seen as applying 50 update steps of the \GMKF without prediction steps), and (iii) the \BLUE that takes the same 50 measurements and estimates the state using linear least squares.
\Cref{fig:update-image1}(a) shows the average estimation errors for the case with binary noise, while \Cref{fig:update-image1}(b) shows the errors for the case of trigonometric noise. 
For each noise scale, we compute the average estimation error across 100 Monte Carlo runs.
The results suggest that under increasingly larger non-Gaussian noises, the \BPUE estimates 
remain accurate and have a much slower error increase compared to  \BLUE. 
The results also show that the batch and recursive formulation produce identical estimates (the corresponding lines in~\Cref{fig:update-image1} fully overlap): 
this is expected since the only approximate step of the \GMKF is in the way we marginalize out previous states after the prediction step.

\begin{figure}[htp]

\centering
\begin{minipage}{0.24\textwidth}
    \centering
    \includegraphics[width=\linewidth, trim={0.2cm 0.2cm 0.8cm 0.2cm},clip]{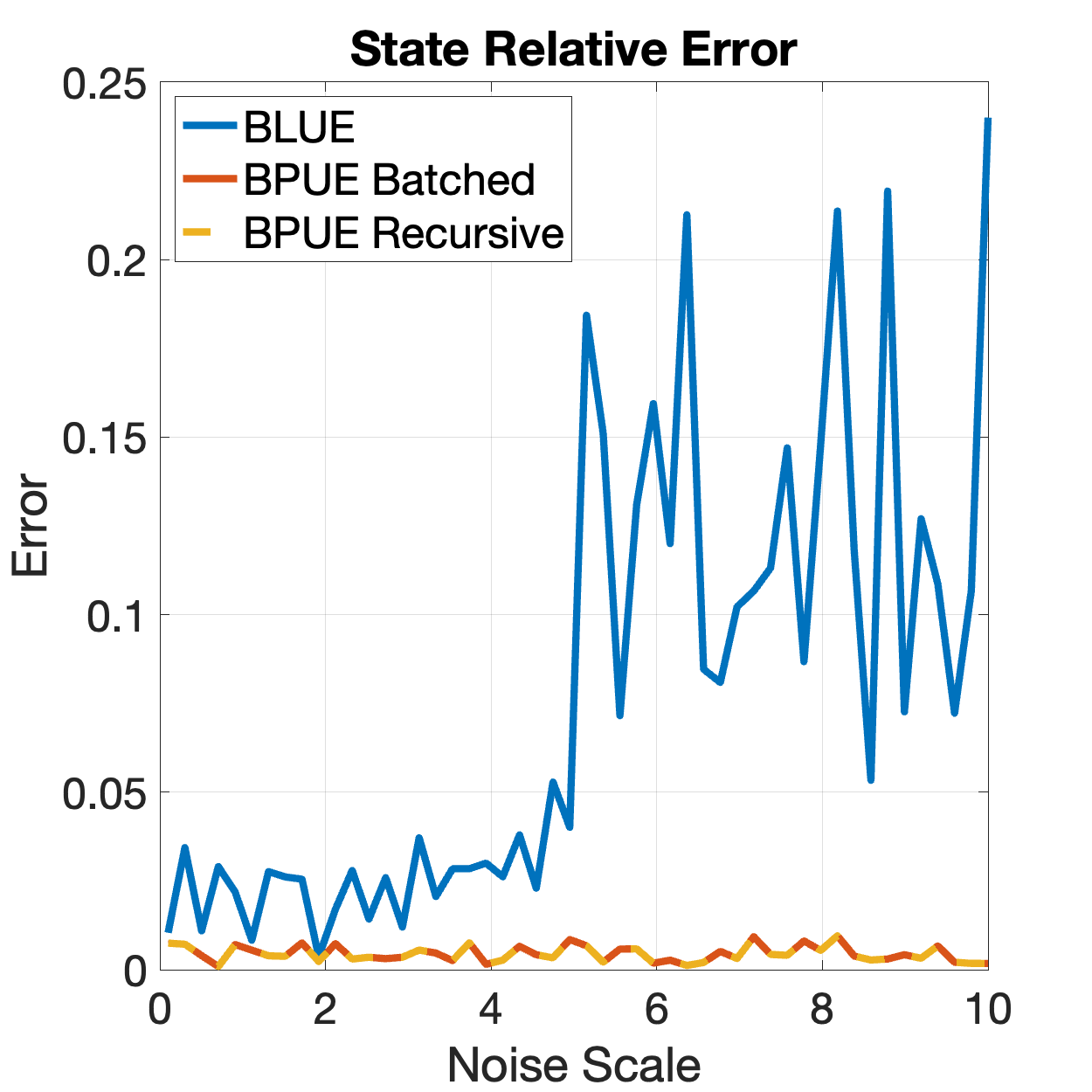} 
    (a) Binary noise
\end{minipage}
\hfill
\begin{minipage}{0.24\textwidth}
    \centering
    \includegraphics[width=\linewidth, trim={0.2cm 0.2cm 1cm 0.3cm},clip]{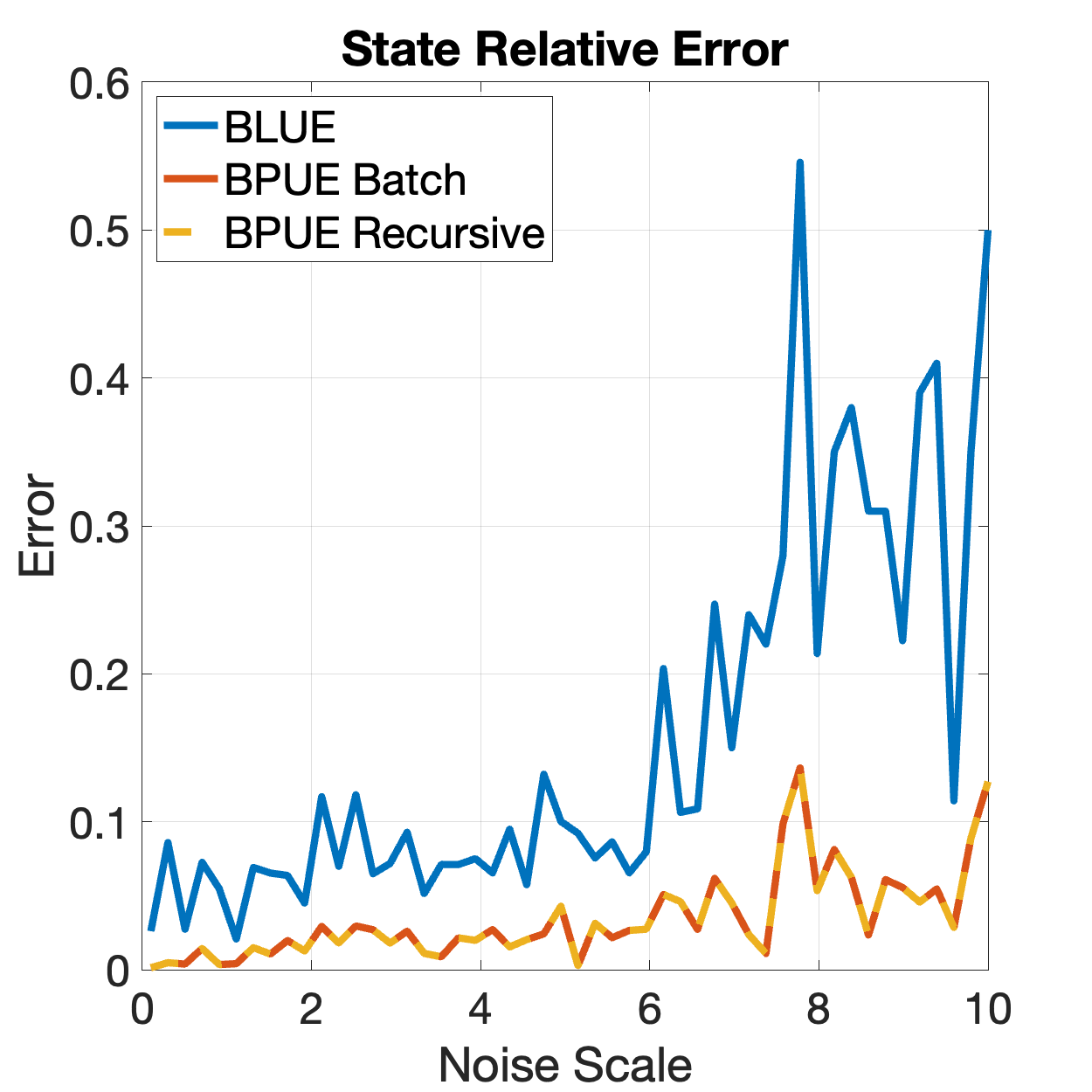} 
    (b) Trigonometric noise
\end{minipage}

\caption{Average estimation errors for \BLUE, batch \BPUE, and recursive \BPUE (update-only \GMKF) for increasing noise scale, for the case of (a) binary noise and (b) trigonometric noise.
\label{fig:update-image1}}
\end{figure} 

We also illustrate the empirical covariance of the batch estimates in \Cref{fig:batchnoise}. Each data point represents an estimate produced by \BPUE using 50 measurements, and we obtain 500 points each one corresponding to a Monte Carlo run. Then, we fit a covariance to the corresponding estimation errors. Note that we only show \BPUE batch since we conclud that the estimates are identical in the recursive case (\Cref{fig:update-image1}).
 Again, the figure clearly shows that \BPUE leads to much smaller estimation errors, compared to \BLUE.

    
\begin{figure}
\begin{minipage}{0.24\textwidth}
    \centering
    \includegraphics[width=\linewidth, trim={0.8cm 6cm 1cm 4cm},clip]{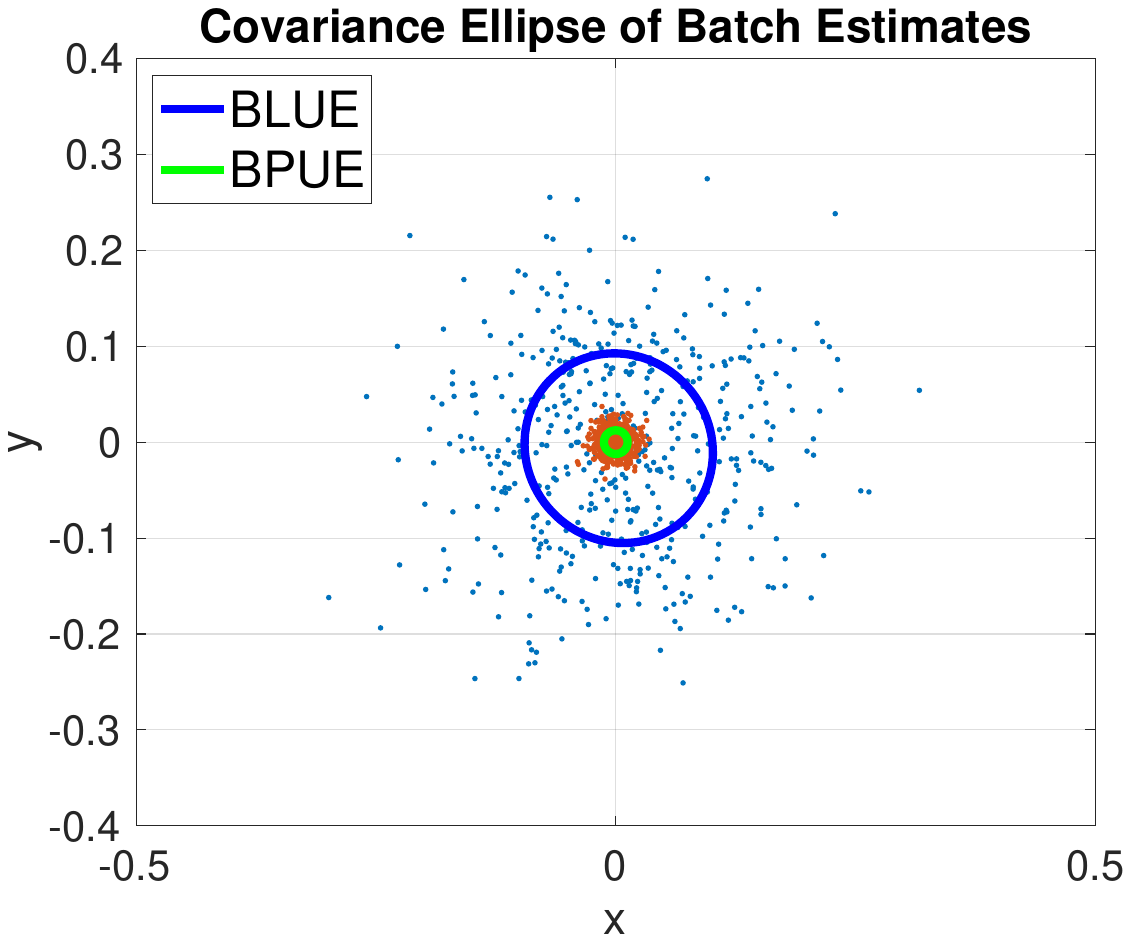} 
    (a) Binary noise
\end{minipage}
\hfill
\begin{minipage}{0.24\textwidth}
    \centering
    \includegraphics[width=\linewidth, trim={0.8cm 6cm 1cm 4cm},clip]{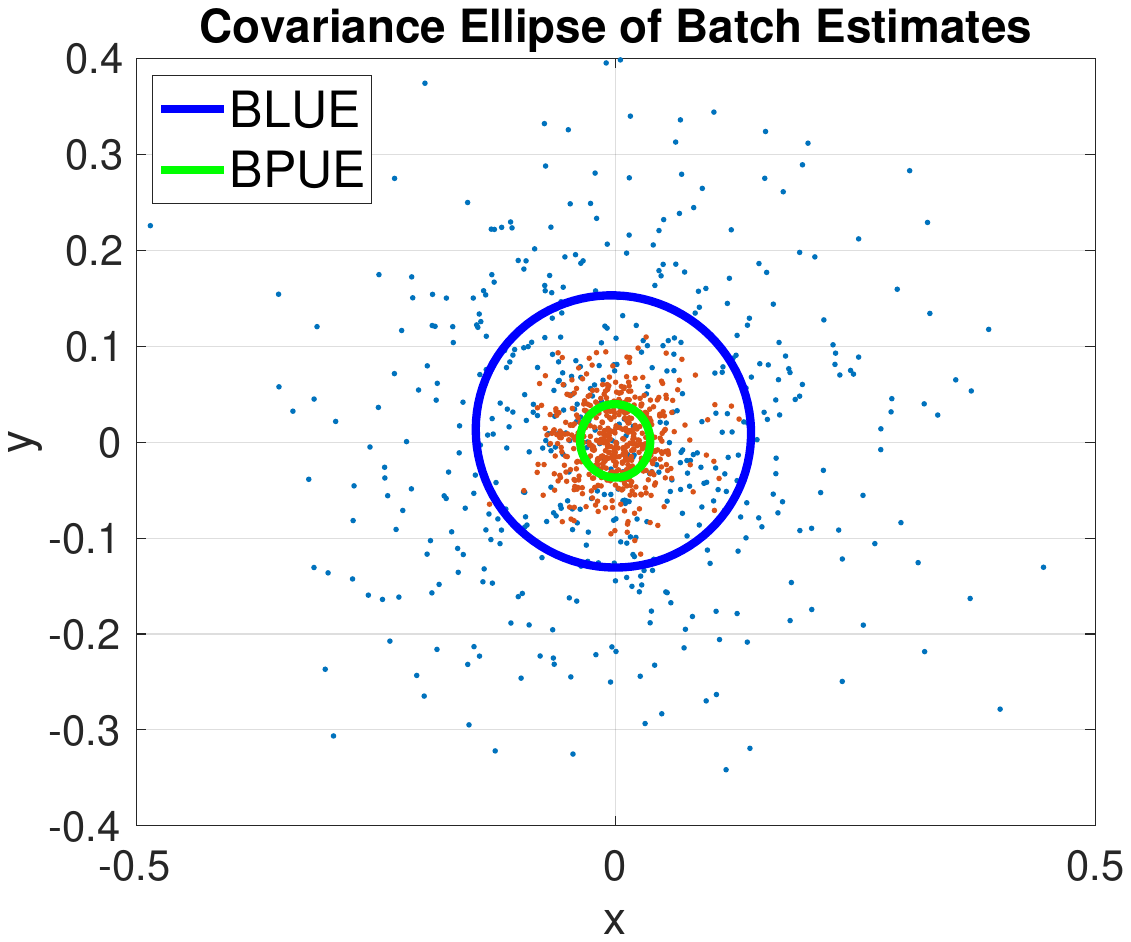} 
    (b) Trigonometric noise
\end{minipage}
\caption{Comparison of the empirical error covariances 
for the \BLUE and \BPUE estimates over 500 Monte Carlo runs and for noise scale $s=1$.
Blue dots correspond to \BLUE estimation errors, while orange dots correspond to \BPUE estimation errors.}
\label{fig:batchnoise}
\vspace{-3mm}
\end{figure} 


\begin{figure*}[htp]

\centering

\begin{minipage}{0.19\textwidth}
    \centering
    \includegraphics[width=\linewidth, trim={1cm 4cm 1cm 5.3cm},clip]{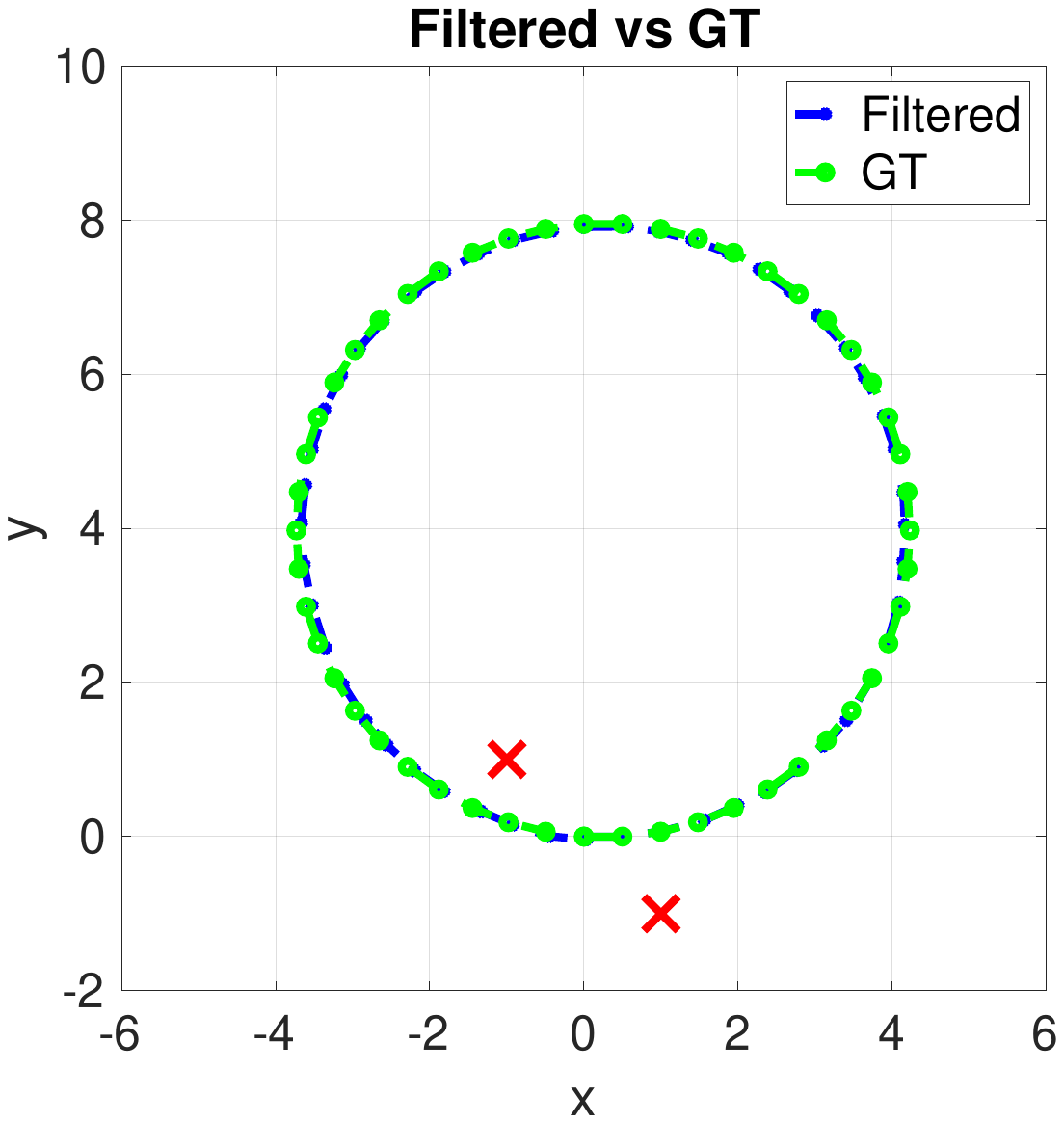}\\ 
    \scriptsize{GMKF with Binary Noise}
\end{minipage}
\hfill
\begin{minipage}{0.19\textwidth}
    \centering
    \includegraphics[width=\linewidth, trim={1cm 4cm 1cm 5.3cm},clip]{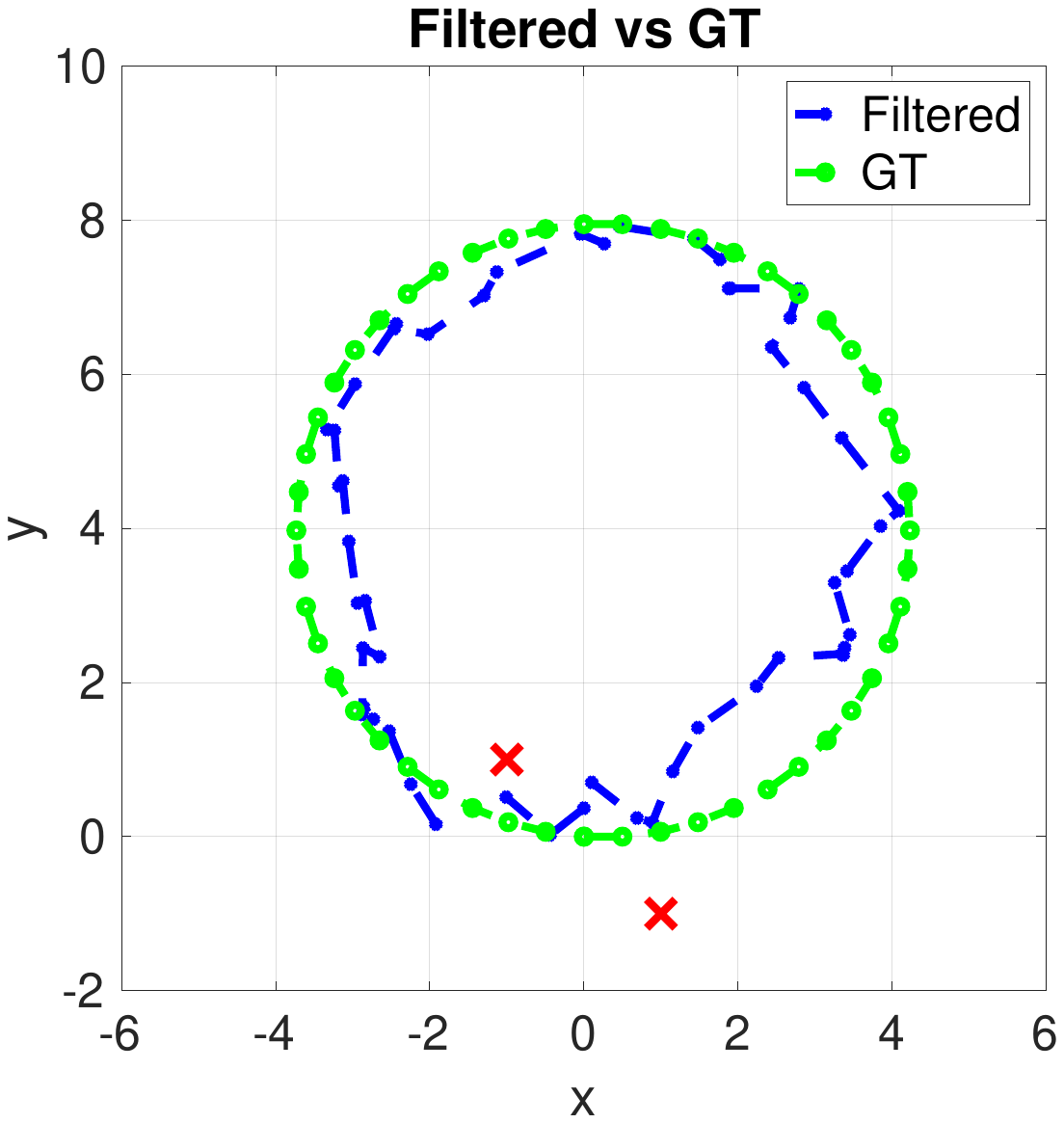}\\ 
    \scriptsize{EKF with Binary Noise}
\end{minipage}
\hfill
\begin{minipage}{0.19\textwidth}
    \centering
    \includegraphics[width=\linewidth, trim={1cm 4cm 1cm 5.3cm},clip]{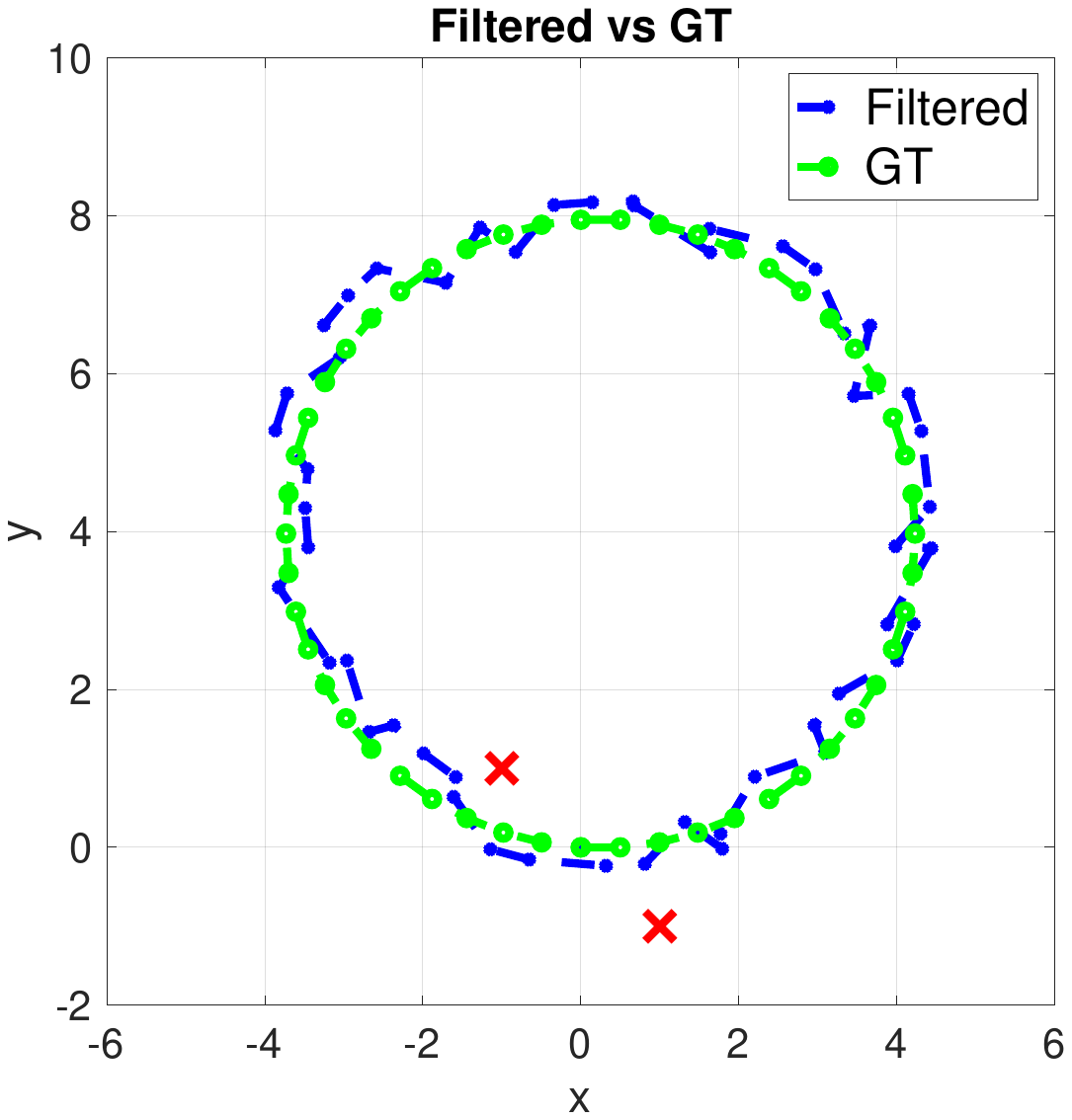}\\ 
    \scriptsize{UKF with Binary Noise}
\end{minipage}
\hfill
\begin{minipage}{0.19\textwidth}
    \centering
    \includegraphics[width=\linewidth, trim={1cm 4cm 1cm 5.3cm},clip]{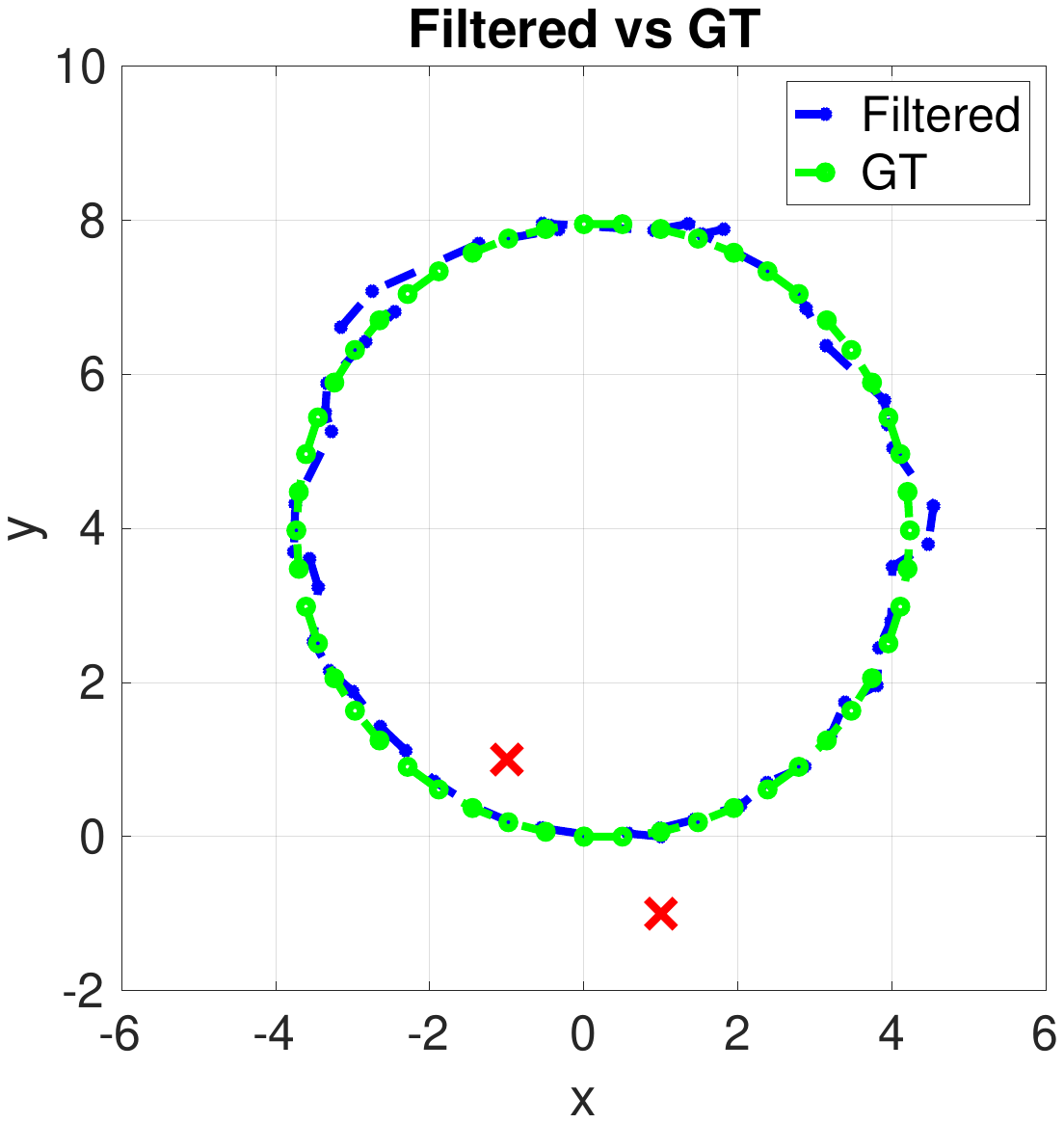}\\ 
    \scriptsize{InEKF with Binary Noise}
\end{minipage}
\hfill
\begin{minipage}{0.19\textwidth}
    \centering
    \includegraphics[width=\linewidth, trim={1cm 4cm 1cm 5.3cm},clip]{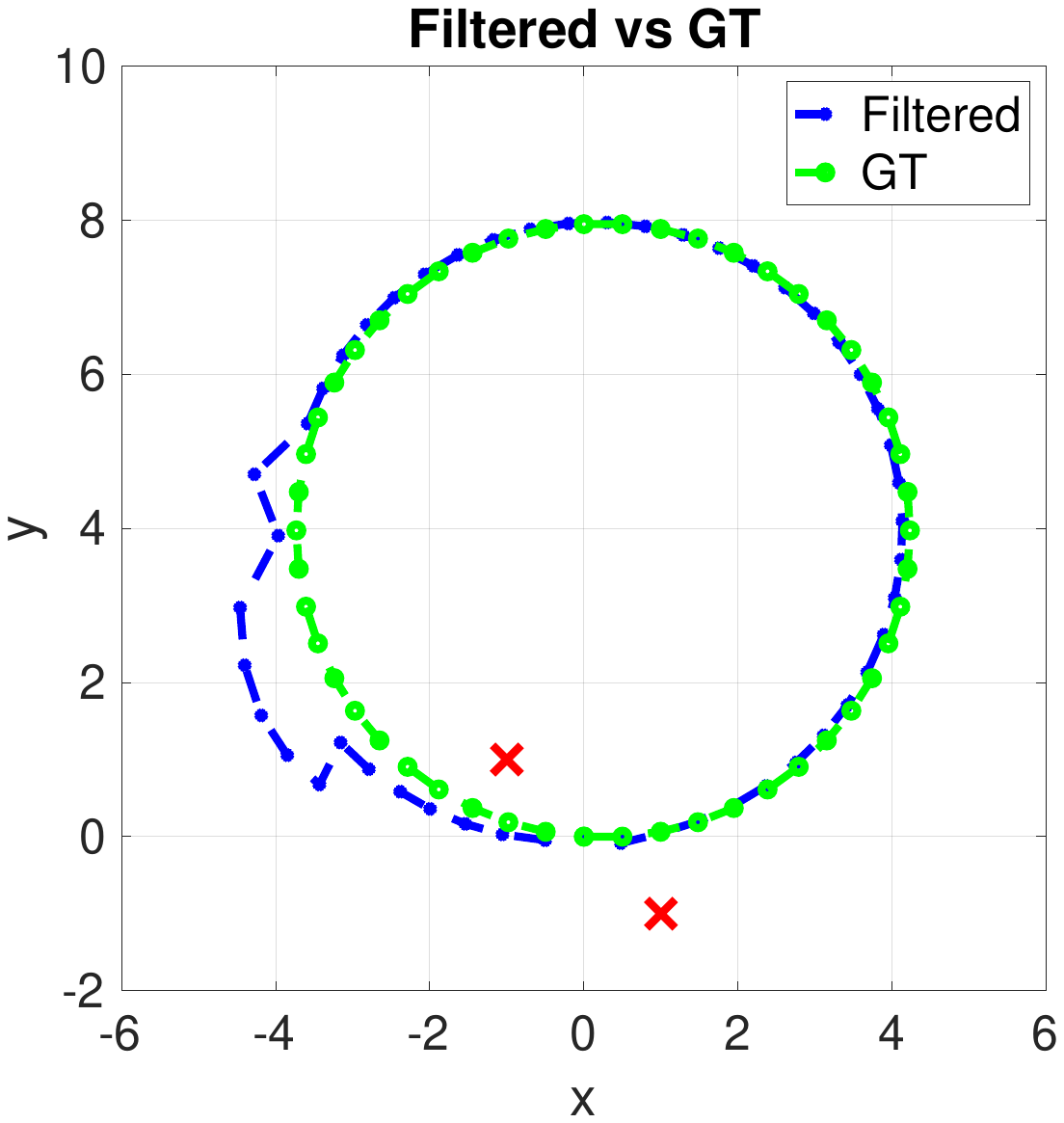}\\ 
    \scriptsize{UKFM with Binary Noise}
\end{minipage}
\hfill

\vspace{1em} 

\begin{minipage}{0.19\textwidth}
    \centering
    \includegraphics[width=\linewidth, trim={1cm 4cm 1cm 5.3cm},clip]{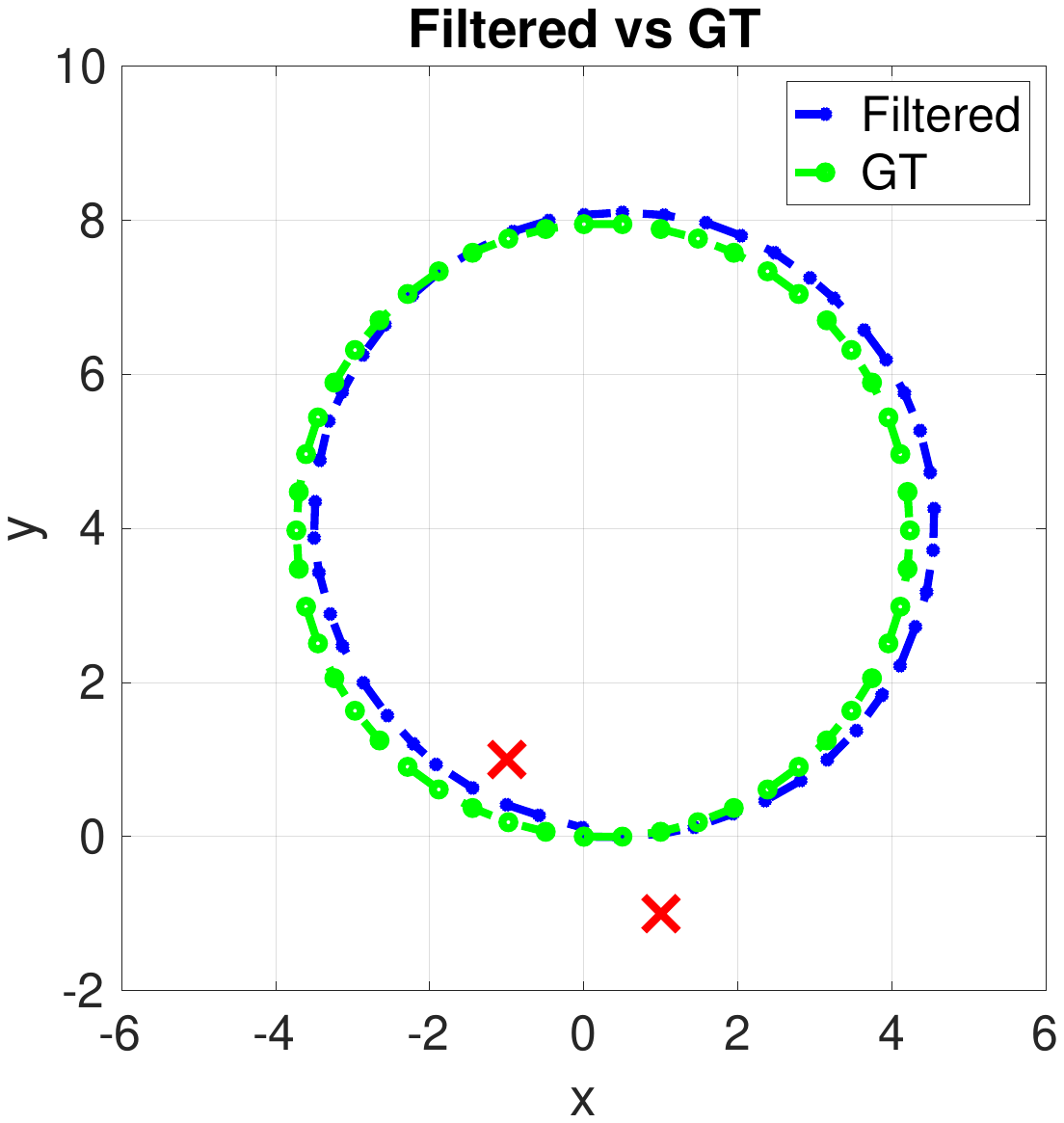}\\ 
    \scriptsize{GMKF with Trig Noise}
\end{minipage}
\hfill
\begin{minipage}{0.19\textwidth}
    \centering
    \includegraphics[width=\linewidth, trim={1cm 4cm 1cm 5.3cm},clip]{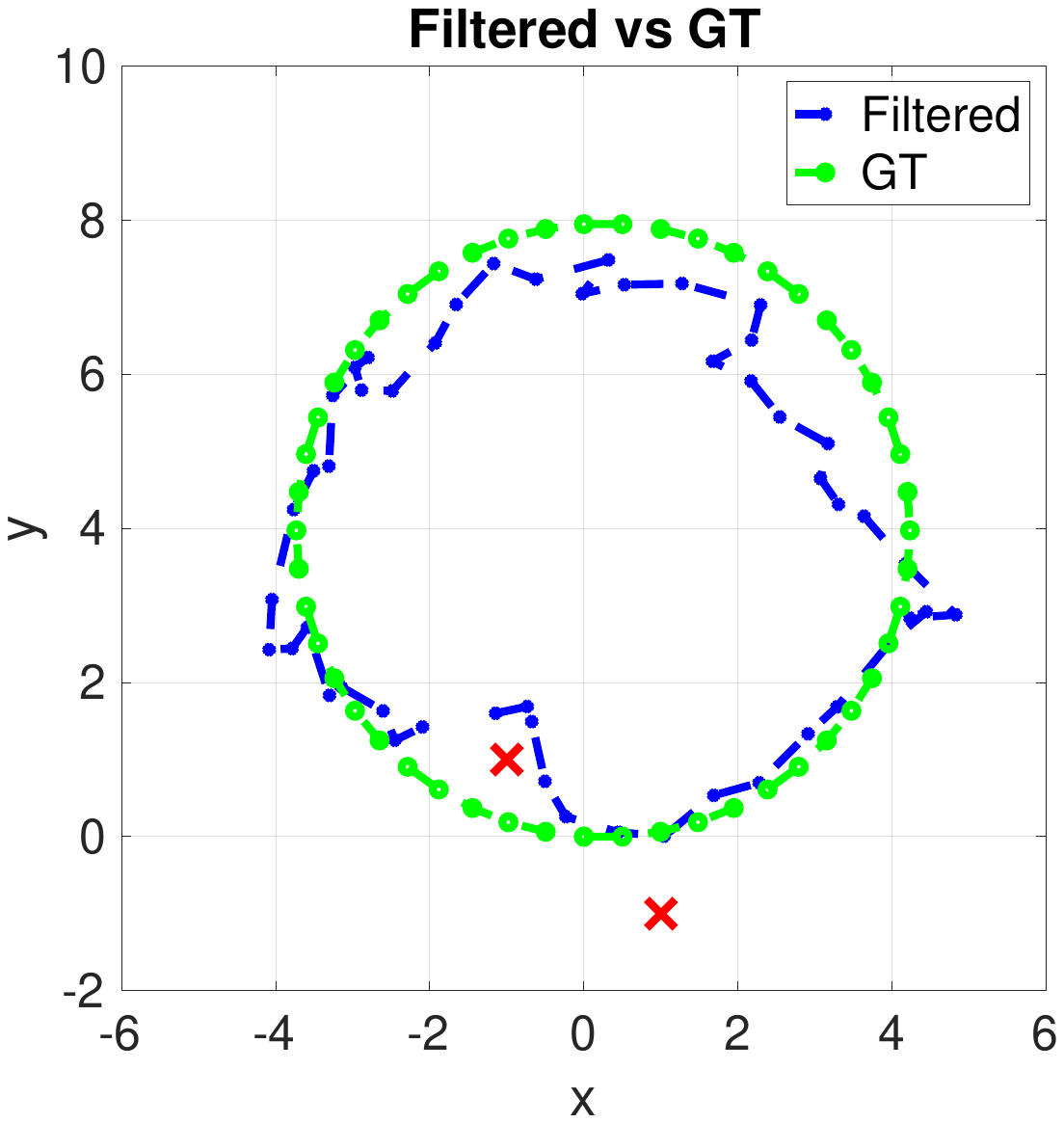}\\ 
    \scriptsize{EKF with Trig Noise}
\end{minipage}
\hfill
\begin{minipage}{0.19\textwidth}
    \centering
    \includegraphics[width=\linewidth, trim={1cm 4cm 1cm 5.3cm},clip]{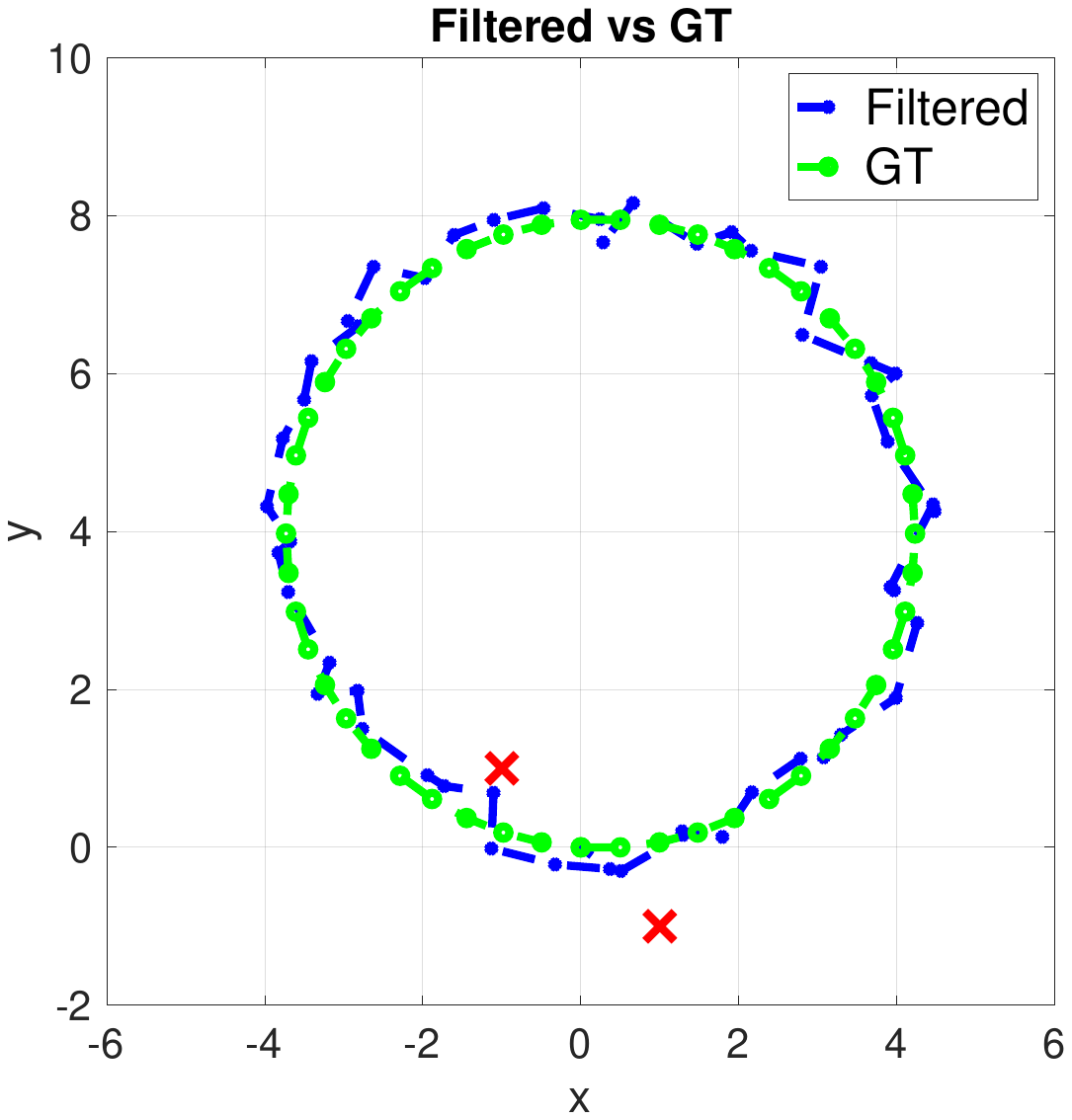}\\ 
    \scriptsize{UKF with Trig Noise}
\end{minipage}
\hfill
\begin{minipage}{0.19\textwidth}
    \centering
    \includegraphics[width=\linewidth, trim={1cm 4cm 1cm 5.3cm},clip]{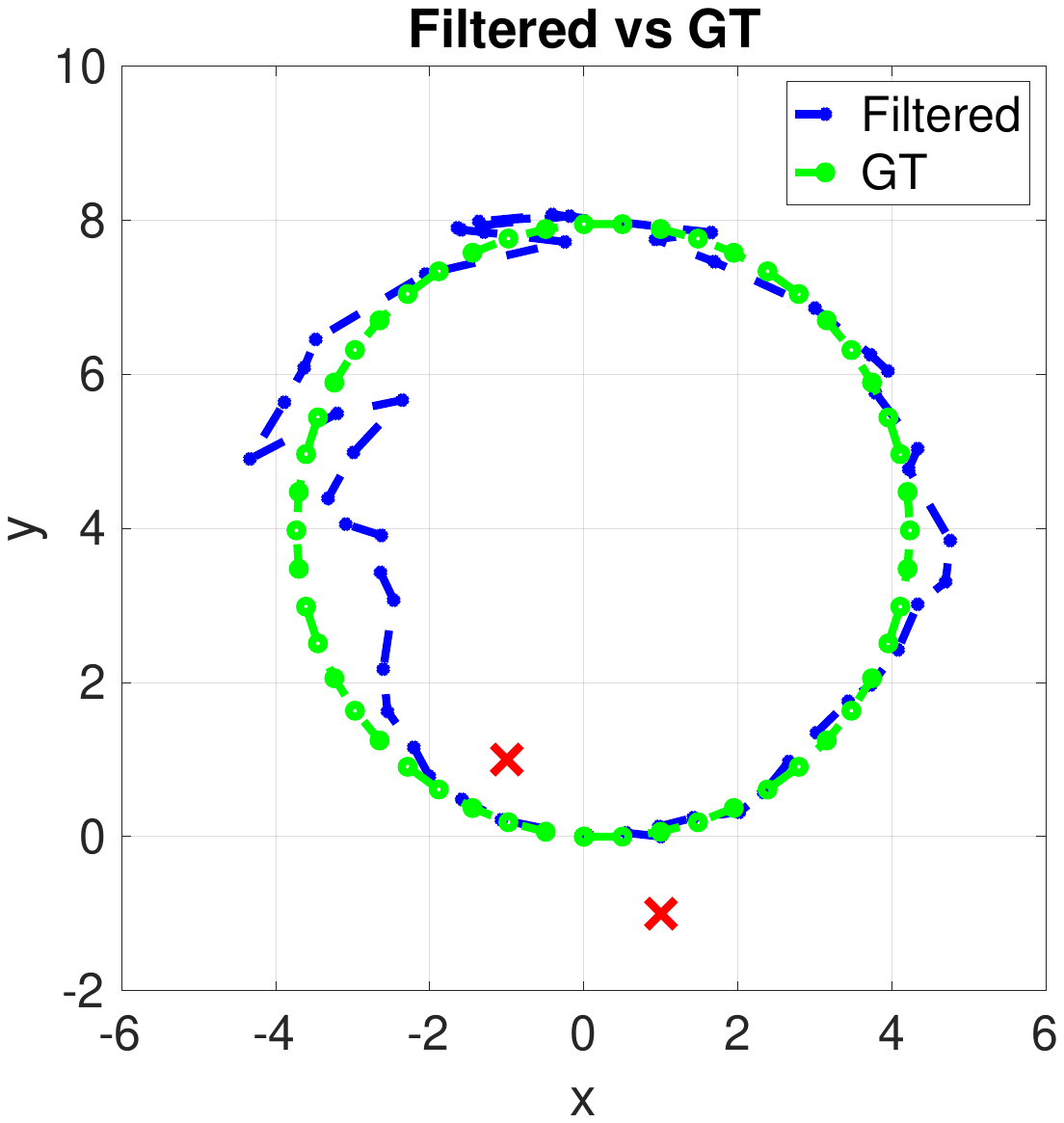}\\ 
    \scriptsize{InEKF with Trig Noise}
\end{minipage}
\hfill
\begin{minipage}{0.19\textwidth}
    \centering
    \includegraphics[width=\linewidth, trim={1cm 4cm 1cm 5.3cm},clip]{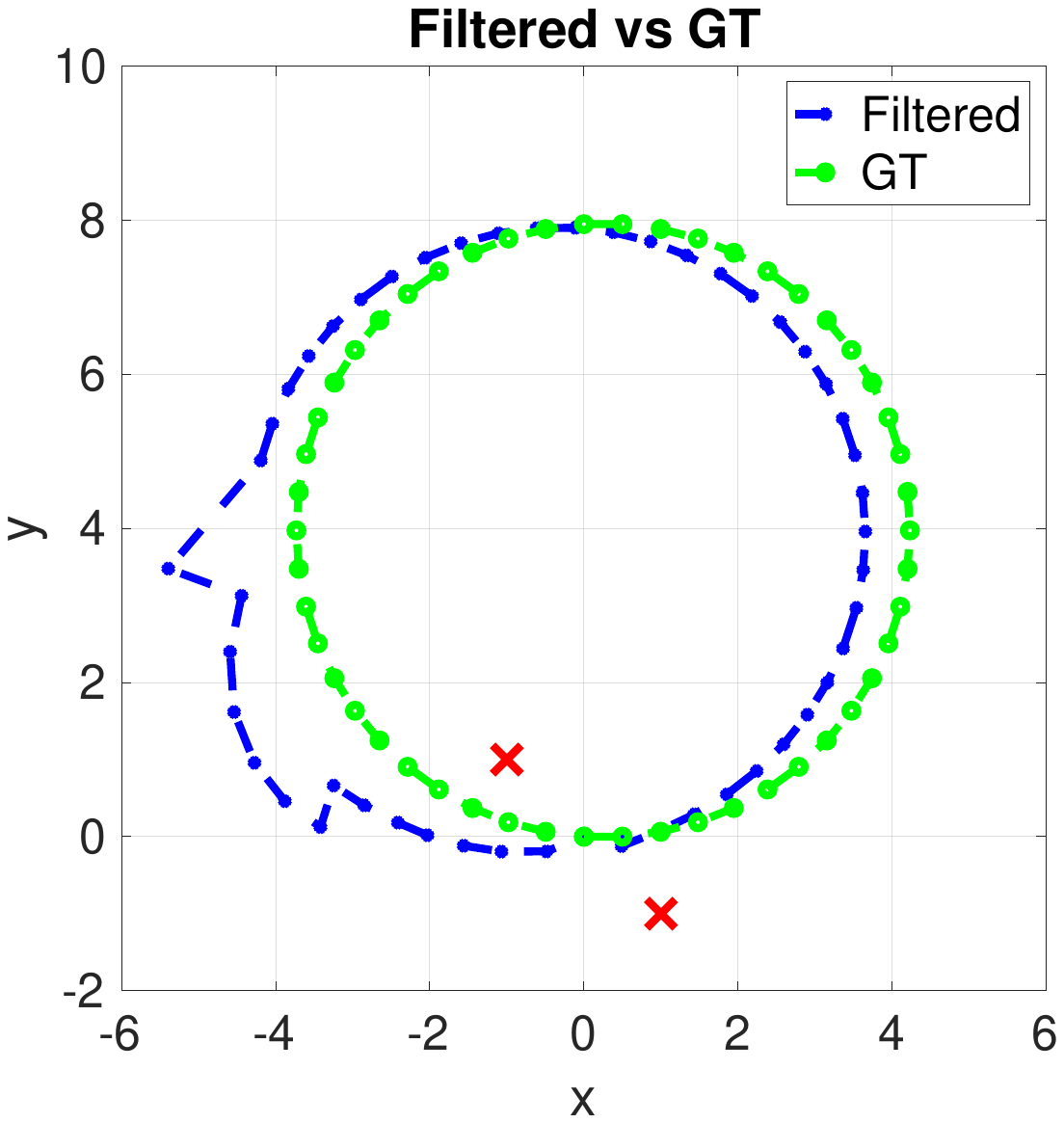}\\ 
    \scriptsize{UKFM with Trig Noise}
\end{minipage}
\hfill
\caption{Visualization of the filtered trajectory vs. ground-truth trajectory for the proposed approach (\GMKF) and baselines (EKF, UKF, InEKF, UFKM). The red crosses denote the positions of the (known) landmarks. In each step, only one landmark position is observed.
}
\label{fig:mkf-image8}
\end{figure*}
\subsection{Robot Localization in $\mathrm{SE}(2)$}
We consider a localization problem for a robot moving in the plane (i.e., with state in $\mathrm{SE}(2)$) 
and observing landmarks at known locations. First, we define the dynamics: 
\begin{equation}
\label{eq:SE2_dynamics}
    \M X_{k+1} = \M X_k \cdot \M U_k \cdot \M W_k, 
\end{equation}
where $\M X_k$ is the  robot pose (represented as a $3 \times 3$ transformation matrix), $\M U_k$ is the pose change between time $k$ and $k+1$, and $\M W_k$ is the process noise, all defined on $\mathrm{SE}(2)$. We define the vectorized state as $\M{x}_k:=[x_k, y_k, c_k, s_k]$ and note that the $\mathrm{SE}(2)$ pose becomes the following matrix with quadratic constraints $g(\M{x})$:
\begin{equation}\label{eq:statecon}
    \M X_k = \begin{bmatrix}
        c_k & -s_k & x_k \\
        s_k &  c_k & y_k \\
        0 &  0 & 1 \\
    \end{bmatrix}, g(\M{x}_k) = c_k^2 + s_k^2 - 1 = 0.
\end{equation}
We consider the range-bearing observation function
\begin{equation}
    \M y_k = \M R_k^{\transpose}(\M p_L - \M p_k) + \M v_k,
\end{equation}
where $\M R_k:= \begin{bmatrix}
        c_k & -s_k \\
        s_k &  c_k \\
    \end{bmatrix}$ is the rotation matrix describing the orientation of the robot at time $k$, $\M p_k$ is the position of the robot,
    $\M p_L$ is the known position of the landmark $L$ in Cartesian coordinates,  and $\M v_k$ is the measurement noise. 
Since the process noise is multiplicative and we need it to be additive to comply with the form assumed in~\eqref{eq:polyDynSys}, 
we approximate the dynamics using the most recent state estimate $\hat{\M X}^{+}_k$:\footnote{The expert reader might realize that 
an alternative way to rephrase the dynamics in the form~\eqref{eq:polyDynSys} is to simply rewrite~\eqref{eq:SE2_dynamics} as
$(\M U_k)^{-1} \cdot (\M X_k)^{-1} \cdot \M X_{k+1} =  \M W_k$. We prefer not to follow this route for practical reasons: the term $(\M U_k)^{-1} \cdot (\M X_k)^{-1} \cdot\M X_{k+1}$ involves degree-2 polynomials, that, when expressed in the extended form~\eqref{eq:ext-polyDynSys}, lead to high-degree polynomials and result in impractically slow moment relaxations. 
Despite this approximation, our tests show that \GMKF largely outperforms competitors and can seamlessly account for state constraints as the ones in~\eqref{eq:statecon}. 
} 
\begin{equation}
\begin{aligned}
    &\M X_{k+1} - \M X_k\M U_k = \hat{\M X}^{+}_k{\M U}_k(\M W_k - \M I), \\
    &\M f(\M X_{k+1}, \M X_k, \M U_k) = {\M U}_k^{\transpose}\hat{\M X}^{+\transpose}_k(\M X_{k+1} - \M X_k\M U_k) = (\M W_k - \M I).
\end{aligned}
\end{equation}
We similarly rewrite the measurement model linearly as: 
\begin{equation}
\begin{aligned}
    &\M R_k\M y_k - (\M p_L - \M p_k) =  \bar{\M R}_k\M v_k, \\ 
    &\M h(\M y_k, \M x_k) = \bar{\M R}^{\transpose}_k(\M R_k\M y_k - (\M p_L - \M p_k)) =  \M v_k. 
\end{aligned}
\end{equation}

We compare the performance across different filters under large non-Gaussian observation noise. We choose the EKF, UKF, and their variants on Lie groups, i.e., the Right Invariant EKF (InEKF) \cite{barrau2018invariant} and UKF on Manifold (UKFM) \cite{brossard2017unscented} as baselines. For the EKF and UKF, we consider the position $(x, y)$ and heading angle $\theta$ as the state. For the InEKF and UKFM, the state is $\mathrm{SE}(2)$ with linearization and covariance defined in the Lie algebra. 

In our experiments, we make the robot track a circle with a constant twist. 
For each of the trials, we fix a random seed and use the same noise values for all techniques. 
The process noise considers the differential wheel odometry model from \cite{long2013banana}, assuming a zero-mean Gaussian distribution in the Lie algebra with covariance $0.04\M I$. 
In our implementation we mapped the covariance to $\mathrm{SE}(2)$ and computed the moments of the noise via sampling.  For the measurement noise, 
we test both the binary and the trigonometric noise discussed above and set $s=1$. 
We choose $r = 2$ in \GMKF to utilize moments of the noise with orders up to four.
In \Cref{fig:mkf-image8}, we visualize the state estimates over time for each filter under both binary and trigonometric measurement noise. 
Due to non-Gaussian noise, all baselines fail to consistently follow the circular trajectory. 
Compared to the baselines, the trajectory estimated by GMKF are closest to the ground truth. Noticeably, the EKF baseline has the largest drift from the ground-truth trajectory.

We also report the average log-scale translation estimation error in meters in \Cref{fig:trans-error} and the average log-scale angular estimation error in radians in \Cref{fig:rot-error} across 20 trials for each filter under the trigonometric noise setup. Again, \GMKF has the lowest translation and rotation error across all filters and its performance remains relatively constant through the steps. While UKFM significantly outperforms the other three baselines, it deviates from the ground-truth trajectory in the latter half of the trajectory.
\begin{figure}[htp]

\centering
\begin{minipage}{0.24\textwidth}
    \centering
    \includegraphics[width=\linewidth, trim={1cm 5cm 1cm 6cm},clip]{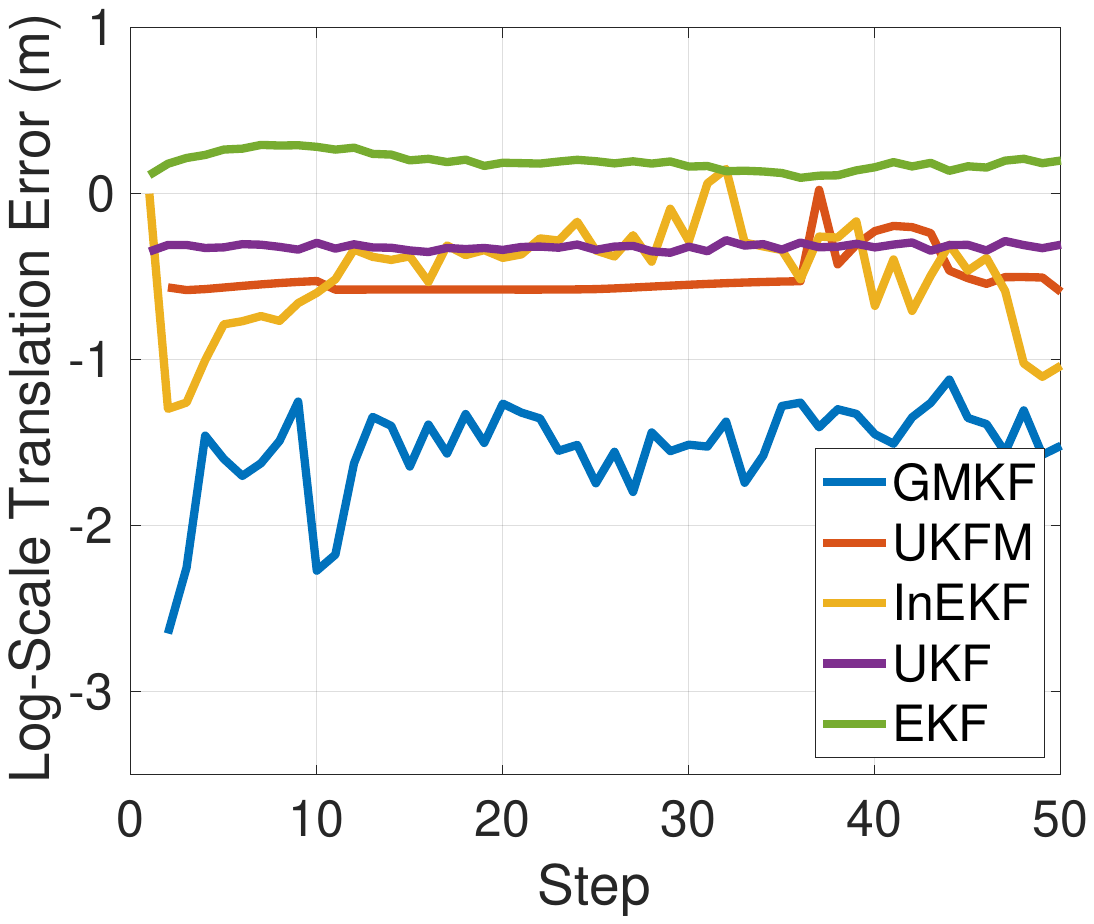} 
    \caption{Comparison of log-scale translation error in meters over time.}
    \label{fig:trans-error}
\end{minipage}
\hfill
\begin{minipage}{0.24\textwidth}
    \centering
    \includegraphics[width=\linewidth, trim={1cm 5cm 1cm 6cm},clip]{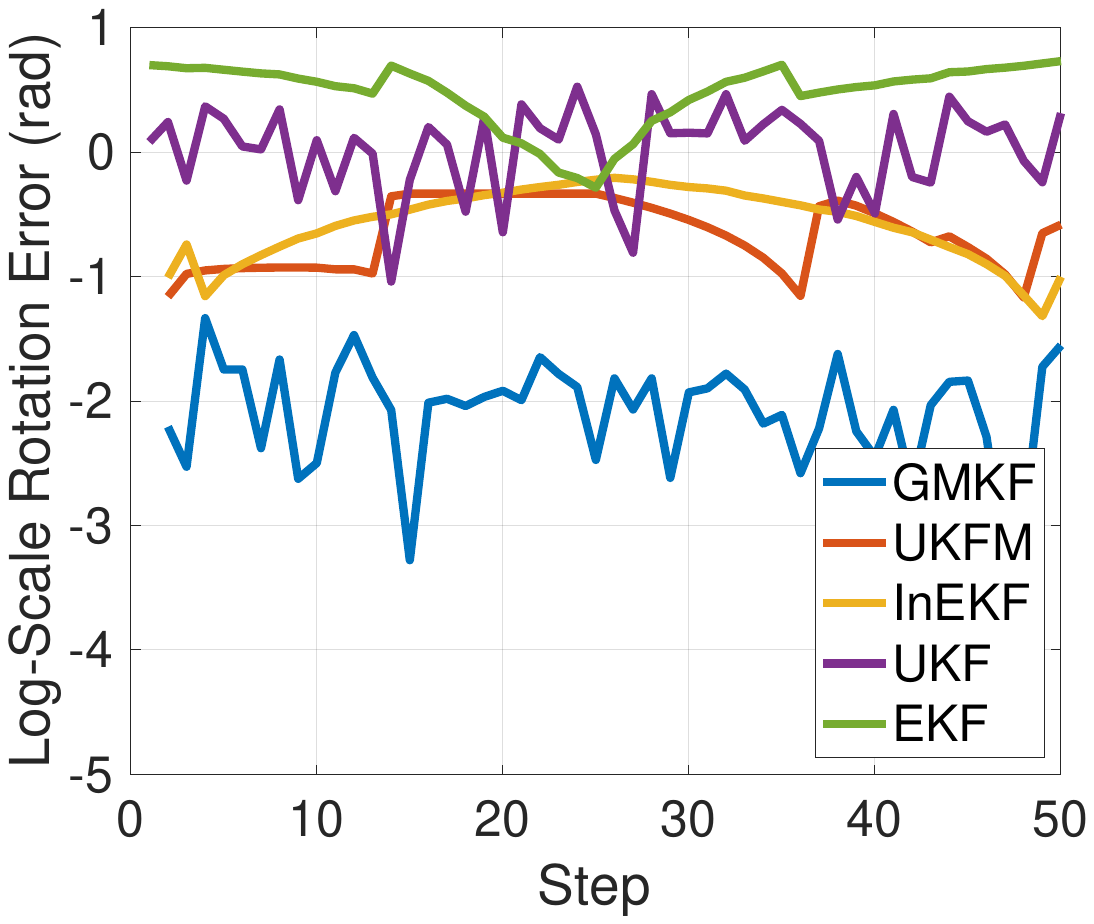} 
    \caption{Comparison of log-scale rotation error in radians over time.}
    \label{fig:rot-error}
\end{minipage}

\end{figure}


\section{Discussions and Limitations}
\label{sec:discussion}


In the experiments, we have commented on the fact that both in the batch and in the recursive setup, the moment relaxations
 we solve empirically compute rank-1 optimal solutions. While this is currently an empirical observation, 
 it would be interesting to derive theoretical conditions on the dynamical system under which the relaxations admit rank-1 solutions, and explore the connections between the rank of the solutions and the observability properties of the polynomial system.

In Section~\ref{sec:theory}, we have shown that the matrix appearing in the SOS belief can be computed from the dual (SOS) relaxation and that such a matrix can be interpreted as an estimation covariance for the limiting case of noiseless measurements.
While this is an encouraging result, the noiseless case is only of theoretical interest and further research is needed to better connect the SOS belief (which we mostly derived using an optimization-theoretic lens) to a fully probabilistic interpretation.
We hope that such interpretation will also provide theoretical support for the ``marginalization'' step we perform after the prediction, which currently is the main obstacle towards claiming optimality of the \GMKF.

Finally, for linear Gaussian systems, the BLUE has been proven to be optimal among all estimators, and matches the performance of our BPUE when $r=1$ (see \Cref{thm:gen_KF}). In this case, increasing the $r$ in the BPUE will not improve the estimation. 
Thus, it is worth investigating how to properly choose the order $r$, also considering that large $r$ values lead to much more expensive optimization problems to solve during the prediction and update step.

\section{Related Work}
\label{sec:relatedWork}

{\bf Kalman Filtering.}
Filtering refers to estimating a system's current state using the set of all measurements up to the current time \cite{grewal2014kalman, smith2006approaches}. When the dynamics and measurement models are linear, and the noises are Gaussian, the Kalman Filter \cite{kalman1960new} becomes provably optimal. 
In particular, the KF computes the same state estimate as the Best Linear Unbiased Estimator (\BLUE) \cite{humpherys2012fresh} applied to all measurements and controls collected until the current time. 
The \BLUE can also be seen as a point estimation problem using the Generalized Method of Moments (GMM) \cite{hansen1982large}. In GMM, the full distribution of the data may not be known, and therefore, maximum likelihood estimation is not applicable \cite{hansen1982large, hayashi2011econometrics}. Inspired by GMM, we only consider the moment of the noise in our state estimation problem to describe arbitrary noise distributions. 


To deal with nonlinear systems, the Extended Kalman Filter (EKF) linearizes the process and measurement models around the current state estimate \cite{thrun2002probabilistic} and propagates the covariance through the linearized system. However, EKF's performance deteriorates as the system's dynamics become more nonlinear and the noise increases, as the quality of the linearization point degrades \cite{huang2010observability, song1992extended}. The symmetry of matrix Lie groups has been leveraged in the Invariant EKF (InEKF) \cite{barrau2016invariant, barrau2018invariant} to design a linearization scheme that leads to more accurate state estimation. The InEKF has proven its success in visual-inertial navigation \cite{wu2017invariant, zhang2017convergence}, legged locomotion~\cite{hartley2020contact, teng2021legged}, and proprioceptive robot state estimation~\cite{lin2023proprioceptive, yu2023fully}. The symmetry in a more general homogeneous manifold has been utilized to design the Equivariant Filter \cite{van2022equivariant}, which subsumes the InEKF. All these methods adopted the Kalman gain and only utilized the mean and covariance of the noise distribution. 

As an alternative to linearization, the \emph{Unscented Kalman Filter (UKF)} \cite{wan2000unscented} applies a deterministic sampling approach to handle nonlinearities. The sample points aim to capture the true mean and covariance, and ---when propagated through the nonlinear system--- they allow reconstructing the posterior mean and covariance accurately to the 3rd order for any nonlinearity. However, for non-Gaussian noise, UKF requires substantial parameter tuning and lacks a systematic way to handle arbitrary noise. Recent advancements also consider the presence of non-Gaussian noise in dynamical systems. Specifically, \citet{izanloo2016kalman} use an information-theoretic approach for state estimation under non-Gaussian noise, but the stability of the result hinges on the kernel size hyperparameter. Other KF variants \cite{stojanovic2016robust} assume non-Gaussian measurement noise but Gaussian process noise \cite{stojanovic2016robust} or linear systems with specific non-Gaussian noise distributions \cite{sun2013state}. None of these works claim the optimality of the resulting estimator.
More closely related to our work is that of \citet{shimizu2023moment} and \citet{jasour2021moment}, where moment propagation is derived for filtering problems by assuming the knowledge of the characteristic function of the Gaussian noise. In our methods, we only utilize moment information instead of the entire distribution, which is inaccessible in typical applications.

{\bf Semidefinite Relaxation}
Semidefinite relaxations have been critical to designing certifiable algorithms for estimation in robotics. The first use of moment relaxations in computer vision goes back to \cite{kahl2007globally}. 
These relaxations transform polynomial optimization problems into semidefinite programs (SDPs) that are convex, using Lasserre's hierarchy of moment relaxations \cite{lasserre2001global}. 
Subsequently, SDP relaxation-based certifiable algorithms have been extended to various geometric perception problems, including pose graph optimization \cite{carlone2016planar, rosen2019se}, rotation averaging \cite{fredriksson2012simultaneous, eriksson2019rotation}, triangulation \cite{aholt2012qcqp, cifuentes2021convex}, 3D registration \cite{briales2017convex, chaudhury2015global, maron2016point, iglesias2020global}, absolute pose estimation \cite{agostinho2023cvxpnpl}, relative pose estimation \cite{garcia2021certifiable, briales2018certifiably, zhao2020efficient}, hand-eye calibration \cite{heller2014hand, giamou2019certifiably}, and category-level object perception \cite{shi2021optimal, yang2020perfect}. 
Moment relaxations have also been proven effective in dealing with heavy-tailed noise and outliers in a variety of 3D vision problems~\cite{yang2022certifiably, yang2020teaser, yang2019quaternion}.
Despite the inherently nonconvex nature of these problems, semidefinite relaxations at the minimal order have been empirically shown to be exact in many practical problems. \citet{cifuentes2020local, cifuentes2020geometry} study the tightness of the Shor relaxation and provide an approximation of the tightness region. \citet{cosse2021stable} prove the tightness of the second-order moment relaxation to a matrix completion problem by an analytical construction of the dual SOS polynomials, whose proof inspired 
some of the results in the current paper. 

\section{Conclusions}
\label{sec:conclusion}

We study the challenging problem of state estimation in polynomial systems corrupted by arbitrary process and measurement noise, which arises in several robotics applications.
Our first contribution is to consider a batch setup where the state is estimated from the set of measurements, controls, and priors available from time 0 until the current time. In this setup, we show that optimal state estimation can be formulated as a Polynomial Optimization Problem (POP) and solved via moment relaxations. We also show that the dual (SOS) relaxation allows forming an SOS belief that can be approximately seen as a Gaussian in the moment space. 
Our second contribution is to consider a filtering setup where one has to recursively estimate the current state given a belief and the most recent measurements. In this setup, we introduce the \emph{\GMKFlong} (\GMKF). The \GMKF formulates prediction and update steps as POPs and solves them using moment relaxations, carrying over a possibly non-Gaussian belief. In the linear-Gaussian case, GMKF reduces to the standard Kalman Filter. 
We evaluate the proposed \GMKF on a robot localization problem and show it performs well under highly non-Gaussian noise and outperforms common alternatives, including the Extended and Unscented Kalman Filter and their variants on matrix Lie group.

{
\bibliographystyle{unsrtnat}
\bibliography{bib/strings-full,bib/ieee-full,bib/references}

\begin{thebibliography}{60}
\providecommand{\natexlab}[1]{#1}
\providecommand{\url}[1]{\texttt{#1}}
\expandafter\ifx\csname urlstyle\endcsname\relax
  \providecommand{\doi}[1]{doi: #1}\else
  \providecommand{\doi}{doi: \begingroup \urlstyle{rm}\Url}\fi

\bibitem[Song and Grizzle(1992)]{song1992extended}
Yongkyu Song and Jessy~W Grizzle.
\newblock The extended {Kalman} filter as a local asymptotic observer for nonlinear discrete-time systems.
\newblock In \emph{Proceedings of the American Control Conference}, pages 3365--3369. IEEE, 1992.

\bibitem[Wan and Van Der~Merwe(2000)]{wan2000unscented}
Eric~A Wan and Rudolph Van Der~Merwe.
\newblock The unscented {Kalman} filter for nonlinear estimation.
\newblock In \emph{Proceedings of the IEEE Adaptive Systems for Signal Processing, Communications, and Control Symposium}, pages 153--158. Ieee, 2000.

\bibitem[Teng et~al.(2023)Teng, Jasour, Vasudevan, and Jadidi]{Teng-RSS-23}
Sangli Teng, Ashkan Jasour, Ram Vasudevan, and Maani~Ghaffari Jadidi.
\newblock {Convex Geometric Motion Planning on {Lie} Groups via Moment Relaxation}.
\newblock In \emph{Proceedings of the Robotics: Science and Systems Conference}, Daegu, Republic of Korea, July 2023.
\newblock \doi{10.15607/RSS.2023.XIX.058}.

\bibitem[Ghaffari et~al.(2022)Ghaffari, Zhang, Zhu, Lin, Lin, Teng, Li, Liu, and Song]{ghaffari2022progress}
Maani Ghaffari, Ray Zhang, Minghan Zhu, Chien~Erh Lin, Tzu-Yuan Lin, Sangli Teng, Tingjun Li, Tianyi Liu, and Jingwei Song.
\newblock Progress in symmetry preserving robot perception and control through geometry and learning.
\newblock \emph{Frontiers in Robotics and AI}, 9:\penalty0 232, 2022.

\bibitem[Teng et~al.(2022{\natexlab{a}})Teng, Chen, Clark, and Ghaffari]{teng2022error}
Sangli Teng, Dianhao Chen, William Clark, and Maani Ghaffari.
\newblock An error-state model predictive control on connected matrix {Lie} groups for legged robot control.
\newblock In \emph{Proceedings of the {IEEE}/{RSJ} International Conference on Intelligent Robots and Systems}, pages 8850--8857. IEEE, 2022{\natexlab{a}}.

\bibitem[Teng et~al.(2022{\natexlab{b}})Teng, Clark, Bloch, Vasudevan, and Ghaffari]{teng2022lie}
Sangli Teng, William Clark, Anthony Bloch, Ram Vasudevan, and Maani Ghaffari.
\newblock {Lie} algebraic cost function design for control on {Lie} groups.
\newblock In \emph{Proceedings of the {IEEE} Conference on Decision and Control}, pages 1867--1874. IEEE, 2022{\natexlab{b}}.

\bibitem[Jang et~al.(2023)Jang, Teng, and Ghaffari]{10301632}
Junwoo Jang, Sangli Teng, and Maani Ghaffari.
\newblock Convex geometric trajectory tracking using lie algebraic mpc for autonomous marine vehicles.
\newblock \emph{IEEE Robotics and Automation Letters}, 8\penalty0 (12):\penalty0 8374--8381, 2023.
\newblock \doi{10.1109/LRA.2023.3328450}.

\bibitem[Yang and Carlone(2022)]{yang2022certifiably}
Heng Yang and Luca Carlone.
\newblock Certifiably optimal outlier-robust geometric perception: Semidefinite relaxations and scalable global optimization.
\newblock \emph{{IEEE} Transactions on Pattern Analysis and Machine Intelligence}, 2022.

\bibitem[Lasserre(2001)]{lasserre2001global}
Jean~B Lasserre.
\newblock Global optimization with polynomials and the problem of moments.
\newblock \emph{SIAM Journal on optimization}, 11\penalty0 (3):\penalty0 796--817, 2001.

\bibitem[Lasserre(2015)]{lasserre2015introduction}
Jean~Bernard Lasserre.
\newblock \emph{An introduction to polynomial and semi-algebraic optimization}, volume~52.
\newblock Cambridge University Press, 2015.

\bibitem[Parrilo(2003)]{parrilo2003semidefinite}
Pablo~A Parrilo.
\newblock Semidefinite programming relaxations for semialgebraic problems.
\newblock \emph{Mathematical programming}, 96\penalty0 (2):\penalty0 293--320, 2003.

\bibitem[Carlone(2023)]{carlone2022estimation}
Luca Carlone.
\newblock Estimation contracts for outlier-robust geometric perception.
\newblock \emph{Foundations and Trends (FnT) in Robotics}, 2023.

\bibitem[Hansen(1982)]{hansen1982large}
Lars~Peter Hansen.
\newblock Large sample properties of generalized method of moments estimators.
\newblock \emph{Econometrica: Journal of the econometric society}, pages 1029--1054, 1982.

\bibitem[Hayashi(2011)]{hayashi2011econometrics}
Fumio Hayashi.
\newblock \emph{Econometrics}.
\newblock Princeton University Press, 2011.

\bibitem[Cosse and Demanet(2021)]{cosse2021stable}
Augustin Cosse and Laurent Demanet.
\newblock Stable rank-one matrix completion is solved by the level 2 lasserre relaxation.
\newblock \emph{Foundations of Computational Mathematics}, 21:\penalty0 891--940, 2021.

\bibitem[Cifuentes et~al.(2020{\natexlab{a}})Cifuentes, Harris, and Sturmfels]{cifuentes2020geometry}
Diego Cifuentes, Corey Harris, and Bernd Sturmfels.
\newblock The geometry of sdp-exactness in quadratic optimization.
\newblock \emph{Mathematical programming}, 182\penalty0 (1-2):\penalty0 399--428, 2020{\natexlab{a}}.

\bibitem[Barrau and Bonnabel(2018)]{barrau2018invariant}
Axel Barrau and Silvere Bonnabel.
\newblock Invariant {Kalman} filtering.
\newblock \emph{Annual Review of Control, Robotics, and Autonomous Systems}, 1:\penalty0 237--257, 2018.

\bibitem[Brossard et~al.(2017)Brossard, Bonnabel, and Condomines]{brossard2017unscented}
Martin Brossard, Silvere Bonnabel, and Jean-Philippe Condomines.
\newblock Unscented {Kalman} filtering on {Lie} groups.
\newblock In \emph{Proceedings of the {IEEE}/{RSJ} International Conference on Intelligent Robots and Systems}, pages 2485--2491. IEEE, 2017.

\bibitem[Long et~al.(2013)Long, Wolfe, Mashner, Chirikjian, et~al.]{long2013banana}
Andrew~W Long, Kevin~C Wolfe, Michael~J Mashner, Gregory~S Chirikjian, et~al.
\newblock The banana distribution is gaussian: A localization study with exponential coordinates.
\newblock \emph{Robotics: Science and Systems VIII}, 265:\penalty0 1, 2013.

\bibitem[Grewal and Andrews(2014)]{grewal2014kalman}
Mohinder~S Grewal and Angus~P Andrews.
\newblock \emph{{Kalman} Filtering: Theory and Practice with {MATLAB}}.
\newblock John Wiley \& Sons, 2014.

\bibitem[Smith and Singh(2006)]{smith2006approaches}
Duncan Smith and Sameer Singh.
\newblock Approaches to multisensor data fusion in target tracking: A survey.
\newblock \emph{{IEEE} Transactions on Knowledge and Data Engineering}, 18\penalty0 (12):\penalty0 1696--1710, 2006.

\bibitem[Kalman(1960)]{kalman1960new}
Rudolph~E. Kalman.
\newblock {A New Approach to Linear Filtering and Prediction Problems}.
\newblock \emph{Journal of Basic Engineering}, 82\penalty0 (1):\penalty0 35--45, 03 1960.

\bibitem[Humpherys et~al.(2012)Humpherys, Redd, and West]{humpherys2012fresh}
Jeffrey Humpherys, Preston Redd, and Jeremy West.
\newblock A fresh look at the {Kalman} filter.
\newblock \emph{SIAM review}, 54\penalty0 (4):\penalty0 801--823, 2012.

\bibitem[Thrun(2002)]{thrun2002probabilistic}
Sebastian Thrun.
\newblock Probabilistic robotics.
\newblock \emph{Communications of the ACM}, 45\penalty0 (3):\penalty0 52--57, 2002.

\bibitem[Huang et~al.(2010)Huang, Mourikis, and Roumeliotis]{huang2010observability}
Guoquan~P Huang, Anastasios~I Mourikis, and Stergios~I Roumeliotis.
\newblock Observability-based rules for designing consistent {EKF} {SLAM} estimators.
\newblock \emph{International Journal of Robotics Research}, 29\penalty0 (5):\penalty0 502--528, 2010.

\bibitem[Barrau and Bonnabel(2016)]{barrau2016invariant}
Axel Barrau and Silv{\`e}re Bonnabel.
\newblock The invariant extended {Kalman} filter as a stable observer.
\newblock \emph{{IEEE} Transactions on Automatic Control}, 62\penalty0 (4):\penalty0 1797--1812, 2016.

\bibitem[Wu et~al.(2017)Wu, Zhang, Su, Huang, and Dissanayake]{wu2017invariant}
Kanzhi Wu, Teng Zhang, Daobilige Su, Shoudong Huang, and Gamini Dissanayake.
\newblock An invariant-{EKF} {VINS} algorithm for improving consistency.
\newblock In \emph{Proceedings of the {IEEE}/{RSJ} International Conference on Intelligent Robots and Systems}, pages 1578--1585. IEEE, 2017.

\bibitem[Zhang et~al.(2017)Zhang, Wu, Song, Huang, and Dissanayake]{zhang2017convergence}
Teng Zhang, Kanzhi Wu, Jingwei Song, Shoudong Huang, and Gamini Dissanayake.
\newblock Convergence and consistency analysis for a {3-D} invariant-{EKF} {SLAM}.
\newblock \emph{IEEE Robotics and Automation Letters}, 2\penalty0 (2):\penalty0 733--740, 2017.

\bibitem[Hartley et~al.(2020)Hartley, Ghaffari, Eustice, and Grizzle]{hartley2020contact}
Ross Hartley, Maani Ghaffari, Ryan~M Eustice, and Jessy~W Grizzle.
\newblock Contact-aided invariant extended {Kalman} filtering for robot state estimation.
\newblock \emph{International Journal of Robotics Research}, 39\penalty0 (4):\penalty0 402--430, 2020.

\bibitem[Teng et~al.(2021)Teng, Mueller, and Sreenath]{teng2021legged}
Sangli Teng, Mark~Wilfried Mueller, and Koushil Sreenath.
\newblock Legged robot state estimation in slippery environments using invariant extended kalman filter with velocity update.
\newblock In \emph{2021 IEEE International Conference on Robotics and Automation (ICRA)}, pages 3104--3110. IEEE, 2021.

\bibitem[Lin et~al.(2023)Lin, Li, Tong, and Ghaffari]{lin2023proprioceptive}
Tzu-Yuan Lin, Tingjun Li, Wenzhe Tong, and Maani Ghaffari.
\newblock Proprioceptive invariant robot state estimation.
\newblock \emph{arXiv preprint arXiv:2311.04320}, 2023.

\bibitem[Yu et~al.(2023)Yu, Teng, Chakhachiro, Tong, Li, Lin, Koehler, Ahumada, Walls, and Ghaffari]{yu2023fully}
Xihang Yu, Sangli Teng, Theodor Chakhachiro, Wenzhe Tong, Tingjun Li, Tzu-Yuan Lin, Sarah Koehler, Manuel Ahumada, Jeffrey~M Walls, and Maani Ghaffari.
\newblock Fully proprioceptive slip-velocity-aware state estimation for mobile robots via invariant kalman filtering and disturbance observer.
\newblock In \emph{2023 IEEE/RSJ International Conference on Intelligent Robots and Systems (IROS)}, pages 8096--8103. IEEE, 2023.

\bibitem[van Goor et~al.(2022)van Goor, Hamel, and Mahony]{van2022equivariant}
Pieter van Goor, Tarek Hamel, and Robert Mahony.
\newblock Equivariant filter (eqf).
\newblock \emph{{IEEE} Transactions on Automatic Control}, 2022.

\bibitem[Izanloo et~al.(2016)Izanloo, Fakoorian, Yazdi, and Simon]{izanloo2016kalman}
Reza Izanloo, Seyed~Abolfazl Fakoorian, Hadi~Sadoghi Yazdi, and Dan Simon.
\newblock {Kalman} filtering based on the maximum correntropy criterion in the presence of non-{Gaussian} noise.
\newblock In \emph{Annual Conference on Information Science and Systems (CISS)}, pages 500--505. IEEE, 2016.

\bibitem[Stojanovic and Nedic(2016)]{stojanovic2016robust}
Vladimir Stojanovic and Novak Nedic.
\newblock Robust {Kalman} filtering for nonlinear multivariable stochastic systems in the presence of non-{Gaussian} noise.
\newblock \emph{International Journal of Robust and Nonlinear Control}, 26\penalty0 (3):\penalty0 445--460, 2016.

\bibitem[Sun et~al.(2013)Sun, Duan, Li, and Wang]{sun2013state}
Xu~Sun, Jinqiao Duan, Xiaofan Li, and Xiangjun Wang.
\newblock State estimation under non-{Gaussian} {L{\'e}vy} noise: A modified {Kalman} filtering method.
\newblock \emph{arXiv preprint arXiv:1303.2395}, 2013.

\bibitem[Shimizu et~al.(2023)Shimizu, Jasour, Ghaffari, and Kato]{shimizu2023moment}
Yutaka Shimizu, Ashkan Jasour, Maani Ghaffari, and Shinpei Kato.
\newblock Moment-based {Kalman} filter: Nonlinear {Kalman} filtering with exact moment propagation.
\newblock \emph{arXiv preprint arXiv:2301.09130}, 2023.

\bibitem[Jasour et~al.(2021)Jasour, Wang, and Williams]{jasour2021moment}
Ashkan Jasour, Allen Wang, and Brian~C Williams.
\newblock Moment-based exact uncertainty propagation through nonlinear stochastic autonomous systems.
\newblock \emph{arXiv preprint arXiv:2101.12490}, 2021.

\bibitem[Kahl and Henrion(2007)]{kahl2007globally}
Fredrik Kahl and Didier Henrion.
\newblock Globally optimal estimates for geometric reconstruction problems.
\newblock \emph{International Journal of Computer Vision}, 74:\penalty0 3--15, 2007.

\bibitem[Carlone et~al.(2016)Carlone, Calafiore, Tommolillo, and Dellaert]{carlone2016planar}
Luca Carlone, Giuseppe~C Calafiore, Carlo Tommolillo, and Frank Dellaert.
\newblock Planar pose graph optimization: Duality, optimal solutions, and verification.
\newblock \emph{{IEEE} Transactions on Robotics}, 32\penalty0 (3):\penalty0 545--565, 2016.

\bibitem[Rosen et~al.(2019)Rosen, Carlone, Bandeira, and Leonard]{rosen2019se}
David~M Rosen, Luca Carlone, Afonso~S Bandeira, and John~J Leonard.
\newblock {SE-Sync}: A certifiably correct algorithm for synchronization over the special euclidean group.
\newblock \emph{International Journal of Robotics Research}, 38\penalty0 (2-3):\penalty0 95--125, 2019.

\bibitem[Fredriksson and Olsson(2012)]{fredriksson2012simultaneous}
Johan Fredriksson and Carl Olsson.
\newblock Simultaneous multiple rotation averaging using {Lagrangian} duality.
\newblock In \emph{Asian Conference on Computer Vision}, pages 245--258. Springer, 2012.

\bibitem[Eriksson et~al.(2019)Eriksson, Olsson, Kahl, and Chin]{eriksson2019rotation}
Anders Eriksson, Carl Olsson, Fredrik Kahl, and Tat-Jun Chin.
\newblock Rotation averaging with the chordal distance: Global minimizers and strong duality.
\newblock \emph{{IEEE} Transactions on Pattern Analysis and Machine Intelligence}, 43\penalty0 (1):\penalty0 256--268, 2019.

\bibitem[Aholt et~al.(2012)Aholt, Agarwal, and Thomas]{aholt2012qcqp}
Chris Aholt, Sameer Agarwal, and Rekha Thomas.
\newblock A qcqp approach to triangulation.
\newblock In \emph{Proceedings of the European Conference on Computer Vision}, pages 654--667. Springer, 2012.

\bibitem[Cifuentes(2021)]{cifuentes2021convex}
Diego Cifuentes.
\newblock A convex relaxation to compute the nearest structured rank deficient matrix.
\newblock \emph{SIAM Journal on Matrix Analysis and Applications}, 42\penalty0 (2):\penalty0 708--729, 2021.

\bibitem[Briales and Gonzalez-Jimenez(2017)]{briales2017convex}
Jesus Briales and Javier Gonzalez-Jimenez.
\newblock Convex global {3D} registration with {Lagrangian} duality.
\newblock In \emph{Proceedings of the {IEEE} Conference on Computer Vision and Pattern Recognition}, pages 4960--4969, 2017.

\bibitem[Chaudhury et~al.(2015)Chaudhury, Khoo, and Singer]{chaudhury2015global}
Kunal~N Chaudhury, Yuehaw Khoo, and Amit Singer.
\newblock Global registration of multiple point clouds using semidefinite programming.
\newblock \emph{SIAM Journal on Optimization}, 25\penalty0 (1):\penalty0 468--501, 2015.

\bibitem[Maron et~al.(2016)Maron, Dym, Kezurer, Kovalsky, and Lipman]{maron2016point}
Haggai Maron, Nadav Dym, Itay Kezurer, Shahar Kovalsky, and Yaron Lipman.
\newblock Point registration via efficient convex relaxation.
\newblock \emph{ACM Transactions on Graphics (TOG)}, 35\penalty0 (4):\penalty0 1--12, 2016.

\bibitem[Iglesias et~al.(2020)Iglesias, Olsson, and Kahl]{iglesias2020global}
Jos{\'e}~Pedro Iglesias, Carl Olsson, and Fredrik Kahl.
\newblock Global optimality for point set registration using semidefinite programming.
\newblock In \emph{Proceedings of the {IEEE} Conference on Computer Vision and Pattern Recognition}, pages 8287--8295, 2020.

\bibitem[Agostinho et~al.(2023)Agostinho, Gomes, and Del~Bue]{agostinho2023cvxpnpl}
S{\'e}rgio Agostinho, Jo{\~a}o Gomes, and Alessio Del~Bue.
\newblock Cvxpnpl: A unified convex solution to the absolute pose estimation problem from point and line correspondences.
\newblock \emph{Journal of Mathematical Imaging and Vision}, 65\penalty0 (3):\penalty0 492--512, 2023.

\bibitem[Garcia-Salguero et~al.(2021)Garcia-Salguero, Briales, and Gonzalez-Jimenez]{garcia2021certifiable}
Mercedes Garcia-Salguero, Jesus Briales, and Javier Gonzalez-Jimenez.
\newblock Certifiable relative pose estimation.
\newblock \emph{Image and Vision Computing}, 109:\penalty0 104142, 2021.

\bibitem[Briales et~al.(2018)Briales, Kneip, and Gonzalez-Jimenez]{briales2018certifiably}
Jesus Briales, Laurent Kneip, and Javier Gonzalez-Jimenez.
\newblock A certifiably globally optimal solution to the non-minimal relative pose problem.
\newblock In \emph{Proceedings of the {IEEE} Conference on Computer Vision and Pattern Recognition}, pages 145--154, 2018.

\bibitem[Zhao(2020)]{zhao2020efficient}
Ji~Zhao.
\newblock An efficient solution to non-minimal case essential matrix estimation.
\newblock \emph{{IEEE} Transactions on Pattern Analysis and Machine Intelligence}, 44\penalty0 (4):\penalty0 1777--1792, 2020.

\bibitem[Heller et~al.(2014)Heller, Henrion, and Pajdla]{heller2014hand}
Jan Heller, Didier Henrion, and Tomas Pajdla.
\newblock Hand-eye and robot-world calibration by global polynomial optimization.
\newblock In \emph{Proceedings of the {IEEE} International Conference on Robotics and Automation}, pages 3157--3164. IEEE, 2014.

\bibitem[Giamou et~al.(2019)Giamou, Ma, Peretroukhin, and Kelly]{giamou2019certifiably}
Matthew Giamou, Ziye Ma, Valentin Peretroukhin, and Jonathan Kelly.
\newblock Certifiably globally optimal extrinsic calibration from per-sensor egomotion.
\newblock \emph{IEEE Robotics and Automation Letters}, 4\penalty0 (2):\penalty0 367--374, 2019.

\bibitem[Shi et~al.(2021)Shi, Yang, and Carlone]{shi2021optimal}
Jingnan Shi, Heng Yang, and Luca Carlone.
\newblock Optimal pose and shape estimation for category-level {3D} object perception.
\newblock In \emph{Proceedings of the Robotics: Science and Systems Conference}, 2021.

\bibitem[Yang and Carlone(2020)]{yang2020perfect}
Heng Yang and Luca Carlone.
\newblock In perfect shape: Certifiably optimal {3D} shape reconstruction from {2D} landmarks.
\newblock In \emph{Proceedings of the {IEEE} Conference on Computer Vision and Pattern Recognition}, pages 621--630, 2020.

\bibitem[Yang et~al.(2020)Yang, Shi, and Carlone]{yang2020teaser}
Heng Yang, Jingnan Shi, and Luca Carlone.
\newblock Teaser: Fast and certifiable point cloud registration.
\newblock \emph{{IEEE} Transactions on Robotics}, 37\penalty0 (2):\penalty0 314--333, 2020.

\bibitem[Yang and Carlone(2019)]{yang2019quaternion}
Heng Yang and Luca Carlone.
\newblock A quaternion-based certifiably optimal solution to the wahba problem with outliers.
\newblock In \emph{Proceedings of the {IEEE} International Conference on Computer Vision}, pages 1665--1674, 2019.

\bibitem[Cifuentes et~al.(2020{\natexlab{b}})Cifuentes, Agarwal, Parrilo, and Thomas]{cifuentes2020local}
Diego Cifuentes, Sameer Agarwal, Pablo~A Parrilo, and Rekha~R Thomas.
\newblock On the local stability of semidefinite relaxations.
\newblock \emph{Mathematical Programming}, pages 1--35, 2020{\natexlab{b}}.

\end{thebibliography}
\clearpage

\begin{appendices}

\section{Auxiliary Matrices }
\label{apx:aux_mat}
This appendix formally defines the symmetric matrices $\M{B}_{\gamma}$ and $\M{B}_{\gamma, \alpha}^{\perp}$.
These matrices are used in the definition of the moment constraints in~Section~\ref{sec:preliminaries} and in some of the proofs of our technical results provided below.

We first introduce the matrix $\M B_{\gamma}$ which retrieves the monomial $\M{x}^{\gamma}$ from the moment matrix $\M{M}_r(\M x)$, i.e., $\langle \M B_{\gamma}, \M{M}_r(\M x) \rangle = \M{x}^{\gamma}$:
\begin{equation}
    \M{B}_{\gamma} := \frac{1}{C_{\gamma}} \sum_{\gamma = \alpha + \beta} \ve_{\alpha} \ve^{\transpose}_{\beta},\ \  C_{\gamma}:= { \left| \left\{ (\alpha, \beta) | \gamma = \alpha + \beta \right\}  \right|  },
\end{equation}
where $\ve_\alpha \in \mathbb{R}^{s(r, n) \times 1}$ is the canonical basis corresponding to the $n$-tuples {$\alpha \in \mathbb{N}_r^n$} and $C_{\gamma}$ is a normalizing factor. 
We note that the moment matrix can be written as:
\begin{equation}
    \M{M}_r(\M x) = \sum_{\gamma} C_\gamma \M{B}_{\gamma} \M x^{\gamma}, \quad \gamma \in \mathbb{N}^n_{2r}.
\end{equation}
Then we introduce $\M{B}^{\perp}_{\gamma, \alpha}$, which is used to enforce that repeated entries {(corresponding to the same monomial $\M{x}^{\gamma}$)} in the moment matrix are identical. We first define the symmetric matrix, for any given $\alpha, \beta \in \mathbb{N}_r^n$:
\begin{equation}
    \M{E}_{\alpha, \beta} = \ve_{\alpha}\ve_{\beta}^{\transpose} + \ve_{\beta}\ve_{\alpha}^{\transpose},
\end{equation}
Then we construct  $\M{B}^{\perp}_{\gamma, \alpha}$, for some given $\gamma\in \mathbb{N}_{2r}^{n}$ and for $\alpha \in \mathbb{N}_r^n$ (with $\alpha < \gamma$) as:
\begin{equation}
   \begin{aligned}
       \M{B}^{\perp}_{\gamma, \alpha} = & { \M{E}_{\bar{\alpha}, \bar{\beta}} } - \M{E}_{{\alpha}, {\beta}}, \\ 
       \end{aligned}
\end{equation}
where $\beta$ is such that $\alpha + \beta = \gamma$ and $\beta \leq \alpha$, 
and where $\bar{\alpha}, \bar{\beta} \in \mathbb{N}_r^n$ are chosen such that $\bar{\alpha}+\bar{\beta}= \gamma$ 
and $\bar{\alpha}$ is the smallest vector (in the lexicographic sense) that satisfies $\bar{\alpha}+\bar{\beta}= \gamma$.
An example of matrices $\M B_{\gamma}$ and $\M{B}_{\gamma, \alpha}^{\perp}$ for a simple problem is given in Appendix~\ref{apx:B_mat_example}. 

\section{Examples of moment matrix}
\label{apx:B_mat_example}
For $r = 2, n = 2$, we have the moment matrix:
\begin{equation*}
    {\M M}_{2}(\M x) =\left[\begin{array}{cccccc}
1 & x_1 & x_2 & x_1^2 & x_1x_2 & x_2^2 \\
x_1 & x_1^2 & x_1x_2 & x_1^3 & x^2_1x_2 & x_1x_2^2 \\
x_2 & x_1x_2 & x^2_2 & x_1^2x_2 & x_1x^2_2 & x_2^3 \\
x_1^2 & x^3_1 & x_1^2x_2 & x_1^4 & x^3_1x_2 & x_1^2x_2^2 \\
x_1x_2 & x_1^2x_2 & x_1x^2_2 & x_1^3x_2 & x_1^2x_2^2 & x_1x_2^3 \\
x_2^2 & x_1x_2^2 & x^3_2 & x_1^2x_2^2 & x_1x^3_2 & x_2^4 \\
\end{array}\right] .
\end{equation*} 
For $\gamma = [2, 0]$, which corresponds to $x_1^2$, we have:
\begin{equation}
    \M B_{\gamma} = \frac{1}{C_{\gamma}} \begin{bmatrix}
        0 & 0 & 0 & 1 & 0 & 0 \\
        0 & 1 & 0 & 0 & 0 & 0 \\
        0 & 0 & 0 & 0 & 0 & 0 \\
        1 & 0 & 0 & 0 & 0 & 0 \\
        0 & 0 & 0 & 0 & 0 & 0 \\
        0 & 0 & 0 & 0 & 0 & 0 \\
    \end{bmatrix}, C_{\gamma} = 3.
\end{equation}
Now let us compute $\M B^{\perp}_{\gamma, \alpha}$ for $\alpha = [1,0] \leq \gamma$.
We note that $\beta = \gamma - \alpha = [1,0]$ and we can choose $\bar{\alpha} = [0,0]$ and $\bar{\beta} = [2,0]$.
The corresponding vectors in the canonical basis are:
\begin{equation}
\begin{aligned}
    \M e_{{\alpha}} &= \begin{bmatrix}
        0 & 1 & 0 & 0 & 0 & 0
    \end{bmatrix}, \M e_{{\beta}} &= \begin{bmatrix}
        0 & 1 & 0 & 0 & 0 & 0
    \end{bmatrix}, \\
     \M e_{\bar{\alpha}} &= \begin{bmatrix}
        1 & 0 & 0 & 0 & 0 & 0
    \end{bmatrix}, {\M e_{\bar{\beta}} } &= \begin{bmatrix}
        0 & 0 & 0 & 1 & 0 & 0
    \end{bmatrix},
\end{aligned}
\end{equation}
from which we obtain:
\begin{equation}
\begin{aligned}
    \M B^{\perp}_{\gamma, \alpha} &=  {\M E_{\bar{\alpha}, \bar{\beta}} - \M E_{{\alpha}, {\beta}} }  =\M e_{\bar{\alpha}}\M e_{\bar{\beta}}^{\transpose} + \M e_{\bar{\beta}}\M e_{\bar{\alpha}}^{\transpose}  - (\M e_{\alpha}\M e_{\beta}^{\transpose} + \M e_{\beta}\M e_{\alpha}^{\transpose})  \\
    &= \begin{bmatrix}
        0 & 0 & 0 & +1 & 0 & 0 \\
        0 & -2 & 0 & 0 & 0 & 0 \\
        0 & 0 & 0 & 0 & 0 & 0 \\
        +1 & 0 & 0 & 0 & 0 & 0 \\
        0 & 0 & 0 & 0 & 0 & 0 \\
        0 & 0 & 0 & 0 & 0 & 0 \\
    \end{bmatrix},
\end{aligned}
\end{equation}
By adding the constraint:
\begin{equation}
    \langle \M B^{\perp}_{\gamma, \alpha}, \M X \rangle = 0, \quad \M X \succeq 0,
\end{equation}
we enforce that the entries corresponding to $x_1^2$ are identical. 

\section{\Cref{lemma:B_orthogonal}}
\label{apx:proof:B_orthogonal}
\begin{restatable}{lemmma}{lemmaBorthogonal}
\label{lemma:B_orthogonal}
 Let us define the set of matrices $\{\M B\} := \{\M B_\gamma\}_{\gamma \in \mathbb{N}_{2r}^n}$ 
 and $\{\M B_{\gamma,\alpha}^{\perp}\} := \{\M B^{\perp}\}_{\gamma \in \mathbb{N}_{2r}^n, \alpha \in \mathbb{N}_{r}^n, \alpha\leq \gamma}$.
 Then $\{\M B\}$ and $\{\M B_{\gamma,\alpha}^{\perp}\}$ are orthogonal and form a basis for the set of symmetric matrices in $\mathbb{R}^{s(r, n) \times s(r, n)}$.
\end{restatable}

\begin{proof}
To show orthogonality, we observe that for $\gamma \neq \gamma'$
$$\left\langle \M B_{\gamma, \alpha}^{\perp}, \M B_{\gamma'} \right\rangle = 0, \forall \alpha,$$ 
which follows from the fact that non-zero entries in the two matrices do not appear at the same locations. 
For $\gamma = \gamma'$:
\begin{equation}
\label{eq:B_orthogonal}
\begin{aligned}
    \left\langle \M B_{\gamma, \alpha}^{\perp}, \M B_{\gamma} \right\rangle &= \langle \M E_{\bar{\alpha}, \bar{\beta}} - \M E_{{\alpha}, {\beta}},  \frac{1}{C_{\gamma}} \sum_{\gamma = \alpha' + \beta'} \M e_{\alpha'} \M e^{\transpose}_{\beta'} \rangle \\
    &= \langle \M e_{\bar{\alpha}}\M e_{\bar{\beta}}^{\transpose} + \M e_{\bar{\beta}}\M e_{\bar{\alpha}}^{\transpose},  \frac{1}{C_{\gamma}} \sum_{\gamma = \alpha' + \beta'} \M e_{\alpha'} \M e^{\transpose}_{\beta'}\rangle \\ & \quad -  \langle \M e_{{\alpha}}\M e_{{\beta}}^{\transpose} + \M e_{{\beta}}\M e_{{\alpha}}^{\transpose},  \frac{1}{C_{\gamma}} \sum_{\gamma = \alpha' + \beta'} \M e_{\alpha'} \M e^{\transpose}_{\beta'} \rangle \\
    & = \frac{2}{C_{\gamma}} - \frac{2}{C_{\gamma}} = 0.
\end{aligned}.
\end{equation} 
where we observed that $\langle \M e_{{\alpha}}\M e_{{\beta}}^{\transpose},  \frac{1}{C_{\gamma}} \sum_{\gamma = \alpha' + \beta'} \M e_{\alpha'} \M e^{\transpose}_{\beta'} \rangle = 
\langle \M e_{{\alpha}}\M e_{{\beta}}^{\transpose},  \frac{1}{C_{\gamma}} \M e_{{\alpha}}\M e_{{\beta}}^{\transpose} \rangle = \frac{1}{C_\gamma}$ (since any other choice of $\alpha',\beta'$ will produce a matrix $\M e_{{\alpha'}}\M e_{{\beta'}}^{\transpose}$ orthogonal to $\M e_{{\alpha}}\M e_{{\beta}}^{\transpose}$).

To show $\left\{\M B\right\} \cup \left\{\M B^{\perp}\right\}$ form a basis for the set of symmetric matrices in $\mathbb{R}^{s(r, n) \times s(r, n)}$, 
{by the structure of $\{\M{B}\}$ we have:}
\begin{equation}
\label{eq:B_basis}
\begin{aligned}
C_{\gamma}\M B_{\gamma}  &= \sum_{\gamma = \alpha + \beta}\M e_{\alpha}\M e_{\beta}^{\transpose} = 
  \M e_{\bar{\alpha}}\M e^{\transpose}_{\bar{\beta}} + \sum_{\gamma = \alpha + \beta, \alpha \neq \bar{\alpha}}\M e_{\alpha}\M e_{\beta}^{\transpose}.
\end{aligned}
\end{equation}

Adding its transpose to each side of~\eqref{eq:B_basis}, and rearranging:  
\begin{equation}
\begin{aligned}
    \M E_{\bar{\alpha}, \bar{\beta}} &= 2C_{\gamma}\M B_{\gamma} - \sum_{\gamma = \alpha + \beta, \alpha \neq \bar{\alpha}}\M E_{\alpha, \beta} \\
    &= 2C_{\gamma}\M B_{\gamma} - \sum_{\gamma = \alpha + \beta, \alpha \neq \bar{\alpha}}(\M B^{\perp}_{\gamma, \min\{\alpha, \beta\} } - \M E_{\bar{\alpha}, \bar{\beta}}).
\end{aligned}
\end{equation}
 Finally, we show that 
\begin{equation}
\label{eq:E_alpha_beta}
    \M E_{\bar{\alpha}, \bar{\beta}} = -\frac{2C_{\gamma}}{C_{\gamma}^{'} - 1}\M{B}_{\gamma} + \frac{1}{C_{\gamma}^{'} - 1}\sum_{\gamma = \alpha + \beta, \alpha \neq \bar{\alpha}}\M{B}_{\gamma, \min\{\alpha, \beta\}}^{\perp}
\end{equation}
with $C_{\gamma}^{'} = |\{ (\alpha, \beta) | \alpha + \beta = \gamma, \alpha \neq \bar{\alpha} \} |$.

By the definition of $\{\M{B}^{\perp}\}$, $\forall \alpha, \beta$ and $\alpha + \beta = \gamma$, we have:
\begin{equation}
\begin{aligned}\label{eq:appC1}
    \M E_{\alpha, \beta} &= \M E_{\bar{\alpha}, \bar{\beta}} - \M B^{\perp}_{\gamma, \alpha}
\end{aligned}.
\end{equation}
Substituting \eqref{eq:E_alpha_beta} into \eqref{eq:appC1}, we show that any  $\M{E}_{\alpha, \beta}$, can be represented by as a linear combination of elements in $\left\{\M B\right\} \cup \left\{\M B^{\perp}\right\}$. 
Since the matrices $\M{E}_{\alpha, \beta}$ (for any $\alpha, \beta \in \mathbb{N}_{r}^n$) form a basis for the set of symmetric matrices in $\mathbb{R}^{s(r, n) \times s(r, n)}$, this proves our claim.
\end{proof}

\section{Proof of \Cref{lemma:exist_sos_xi}}
\label{apx:proof:exist_sos_xi}
{
\begin{restatable}[Equivalent SOS Polynomial Descriptions]{lemmma}{lemmaexistsosxi}
    \label{lemma:exist_sos_xi}
    Let $\M{Y}_1, \M{Y}_2$ be two symmetric matrices such that $\vv^{\transpose}_r(\M x) \M{Y}_1 \vv^{\transpose}_r(\M x)=\vv^{\transpose}_r(\M x) \M{Y}_2 \vv^{\transpose}_r(\M x), \forall \M x$, i.e., the polynomials corresponding to $\M{Y}_1$ and $\M{Y}_2$ are identical. Then, there exists a matrix $\M{\xi}$ such that $\M{Y}_1=\M{Y}_2+{\M{\xi}}$,
    where $\M{\xi}=\sum_{i} \mu_{i} \M{B}^{\perp}_{i}, \mu_i \in \mathbb{R}$. 
\end{restatable}
}
\begin{proof}
We prove \Cref{lemma:exist_sos_xi} using the result of \Cref{lemma:B_orthogonal}:
        \begin{align}
        &\M v^{\transpose}_r(\M x) (\M Y_1-\M Y_2) \M v_r( \M x)=0, \quad \forall \M x \Leftrightarrow  \\
        &\left\langle \M Y_1-\M Y_2, \M v_r( \M x)\M v^{\transpose}_r(\M x)\right\rangle=0, \quad \forall \M x \Leftrightarrow \\
        &\langle \M Y_1-\M Y_2, \sum \M x^{\gamma}\M B_{\gamma}\rangle=0, \quad \forall \M x \Rightarrow \\
        &\left\langle \M Y_1-\M Y_2, {\M B}_{{\gamma}}\right\rangle= 0, \quad \forall \gamma
        \end{align}
As $\left\{\M B^{\perp}\right\}$ and $\left\{\M B\right\}$  are orthogonal and form a basis for symmetric matrices with compatible size as $\M{Y}_1, \M{Y}_2$, there exists an $\M\xi$ which can be constructed from $\{\M B^{\perp}\}$ such that
    \begin{equation}
    \label{eq:exist_of_xi}
        \M \xi := \M Y_1 - \M Y_2 = \sum_{i} \mu_{i} \M B^{\perp}_{i}
    \end{equation}
\end{proof}
This proof is equivalent to the version in \cite{cosse2021stable} but with the application of $\left\{ \M B^{\perp} \right\}$. 

\section{Example of Extended Polynomial System \\ and Its Affine Form}
\label{appx:momcond_linear}
We provide an example of the extended polynomial system and transform it to the affine form \eqref{eq:gmm_cond} used in \eqref{eq:BPUE}:

Consider a linear measurement model with state $\M x \in \mathbb{R}^2$, measurements $\M y$, and measurement noise $\M w$:
\begin{equation}
\begin{aligned}
    y_1 - x_1 &= w_1 \\
    y_2 - x_2 &= w_2
\end{aligned}
\end{equation}
The extended system at $r=2$ is: 
\begin{equation}
\M \phi_r(\M y - \M x) =  \M \phi_r(\M w)
\end{equation}
where
\begin{equation}
\begin{aligned}
     \M \phi_r(\M y - \M x) &= \begin{bmatrix}
            y_1 - x_1\\
            y_2 - x_2\\
            (y_1-x_1)(y_2-x_2)\\
            (y_1-x_1)^2\\
            (y_2-x_2)^2
        \end{bmatrix}\\ &= \begin{bmatrix}
            y_1 - x_1\\
            y_2 - x_2\\
            y_1y_2 + x_1x_2 - y_1x_2 - y_2x_1\\
            y_1^2  + x_1^2 - 2y_1x_1\\
            y_2^2  + x_2^2 - 2y_2x_2
        \end{bmatrix}
\end{aligned}
\end{equation}
and
\begin{equation}
    \begin{aligned}
        \M \phi_r(\M w) = \begin{bmatrix}
            w_1\\w_2\\w_1w_2\\w_1^2\\w_2^2
        \end{bmatrix}
    \end{aligned}.
\end{equation}

We can use an affine function to express the moment condition $ \M \phi_r(\M y - \M x) - \mathbb{E}[ \M \phi_r(\M w) ] =  \M \phi_r(\M w) - \mathbb{E}[ \M \phi_r(\M w) ]$:
\begin{equation}
    \vb(\M y) - \M{A}(\M y)\M \phi_r(\M x) = \vv, \quad \M V = \operatorname{Cov}[\vv],
\end{equation}
where
\begin{equation}
\begin{aligned}
    \M b(\M y) &= \begin{bmatrix}
        y_1 - \mathbb{E}[w_1]\\
        y_2 - \mathbb{E}[w_2]\\
         y_1y_2 - \mathbb{E}[w_1w_2]\\
        y_1^2 - \mathbb{E}[w_1^2]\\
        y^2_2 - \mathbb{E}[w_2^2]\\
    \end{bmatrix},     \M v = \begin{bmatrix}
        w_1 - \mathbb{E}[w_1]\\
        w_2 - \mathbb{E}[w_2]\\
         w_1w_2 - \mathbb{E}[w_1w_2]\\
        w_1^2 - \mathbb{E}[w_1^2]\\
        v^2_2 - \mathbb{E}[w_2^2]\\
    \end{bmatrix}, \\
    \M A(\M y) &= \begin{bmatrix}
        1&0&0&0&0\\
        0&1&0&0&0\\
        y_2&y_1&-1&0&0\\
        2y_1&0&0&-1&0\\
        0&2y_2&0&0&-1
    \end{bmatrix}, 
    \M \phi_r(\M x) = \begin{bmatrix}
        x_1\\x_2\\x_1x_2\\x_1^2\\x_2^2
        \end{bmatrix}.
\end{aligned}
\end{equation}

\section{Proof of \Cref{prop:blue-quad}}
\begin{proof}
 The proof proceeds by simple inspection. We first develop the term $||\M A \M x - \M b||_{\M{V}^{-1}}^2$ as:
    \begin{equation}\label{eq:appF1}
    \begin{aligned}
        ||\M A \M x - \M b||_{\M{V}^{-1}}^2 &= \M x^\transpose \M A^\transpose \M{V}^{-1} \M A \M x - 2\M x^\transpose \M A^\transpose\M{V}^{-1}\M b + \M b ^\transpose \M{V}^{-1}\M b.
    \end{aligned}
    \end{equation}
    Now let us substitute $\hat{\M x} = ( \M A^\transpose \M{V}^{-1} \M A)^{-1} \M A^\transpose \M{V}^{-1} \M b$ 
    and its covariance $\M{\Sigma} = (\M A^\transpose \M{V}^{-1} \M A)^{-1}$ (which is invertible under the assumption that $\M A$ is full column rank), 
    back into the term $||\M x - \hat{\M x}||_{\M{\Sigma}^{-1}}^2$ and develop the square:
    \begin{equation}
        \begin{aligned}
            &\quad||\M x - (\M A^\transpose \M{V}^{-1} \M A)^{-1}\M A^\transpose \M{V}^{-1} \M b||^2_{\M A^\transpose \M{V}^{-1} \M A}\\
            &= \left(\M x - (\M A^\transpose \M{V}^{-1} \M A)^{-1}\M A^\transpose \M{V}^{-1} \M b\right)^\transpose\M A^\transpose \M{V}^{-1} \M A\\
            &\quad\cdot\left(\M x - (\M A^\transpose \M{V}^{-1} \M A)^{-1}A^\transpose \M{V}^{-1} \M b\right)\\
            &=(\M x^\transpose(\M A^\transpose \M{V}^{-1} \M A) - \M b^\transpose \M{V}^{-1} \M A) \\
            &\quad \cdot (\M x - (\M A^\transpose \M{V}^{-1} \M A)^{-1}\M A^\transpose \M{V}^{-1} \M b) \nonumber
        \end{aligned}
    \end{equation}
    \begin{equation}\label{eq:appF2}
        \begin{aligned}
            &= \M x^\transpose\M A^\transpose \M{V}^{-1} \M A\M x - 2\M x^\transpose \M A^\transpose\M{V}^{-1}\M b\\
            &\quad+ \M b^\transpose\M{V}^{-1} \M A(\M A^\transpose \M{V}^{-1} \M A)^{-1}\M A^\transpose \M{V}^{-1} \M b.
        \end{aligned}
    \end{equation}

By comparing~\eqref{eq:appF1} and~\eqref{eq:appF2} we conclude that:
\begin{equation}
    \begin{aligned}
         \| \M{A}\M{x} - \M{b} \|^2_{\M{V}^{-1}} &= \| \M x - \hat{\M x} \|^2_{\M{\Sigma}^{-1}} + \hat{\rho},
         \end{aligned}
\end{equation}
where $\hat{\rho} =: \M b^\transpose \M{V}^{-1} \M b - \M b^\transpose\M{V}^{-1}\M A(\M A^\transpose \M{V}^{-1} \M A)^{-1}\M A^\transpose \M{V}^{-1} \M b$.
Now we note that $\hat{\rho}$ is the optimal objective of~\eqref{eq:BLUE}: this can be seen by substituting the estimate $\hat{\M x}$ (which is optimal by construction) back into the \eqref{eq:BLUE} objective:
\begin{equation}
    \begin{aligned}
         &\| \M{A} (\M A^\transpose \M{V}^{-1} \M A)^{-1}\M A^\transpose \M{V}^{-1} \M b - \M{b} \|^2_{\M{V}^{-1}} \\
        =& 
        \M b^\transpose \M{V}^{-1} \M A (\M A^\transpose \M{V}^{-1} \M A)^{-1} \M{A}^\transpose
        \M{V}^{-1} \M{A} (\M A^\transpose \M{V}^{-1} \M A)^{-1}\M A^\transpose \M{V}^{-1} \M b \\ 
        & + \M{b}^\transpose \M{V}^{-1} \M{b} - 2 \M b^\transpose \M{V}^{-1} \M{A} (\M A^\transpose \M{V}^{-1} \M A)^{-1}\M A^\transpose \M{V}^{-1} \M b 
        \\
        =&
        \M b^\transpose \M{V}^{-1} \M A (\M A^\transpose \M{V}^{-1} \M A)^{-1}\M A^\transpose \M{V}^{-1} \M b \\ 
        & + \M{b}^\transpose \M{V}^{-1} \M{b} - 2 \M b^\transpose \M{V}^{-1} \M{A} (\M A^\transpose \M{V}^{-1} \M A)^{-1}\M A^\transpose \M{V}^{-1} \M b 
        \\
        =&
        \M{b}^\transpose \M{V}^{-1} \M{b} - \M b^\transpose \M{V}^{-1} \M{A} (\M A^\transpose \M{V}^{-1} \M A)^{-1}\M A^\transpose \M{V}^{-1} \M b = \hat{\rho}.
    \end{aligned}
\end{equation}
which concludes the proof.
\end{proof}

\section{Proof of \Cref{theorem:SDP=POP}}
\label{apx:SDP=SOS}

\begin{proof}
    Before the main proof, we first convert the optimality condition 1) in \Cref{def:kkt_sdp} to their polynomial form.  Taking inner products on both sides of the stationary condition 1) in \Cref{def:kkt_sdp} with ${\M X}:= \M v_r(\M x)\M v^{\transpose}_r(\M x)$:
    $$\langle\M Y, \M{X} \rangle= \langle \M C + \sum_{j} \lambda_{j} \M G_{j} + \big( \sum_{i}\mu_{i} \M B^{\perp}_{i} - \rho \M e_1\M e_1^{\transpose} \big), \M{X} \rangle$$
    we have the optimality condition in the polynomial form \cite{lasserre2001global, lasserre2015introduction}:
\begin{equation}
    \M v^{\transpose}_r(\M x)\M Y\M v_r(\M x) = p(\M x) + \sum_j h_j(\M x)g_j(\M x) - \rho,
\end{equation}
where the term corresponding to $\M B^{\perp}$ vanishes by construction, and $h_i(\M x)$ are the polynomials defined by the localizing matrix $\M G_i$ and associated multiplier $\lambda_i$. Now we start the main proof. 

    $\Longrightarrow:\;$ 
    We prove that when the primal and dual solutions take the form described in \Cref{theorem:SDP=POP}, then they must satisfy the optimality conditions in~\Cref{def:kkt_sdp}.
    Suppose the dual solution (written as a polynomial) has the form:
    \begin{equation}
        \begin{aligned}
        \M v^{\transpose}_r(\M x)\hat{\M{Y}}\M v_r(\M x) &= \|\M \phi_r(\M x) - \M \phi_r(\hat{\M x})\|^2_{\M{\Sigma}^{-1}}, \quad  \M{\Sigma} \succ 0 
        \end{aligned}.
    \end{equation}
    with: 
    $$\hat{\M{Y}}:= \begin{bmatrix}
        \M \phi_r(\hat{\M x})^{\transpose}\M{\Sigma}^{-1}\M \phi_r(\hat{\M x}) & -(\M{\Sigma}^{-1}\M \phi_r(\hat{\M x}))^{\transpose} \\
            -\M{\Sigma}^{-1}\M \phi_r(\hat{\M x}) & \M{\Sigma}^{-1}
        \end{bmatrix},$$
    and:
    $$
        \hat{\M{X}}:= \begin{bmatrix}
            1 \\
            \M \phi_r(\hat{\M x})
        \end{bmatrix}\begin{bmatrix}
            1 \\
            \M \phi_r(\hat{\M x})
        \end{bmatrix}^{\transpose}.
    $$
    First, we observe that $\hat{\M X}, \hat{\M Y} \succeq 0$ and that ---by simple inspection--- $\langle \hat{\M X}, \hat{\M Y} \rangle = 0$; hence these matrices satisfy condition 3) in \Cref{def:kkt_sdp}.
    Second, it is straightforward to verify that $\operatorname{Rank}(\hat{\M{X}}) = 1$, $\operatorname{Rank}(\hat{\M{Y}}) = s(r, n) - 1$ hence the matrices satisfy condition 4) in \Cref{def:kkt_sdp}. By the rank condition, the null space of $\hat{\M Y}$ must be in the span of $\M {v}_r(\hat{\M x})$. Thus $\hat{\M{X}}$ is the unique moment matrix that satisfies the conditions 3) and 4). 
    Third, we note that the dual solution exists and is bounded. Thus the primal feasibility must be satisfied to ensure the multiplier not going to infinity. Then we conclude that $\hat{\M X}$ is a feasible solution, hence it satisfies condition 2) in \Cref{def:kkt_sdp}. 

    To verify the satisfaction of the optimality condition 1), we write
     the optimality condition in polynomial form:
    \begin{equation}\label{eq:appG1}
    \M v^{\transpose}_r(\M x)\hat{\M{Y}}\M v_r(\M x) = \M v^{\transpose}_r(\M x)\left(\M C -\rho \M e_1\M e_1^{\transpose} + \sum_{j} \lambda_{j} \M G_{j}\right)\M v_r(\M x). 
    \end{equation}
    By Lemma \ref{lemma:exist_sos_xi}, there must exist $\M \xi = \sum_{i} \mu_{i} \M B^{\perp}_{i}$ that relates the matrices on the left and right-hand side of~\eqref{eq:appG1}: 
    $$\hat{\M{Y}} = \M C - \rho \M e_1\M e_1^{\transpose} + \sum_{j} \lambda_{j} \M G_{j} + \sum_{i} \mu_{i} \M B^{\perp}_{i},$$
    which shows that the dual solution $\hat{\M{Y}}$ also satisfies the optimality condition 1) in \Cref{def:kkt_sdp}. 
    
     $\Longleftarrow:\;$
     We prove that any primal-dual pair that satisfies the optimality conditions in \Cref{def:kkt_sdp} must be in the form described in \Cref{theorem:SDP=POP}.  
    If the SDP has a rank-1 solution and satisfies the strict complementary condition in \Cref{def:kkt_sdp}, we have $\langle \hat{\M{X}}, \hat{\M{Y}}\rangle = 0$, with $\operatorname{Rank}(\hat{\M{X}}) = 1, \operatorname{Rank}(\hat{\M{Y}}) = s(r, n) - 1$. By SVD of $\hat{\M{Y}}$, we have:

    \begin{equation}
        \begin{aligned}
            \hat{\M{Y}} &= \sum_{i=1}^{s(r, n) - 1} \delta_i \M f_i\M f_i^{\transpose} \\
              &= \begin{bmatrix}
                  \sqrt{\delta_1}\M f^{\transpose}_1 \\
                  \vdots \\
                  \sqrt{\delta_{s(r, n) - 1}}\M f^{\transpose}_{{s(r, n) - 1}} \\
              \end{bmatrix}^{\transpose}
              \begin{bmatrix}
                  \sqrt{\delta_1}\M f^{\transpose}_1 \\
                  \vdots \\
                  \sqrt{\delta_{s(r, n) - 1}}\M f^{\transpose}_{{s(r, n) - 1}} \\
              \end{bmatrix} \\
              &=: \M F\M F^{\transpose},
        \end{aligned}
    \end{equation}
    where $\M f_i \in \mathbb{R}^{s(r, n)\times 1}$ are mutually orthogonal and $\delta_i > 0$ are non-trivial singular values. Similarly, by SVD of $\hat{\M{X}}$ and $\langle \hat{\M{X}}, \hat{\M{Y}} \rangle = 0$, we have:
    \begin{equation}
        \hat{\M{X}} =  \begin{bmatrix}
            1 \\
            \M \phi_r(\hat{\M x})
        \end{bmatrix}\begin{bmatrix}
            1 \\
            \M \phi_r(\hat{\M x})
        \end{bmatrix}^{\transpose}, 
        \text{ and }
        \begin{bmatrix}
            1 \\
            \M \phi_r(\hat{\M x})
        \end{bmatrix}^{\transpose}\M f_i = 0, \forall i.
    \end{equation}
Let us partition $\M F$ by:
    \begin{equation}
    \begin{aligned}
        \M F &= \begin{bmatrix}
            \M f \\
            \bar{\M F}
        \end{bmatrix}
    \end{aligned}
    \end{equation}
where $\M f\in\mathbb{R}^{1\times (s(r,n) - 1) }, \bar{\M F} \in \mathbb{R}^{(s(r, n) - 1) \times (s(r, n) - 1) }$. By the strict complementary condition $\langle \hat{\M{X}}, \hat{\M{Y}} \rangle = 0$, we have
    \begin{equation}
    \begin{aligned}
        \M F^{\transpose}\begin{bmatrix}
            1 \\
            \M \phi_r(\hat{\M x})
        \end{bmatrix} &= \begin{bmatrix}
            \M f^{\transpose}, \bar{\M F}^{\transpose}
        \end{bmatrix}\begin{bmatrix}
            1 \\
            \M \phi_r(\hat{\M x})
        \end{bmatrix} \\
        &= \M f^{\transpose} + \bar{\M F}^{\transpose}
            \M \phi_r(\hat{\M x}) = 0.
    \end{aligned}
    \end{equation}
    \begin{equation}
        \M f^{\transpose} = - \bar{\M F}^{\transpose}\M \phi_r(\hat{\M x})
    \end{equation}
    As ${\M f}^{\transpose}$ is a linear combination of the rows of $\bar{\M F}$, the matrix $\bar{\M F}$ must have full rank to ensure that $\operatorname{Rank}(\M F) = s(r, n) - 1$. We can then obtain the SOS belief via:
    \begin{equation}
    \begin{aligned}
        &\sigma(\hat{\M{x}}, \M\Sigma)  \\
        &:= \begin{bmatrix}
            1 \\
            \M \phi_r(\M x)
        \end{bmatrix}^{\transpose}\M F\M F^{\transpose}\begin{bmatrix}
            1 \\
            \M \phi_r(\M x)
        \end{bmatrix} \\
        &= \begin{bmatrix}
            1 \\
            \M \phi_r(\M x)
        \end{bmatrix}^{\transpose}
        \begin{bmatrix}
            \M f\M f^{\transpose} & \M f\bar{\M F}^{\transpose} \\
            \bar{\M F}\M f^{\transpose} & {\bar{\M F}\bar{\M F}^{\transpose}}\\
        \end{bmatrix}
        \begin{bmatrix}
            1 \\
            \M \phi_r(\M x)
        \end{bmatrix} \\
    \end{aligned} \nonumber
    \end{equation}
    \begin{equation}
    \begin{aligned}
        &= \begin{bmatrix}
            1 \\
            \M \phi_r(\M x)
        \end{bmatrix}^{\transpose}
        \begin{bmatrix}
            \M \phi^{\transpose}_r(\hat{\M x}){\bar{\M F}\bar{\M F}^{\transpose}}\M \phi_r(\hat{\M x}) & -\M \phi^{\transpose}_r(\hat{\M x}){\bar{\M F}\bar{\M F}^{\transpose}}\\
            -{\bar{\M F}\bar{\M F}^{\transpose}}\M \phi_r(\hat{\M x}) & {\bar{\M F}\bar{\M F}^{\transpose}}\\
        \end{bmatrix}
        \begin{bmatrix}
            1 \\
            \M \phi_r(\M x)
        \end{bmatrix} \\
        & = \|\M \phi_r(\M x) - \M \phi_r(\hat{\M x})\|^2_{{\bar{\M F}\bar{\M F}^{\transpose}}} \\
    \end{aligned}
    \end{equation}

    As $\bar{\M F}$ is full rank thus is bijective, $ \M s^{\transpose}\bar{\M F}\bar{\M F}^{\transpose}\M s = 0 \Leftrightarrow \bar{\M F}^{\transpose}\M s = 0 \Leftrightarrow \M s \in \operatorname{Null}(\bar{\M F}^{\transpose}) \Leftrightarrow \M s = 0 \Leftrightarrow {\bar{\M F}\bar{\M F}^{\transpose}}\succ 0$. Then we let $ \M{\Sigma} = ({\bar{\M F}\bar{\M F}^{\transpose}})^{-1}$ and have the SOS belief: $$\|\M \phi_r(\M x) - \M \phi_r(\hat{\M x})\|^2_{\M{\Sigma}^{-1}}, \quad \M{\Sigma} \succ 0.$$
    Furthermore, can see that the dual solution takes the form:
    \begin{equation}
        \hat{\M{Y}} = \begin{bmatrix}
            \M \phi^{\transpose}_r(\hat{\M x}){\bar{\M F}\bar{\M F}^{\transpose}}\M \phi_r(\hat{\M x}) & -\M \phi^{\transpose}_r(\hat{\M x}){\bar{\M F}\bar{\M F}^{\transpose}}\\
            -{\bar{\M F}\bar{\M F}^{\transpose}}\M \phi_r(\hat{\M x}) & {\bar{\M F}\bar{\M F}^{\transpose}}\\
        \end{bmatrix}.
    \end{equation}
    which matches the form in \Cref{theorem:SDP=POP} with $ \M{\Sigma} = ({\bar{\M F}\bar{\M F}^{\transpose}})^{-1}$.
\end{proof}

\section{Derivation of the Optimal Solution to BPUE under Constraints}
\label{appx:bpue-sos-cons}
\begin{theorem}[SOS Belief of Constrained BPUE]
In the constrained case, where $\M x_k \in \mathcal{K}$, the SOS belief includes additional polynomial multipliers:
\begin{equation*}
        \|\M \phi_r(\M x) - \M \phi_r(\hat{\M x})\|^2_{\M \Sigma^{-1}} + \rho = \|\M A \phi_r(\M x) - \M b\|^2_{\M V^{-1}} + \sum_j h_j(\M x)g_j(\M x),
\end{equation*}
with $\M{\Sigma}\succ 0$ and where $h_j(\M{x})$ are the multipliers. For a feasible point $\M x_k \in \mathcal{K}$, we have $g_j(\M x_k) = 0$, thus $\M{\Sigma}\succ 0$ still behaves as an approximated covariance on the support $\mathcal{K}$. 
\end{theorem}
\begin{proof}
Under \Cref{assump:rank1sdp}, 
by the stationary condition 1) in \Cref{def:kkt_sdp} with ${\M X}:= \M v_r(\M x)\M v^{\transpose}_r(\M x)$, we have
\begin{equation}
\begin{aligned}
\M v^{\transpose}_r( \M x) \hat{\M Y} \M v_r(\M x) 
    &= \langle \hat{\M Y}, \M X  \rangle\\
    &=\|\M b - \M A\M \phi_r(\M x)\|^2_{\M{V}^{-1}} - \rho \\
    &  + \sum_i\mu_i\left\langle \M B_i^{\perp}, \M v_r( \M x)\M v^{\transpose}_r(\M x) \right \rangle \\
    & + \sum_{j} \lambda_{j} \left\langle \M G_{j}, \M v_r( \M x)\M v^{\transpose}_r(\M x)  \right\rangle,
\end{aligned}
\end{equation}

Consider the partition of the matrix $\M{B}^{\perp}$ and $\M{G}$:
\begin{equation}
    \M B^{\perp} =: \begin{bmatrix}
        0 & \M B^{\perp \intercal}_{l} \\
        \M B^{\perp}_{l}& \M B^{\perp}_{Q} 
    \end{bmatrix},    \M G =: \begin{bmatrix}
        \M{G}_c & \M G^{\intercal}_{l} \\
        \M G_{l}& \M G_{Q} 
    \end{bmatrix}.
\end{equation}
Then, the \textit{dual solution} $\hat{\M Y}$ takes the following form:
\begin{equation}
\label{eq:dual_sol_belief}
\hat{\M Y} = \begin{bmatrix} \M Y_c&\M Y_l^\transpose\\\M Y_l & \M Y_Q 
    \end{bmatrix}
\end{equation}
where
\begin{equation*}
\begin{aligned}
\M Y_c &= \M b^{\transpose}\M{V}^{-1}\M b - {\rho} + \sum_j\lambda_j\M G_{c,j} \\
\M Y_l &= -{\M A}^{\intercal}\M{V}^{-1}{\M b} + \sum_i{\mu}_i \M B^{\perp}_{l, i}+ \sum_j{\lambda}_j \M G_{l, j}\\
\M Y_Q &= {\M A}^{\intercal}\M{V}^{-1}{\M A} + \sum_i {\mu}_i\M B_{Q,i}^{\perp} + \sum_j{\lambda}_j \M G_{Q, j}
\end{aligned}
\end{equation*}
By the structure of the SOS belief in \Cref{theorem:SDP=POP}, we have:
\begin{equation}
    \M{\Sigma}^{-1} = {\M A}^{\intercal}\M{V}^{-1}{\M A} + \sum_i {\mu}_i\M B_{ Q,i}^{\perp} + \sum_j{\lambda}_j \M G_{Q, j}.
\end{equation}
and we can further quantify the estimation error by substituting $\M{b} = \M{A}\M{\phi}_r(\bar{\M{x}}) + \M{v}$:
\begin{equation}
\label{eq:BPUE_estimate_cons}
\begin{aligned}
    \quad\M \phi_r(\hat{\M x}) &= -\M{Y}_{Q}^{-1}\M{Y}_l \\
    &= {\M \Sigma} ({\M A}^{\intercal}\M{V}^{-1}{\M b} - \sum_i{\mu}_i \M B^{\perp}_{l, i} - \sum_j{\lambda}_j \M G_{l, j}) \\
            &= {\M \Sigma}{\M A}^{\intercal}\M{V}^{-1}{\M A}\M \phi_r(\bar{\M x}) + {\M \Sigma}{\M A}^{\intercal}\M{V}^{-1}\M v \\
            & \quad - \sum_i{\mu}_i \M B^{\perp}_{l, i} -  \sum_j{\lambda}_j \M G_{l, j}
\end{aligned}
\end{equation}
From which we define the same  $\delta\M{\phi}$ as in the unconstrained case: $$\delta\M{\phi} := \M{\Sigma}\M{A}^{\transpose}\M{V}^{-1}\M{v}.$$

For the derivation of the approximated covariance under \Cref{assump:observability}, we follow the same steps in \Cref{sec:theory-covariance} to substitute $\M A = \bar{\M A}$ and $\M b = \bar{\M b} = \bar{\M A}\M \phi_r(\bar{\M x})$ in the noiseless case and set $\M \mu = 0$ and $\M \lambda = 0$. 
This allows concluding that (as in the unconstrained case) $\bar{\M\Sigma} = \bar{\M{A}}^{\transpose}\M{V}^{-1}\bar{\M{A}}$ is an approximated covariance for the limiting case of noiseless measurements.
\end{proof}

\section{Connection to Kalman Filtering}
\label{apx:linearkf}

\subsection{Recovering the BLUE}
For unconstrained~\eqref{eq:BPUE} with Gaussian noise and with $r = 1$, problem~\eqref{eq:BPUE} simplifies to the same linear least squares problem as~\eqref{eq:BLUE}, hence the optimal solution of~\eqref{eq:BLUE} becomes the familiar least squares estimate, which is an optimal estimator in this setup. 


\subsection{Recovering the Kalman Filter}

Here we show that in the linear Gaussian case, using GMKF with $r=1$ produces the standard Kalman filter.

{\bf Update Step and Kalman Gain.}
Consider the optimization in the update step of the GMKF, with an SOS belief obtained from the previous step and a linear observation with Gaussian noise. Let us choose $r = 1$ and assume $\mathcal{K} = \mathbb{R}^n$:
\begin{equation}
    \min_{\M x}\quad \|{\M x} -\hat{\M x}^{-} \|^2_{\M{\Sigma}^{-1}} + \|\M b - \M A\M{x}\|^2_{\M{V}^{-1}}.
\end{equation}
Via the optimality conditions in \Cref{def:kkt_sdp} and  using \Cref{theorem:SDP=POP}, the dual solution with the same partition as \eqref{eq:dual_sol_belief} becomes:
\begin{equation}
\label{eq:dual_sol_kf}
\begin{aligned}
\M Y_l &=  -\M{\Sigma}^{-1}\hat{\M{x}}^{-}  -{\M A}^{\intercal}\M{V}^{-1}{\M b}\\
\M Y_Q &= \M{\Sigma}^{-1} + {\M A}^{\intercal}\M{V}^{-1}{\M A}
\end{aligned}
\end{equation}
Thus, the estimate becomes:
\begin{equation}
\begin{aligned}
    \hat{\M x}^{+} &= -\M{Y}_Q^{-1}\M{Y}_l \\
                   &= \M{Y}_Q^{-1}(\M{\Sigma}^{-1}\hat{\M{x}}^{-} + \M{A}^{\transpose}\M{V}^{-1}( \M{A}\hat{\M{x}}^{-}  + \M{b} - \M{A}\hat{\M{x}}^{-} )) \\
                   & = \hat{\M{x}}^{-} + \M{Y}_Q^{-1}\M{A}^{\transpose}\M{V}^{-1}(\M{b} - \M{A}\hat{\M{x}}^{-}) 
\end{aligned}
\end{equation}
The structure of the estimate resembles the result of the update step of the standard Kalman filter, and indeed inspecting the matrix multiplying the innovation term $\M{b} - \M{A}\hat{\M{x}}^{-}$, we find that it is exactly the Kalman gain:
    $$\M K = (\M A^{\transpose}\M{V}^{-1}\M A + \M{\Sigma}^{-1})^{-1}\M A^{\transpose}\M{V}^{-1}.$$


Now we can recover the covariance given the dual solution in \eqref{eq:dual_sol_kf}. Let $\hat{\M\Sigma}^{+}_k = \M{Y}_Q^{-1}$ and $\M\Sigma = \hat{\M\Sigma}_k^{-}$, we have: 
\begin{equation}
    (\hat{\M{\Sigma}}^{+}_{k})^{-1} = (\hat{\M{\Sigma}}^{-}_{k})^{-1} + \M A^{\transpose}\M V_k^{-1}\M A. 
\end{equation}
Using the matrix inversion lemma, we further have:
\begin{equation}
    \hat{\M{\Sigma}}^{+}_k = \hat{\M{\Sigma}}^{-}_k - \hat{\M{\Sigma}}^{-}_k\M A^{\transpose}(\M V_k + \M{A}\hat{\M{\Sigma}}^{-}_k\M A^{\transpose})^{-1}\M A\hat{\M{\Sigma}}^{-}_k.
\end{equation}
which matches the covariance update in the  Kalman filter. 

{\bf Prediction and Marginalization.}
Consider the prediction step with linear dynamics 
$$\M{x}_{k+1} - \M{F}\M{x}_{k} - \M{G}\M{u}_k = \M{w}_k.$$
We write the prediction step with $r = 1$ as:
\begin{equation}
    \min_{\M{x}_{k+1}, \M{x}_k} \|\M{x}_{k} - \hat{\M x}^{+}_{k}\|^2_{\M{\Sigma}^{-1}_{k}} + \| \M{x}_{k+1} - \M{F}\M{x}_{k} - \M{G}\M{u}_k \|^{2}_{\M{Q}^{-1}}.
\end{equation}
Then we factorize the cost as the quadratic form:
\begin{equation}
    \begin{bmatrix}
        1 \\
        \M{x}_k \\
        \M{x}_{k+1}
    \end{bmatrix}^{\transpose}
    \begin{bmatrix}
        \M{Y}_c & \M{Y}_l^{\transpose} \\
        \M{Y}_l & \M{Y}_Q
    \end{bmatrix}
    \begin{bmatrix}
        1 \\
        \M{x}_k \\
        \M{x}_{k+1}
    \end{bmatrix}
\end{equation}
with 
$$\M{Y}_Q = \begin{bmatrix}
         \M{\Sigma}^{-1}_{k} + \M{F}^{\transpose}\M{Q}^{-1}\M{F} & -\M{F}^{\transpose}\M{Q}^{-1} \\
         -\M{Q}^{-1}\M{F} & \M{Q}^{-1}
    \end{bmatrix}$$
$$\M{Y}_l = -\M{Y}_Q\begin{bmatrix}
     \hat{\M{x}}_k^{-} \\
     \M{F}\hat{\M x}^{-}_k + \M{G}\M{u}_k
\end{bmatrix}$$
Then via Schur complement, we have 

$$\M{Y}_Q^{-1} = \begin{bmatrix}
    \M{\Sigma}_k & \M (\M{Y}_{Q}^{-1})_c^{\transpose} \\
    \M (\M{Y}_{Q}^{-1})_c & \M{F}\M{\Sigma}_k\M{F}^{\transpose} + \M{Q}
\end{bmatrix}, $$ 
where the block for cross terms $\M (\M{Y}_{Q}^{-1})_c = (\M{F}\M{\Sigma}_k\M{F}^{\transpose} + \M{Q}) ^{-1}
\M{Q}^{-1}\M{F}
(\M{\Sigma}^{-1}_{k} + \M{F}^{\transpose}\M{Q}^{-1}\M{F})^{-1} $. 
The block of $\M{Y}_Q^{-1}$ corresponding to $\M{x}_{k+1}$ becomes: 
\begin{equation}
\label{eq:appI-pred1}
\M{\Sigma}_{k+1} := \M{F}\M{\Sigma}_k\M{F}^{\transpose} + \M{Q}. 
\end{equation}
We can further find that the solution of $\M{x}_{k+1}$ becomes:
\begin{equation}
\label{eq:appI-pred2}
\hat{\M{x}}^{-}_{k+1} = \M{F}\hat{\M{x}}^{+}_{k} + \M{G}\M{u}_k.
\end{equation}
It can be easily seen that~\eqref{eq:appI-pred1} and~\eqref{eq:appI-pred2} recover the estimate and covariance prediction in the classical Kalman filter.
\end{appendices}
\clearpage

}

\end{document}